%% file: main_aistats25.tex
\documentclass[twoside]{article}

\usepackage[accepted]{aistats2025}

\usepackage{hyperref}       %

\usepackage[utf8]{inputenc} %
\usepackage[T1]{fontenc}    %
\usepackage{url}            %
\usepackage{booktabs}       %
\usepackage{amsfonts}       %
\usepackage{nicefrac}       %
\usepackage{microtype}      %
\usepackage{xcolor}         %

\usepackage{algorithm}
\usepackage{algorithmic}
\usepackage{amsmath}
\usepackage{amssymb}
\usepackage{mathtools}
\usepackage{amsthm}
\usepackage{color}
\usepackage{graphicx}
\usepackage{mathtools}
\usepackage{enumitem}
\usepackage{thmtools}
\usepackage{thm-restate}
\usepackage{enumitem}
\usepackage{array}
\usepackage[breakable]{tcolorbox}
\usepackage{bm}
\usepackage{float}
\usepackage[normalem]{ulem}
\usepackage{makecell}
\usepackage{pifont}
\usepackage{natbib}

\usepackage[capitalize,noabbrev]{cleveref}

\theoremstyle{plain}
\newtheorem{theorem}{Theorem}[section]

\newtheorem{lemma}[theorem]{Lemma}

\theoremstyle{definition}

\newtheorem{assumption}[theorem]{Assumption}
\theoremstyle{remark}
\newtheorem{remark}[theorem]{Remark}

\usepackage[textsize=tiny]{todonotes}

\begin{document}

\input{defs.tex}

\twocolumn[

\aistatstitle{Cubic regularized subspace Newton for non-convex optimization}

\aistatsauthor{ Jim Zhao\footnotemark[1] \And Aurelien Lucchi \And  Nikita Doikov }

\aistatsaddress{ University of Basel \\
  \texttt{jim.zhao@unibas.ch} \And  University of Basel\\
  \texttt{aurelien.lucchi@unibas.ch} \And EPFL\\
  \texttt{nikita.doikov@epfl.ch}  } ]

\begin{abstract}
This paper addresses the optimization problem of minimizing non-convex continuous functions,
a problem highly relevant in high-dimensional machine learning scenarios, particularly those involving over-parameterization. We analyze a randomized coordinate second-order method named SSCN, which can be interpreted as applying the cubic regularization of Newton's method in random subspaces. This approach effectively reduces the computational complexity associated with utilizing second-order information, making it applicable in higher-dimensional scenarios.
Theoretically, we establish strong global convergence guarantees for non-convex functions to a stationary point, with interpolating rates for arbitrary subspace sizes and
allowing inexact curvature estimation, starting from an arbitrary initialization.
When increasing the subspace size, our complexity matches the $\bigO(\epsilon^{-3/2})$ 
rate of the full Newton's method with cubic regularization.
Additionally, we propose an adaptive sampling scheme ensuring the exact convergence rate of $\bigO(\epsilon^{-3/2}, \epsilon^{-3})$ to a second-order stationary point, without requiring to sample all coordinates.
Experimental results demonstrate substantial speed-ups achieved by SSCN 
compared to conventional first-order methods and other second-order subspace methods.
\end{abstract}

\footnotetext[1]{Reverse alphabetical ordering, all authors contributed equally.}

\input{01_introduction}

\input{02_related_work}

\input{03_analysis}

\input{04_experiments}
\input{05_conclusion}

\newpage

\bibliographystyle{plainnat}
\bibliography{main}

\newpage

\input{06_appendix}

\end{document}

%% file: defs.tex
\newcommand{\cov}[2]{\text{Cov}\left[#1, #2\right]}
\newcommand*\diff{\mathop{}\!\mathrm{d}}
\newcommand{\dw}{\nabla_{\w}}
\newcommand{\err}{\text{err}}
\newcommand{\tf}{\tilde{f}}
\newcommand{\sign}{\text{sign}}
\newcommand{\sqnorm}[1]{\left\|#1\right\|^2} 
\newcommand{\tr}{\text{Tr}}
\newcommand{\var}{\text{var}}
\newcommand{\vectornorm}[1]{\left\|#1\right\|}
\newcommand{\norm}[1]{{}\left\| #1 \right\|}
\newcommand{\bigO}{\mathcal{O}}

\newcommand{\e}{{\bf e}}
\newcommand{\g}{{\bf g}}
\newcommand{\h}{{\bf h}}
\newcommand{\p}{{\bf p}}
\newcommand{\s}{{\bf s}}
\renewcommand{\u}{{\bf u}}
\newcommand{\vb}{{\bf v}}
\newcommand{\w}{{\bf w}}
\newcommand{\x}{{\bf x}}
\newcommand{\y}{{\bf y}}
\newcommand{\z}{{\bf z}}

\def\Am{{\bf A}}
\def\Bm{{\bf B}}
\def\Cm{{\bf C}}
\def\Em{{\bf E}}
\def\Gm{{\bf G}}
\def\Hm{{\bf H}}
\def\Im{{\bf I}}
\def\Pm{{\bf P}}
\def\Qm{{\bf Q}}
\def\Rm{{\bf R}}
\def\Sm{{\bf S}}
\def\Tm{{\bf T}}
\def\Um{{\bf U}}
\def\Zm{{\bf Z}}
\def\Ym{{\bf Y}}
\def\Xm{{\bf X}}

\newcommand{\A}{{\mathcal A}}
\newcommand{\B}{{\mathcal B}}
\newcommand{\C}{{\mathcal C}}
\newcommand{\D}{{\mathcal D}}
\newcommand{\E}{{\mathbb E}}
\newcommand{\F}{{\mathcal F}}
\newcommand{\K}{{\mathcal K}}
\newcommand{\M}{{\mathcal M}}
\newcommand{\N}{{\mathbf N}}
\renewcommand{\P}{{\mathbb P}}
\newcommand{\R}{{\mathbb{R}}}
\renewcommand{\S}{{\mathcal S}}
\newcommand{\T}{{\mathbb T}}
\newcommand{\U}{{\mathcal U}}
\newcommand{\X}{{\mathcal X}}
\newcommand{\Y}{{\mathcal Y}}
\newcommand{\Z}{{\mathbf Z}}

\newcommand{\cD}{{\mathcal D}}
\newcommand{\cN}{{\mathcal N}}
\newcommand{\cS}{{\mathcal S}}

\newtheorem{condition}{Condition}

\DeclarePairedDelimiter\ceil{\lceil}{\rceil}
\DeclarePairedDelimiter\floor{\lfloor}{\rfloor}
\newcommand{\abs}[1]{\left\vert {#1} \right\vert}
\newcommand*{\nfrac}[2]{#1\left/\vphantom{#1}#2\right.}
\newcommand{\argmin}{\text{argmin}}

\definecolor{mydarkgreen}{RGB}{39,130,67}
\definecolor{mydarkred}{RGB}{192,47,25}
\newcommand{\cmark}{\textcolor{mydarkgreen}{\ding{51}}}
\newcommand{\xmark}{\textcolor{mydarkred}{\ding{55}}}

\newtcolorbox{mycoloredbox}[2]{
  breakable,
  colback=gray!20,   %
  colframe=white,    %
  boxrule=0mm,       %
  left=-3pt,         %
  right=4pt,         %
  top=3pt,           %
  bottom=3pt,        %
  width=0.48\textwidth, %
  #1,               %
  title={#2}        %
}

\newtcolorbox{mybox}[2][]
{
  colframe = #2!25,
  colback  = #2!15!white!15,
  #1,
  breakable
}

\newenvironment{theorembox}
   {\begin{mybox}{gray}\begin{theorem}}
   {\end{theorem}\end{mybox}}

\newenvironment{lemmabox}
   {\begin{mybox}{gray}\begin{lemma}}
   {\end{lemma}\end{mybox}}

\newenvironment{conditionbox}
   {\begin{mybox}{gray}\begin{condition}}
   {\end{condition}\end{mybox}}

\newcommand{\beq}{\begin{equation}}
\newcommand{\eeq}{\end{equation}}
\newcommand{\ba}{\begin{array}}
\newcommand{\ea}{\end{array}}
\newcommand{\la}{{\langle}}
\newcommand{\ra}{{\rangle}}
\newcommand{\mat}[1]{\bm{#1}}

\newcommand{\nikita}[1]{{
\definecolor{ggreen}{HTML}{008800}
\color{ggreen}{\footnotesize Nikita: #1}}
}

\def\SM{{S}}

%% file: 01_introduction.tex
\section{INTRODUCTION}

In this paper, we address the problem of minimizing an objective function of the form
$\min_{\x \in \R^n} f(\x)$,
where $f: \R^n \to \R$ is a two-times differentiable, possibly non-convex function, having Lipschitz continuous gradient and Hessian.
Our focus lies on scenarios where the dimension $n$ has the potential to be significantly large, a context that holds relevance for numerous machine learning applications. This pertains particularly to situations in which models tend to exhibit a large number of parameters, commonly referred to as over-parametrization. One method of choice to optimize the objective functions associated with such models 
is to use (randomized) coordinate descent (CD) methods~\citep{nesterov2012efficiency}
or more generally subspace descent methods~\citep{kozak2019stochastic}. The latter class of methods relies on iterative updates of the form
\begin{equation}
\ba{rcl}
\x_{k+1} & = & \x_k - \eta_k \Sm_k^{\top}\Sm_k \nabla f(\x_k),
\ea
\label{eq:subspace_method}
\end{equation}
where $\eta_k > 0$ is a step-size and $\Sm_k \in \R^{\tau \times n}$ is a thin matrix where $\tau \ll n$
is a subspace size. This update corresponds to moving in the (negative) direction of the gradient in the subspace spanned by the columns of $\Sm_k$. 
While various rules exist to choose the matrix $\Sm_k$~\citep{hanzely2020stochastic},
we will here focus on the case where $\Sm_k$ is randomly sampled according to an arbitrary but fixed distribution of coordinate subsets $\Sm_k$ in $\R^n$. We choose to focus on coordinate subsets because, in this scenario, there is no additional cost for applying a projection onto the subspace.

While first-order methods such as Eq.~\eqref{eq:subspace_method} are simple and relatively well-studied, their convergence is notably slow. In contrast, second-order optimization methods excel in terms of convergence speed as they possess the capability to capture more information about the optimization landscape. 
By incorporating such information about the curvature of the objective function, second-order methods such as Newton's method with cubic regularization~\citep{nesterov2006cubic} can navigate complex landscapes with greater efficiency. This often results in faster convergence rates (in terms of iterations) and enhanced accuracy in finding optimal solutions. While first-order methods such as gradient descent are computationally cheaper per iteration, second-order methods can offer significant advantages in scenarios where function landscapes are intricate or when the convergence speed is crucial. However, their per-iteration cost makes them expensive in high-dimensional spaces. 

Based on this observation, we study efficient subspace second-order methods, of the following form
\beq \label{eq:subspace_2nd_order}
\ba{rcl}
\x_{k + 1} & = & \x_k - \Sm_k^{\top}\bigl[ 
\Hm_k + \alpha_k \Im \bigr]^{-1}\Sm_k \nabla f(\x_k),
\ea
\eeq
where $\Hm_k \in \R^{\tau \times \tau}$
is a fixed curvature matrix,
$\Im \in \R^{\tau \times \tau}$ is the identity matrix,
and $\alpha_k > 0$ is a regularization constant. 
Note that in~\eqref{eq:subspace_2nd_order} we need to invert
only the matrix of size $\tau \times \tau$,
which is computationally cheap for small $\tau \ll n$.
Clearly, substituting $\Hm_k := \mat{0}$
into~\eqref{eq:subspace_2nd_order}, we recover 
the pure CD method~\eqref{eq:subspace_method}.
However, to capture the second-order information
about the objective function, we make the following natural choice
\beq \label{ChoiceHk}
\ba{rcl}
\Hm_k & = & \Sm_k \nabla^2 f(\x_k) \Sm_k^{\top},
\ea
\eeq
where $\nabla^2 f(\x_k) \in \R^{n \times n}$
is the Hessian matrix.
Therefore, for a specific coordinate subset,
we need to use the second-order information only along the chosen
coordinates, making our method scalable.
It remains only to choose a parameter $\alpha_k > 0$
in~\eqref{eq:subspace_2nd_order},
which we select by using the cubic regularization technique~\cite{nesterov2006cubic}.
We use the name 
\textit{stochastic subspace cubic Newton} (SSCN) for this algorithm,
as it was introduced recently by~\cite{hanzely2020stochastic}.

\begin{figure}[tbh]
    \centering
    \begin{tabular}{cc}
        \includegraphics[width=0.23\textwidth]{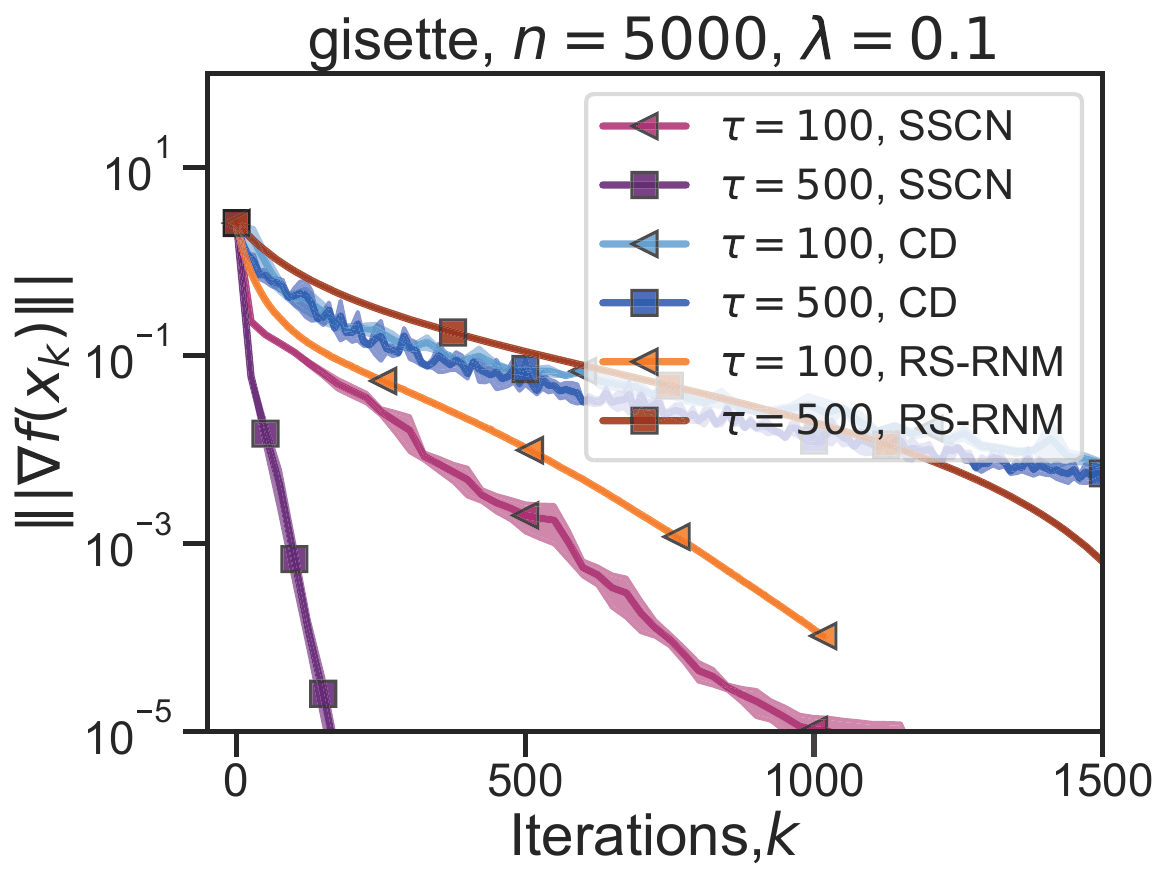}
        \includegraphics[width=0.23\textwidth]{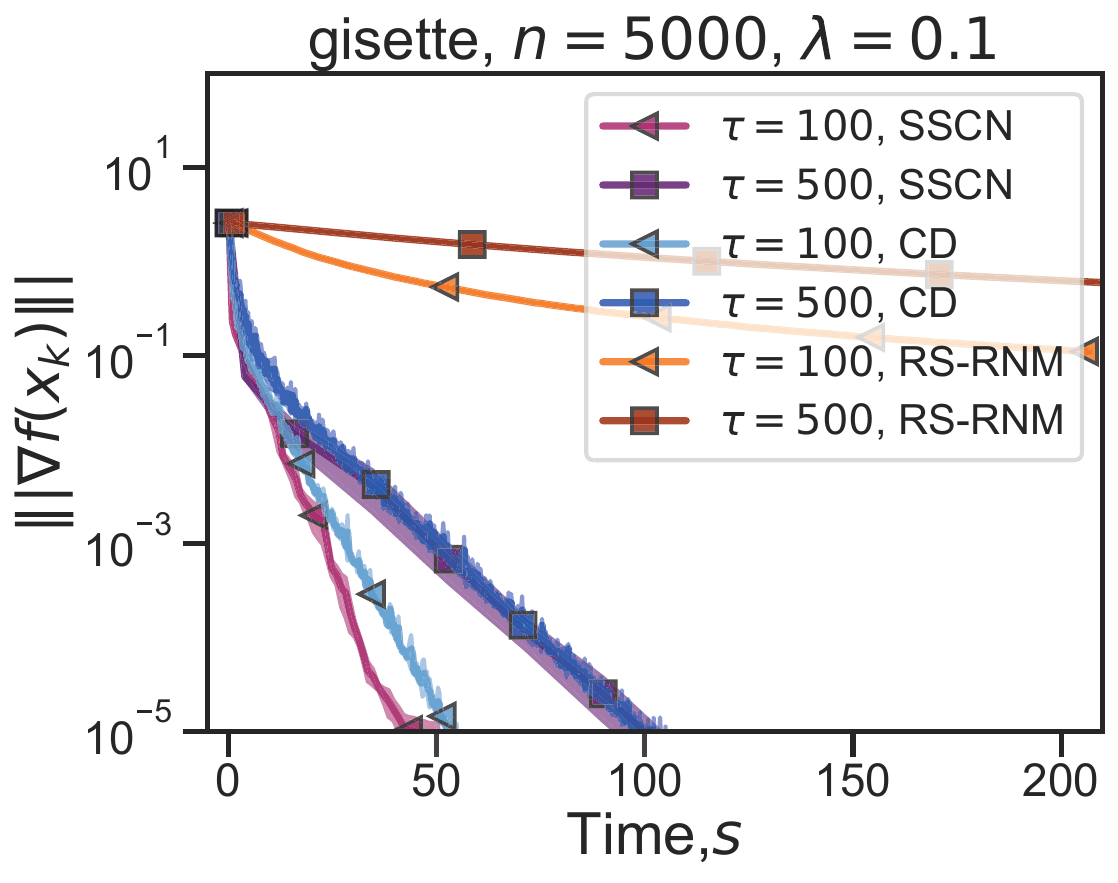} \\
        \includegraphics[width=0.23\textwidth]{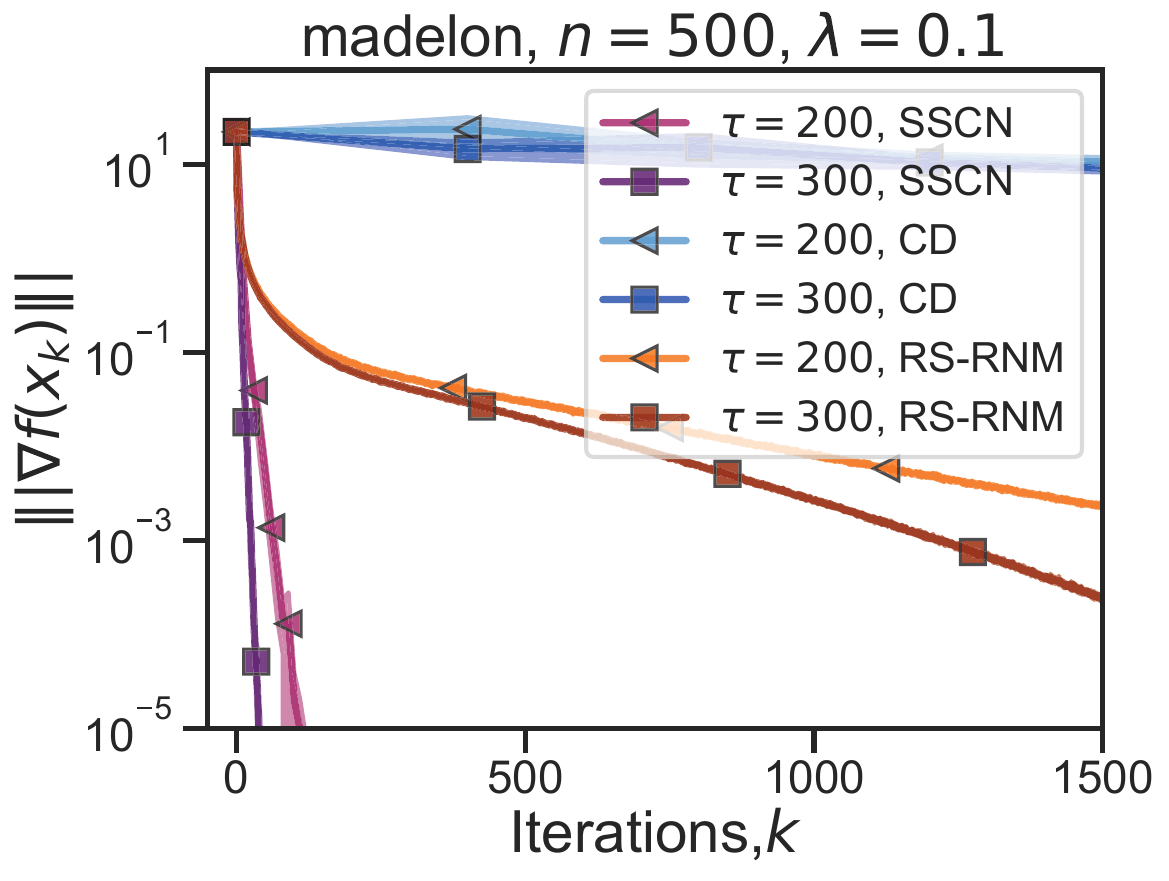}
        \includegraphics[width=0.23\textwidth]{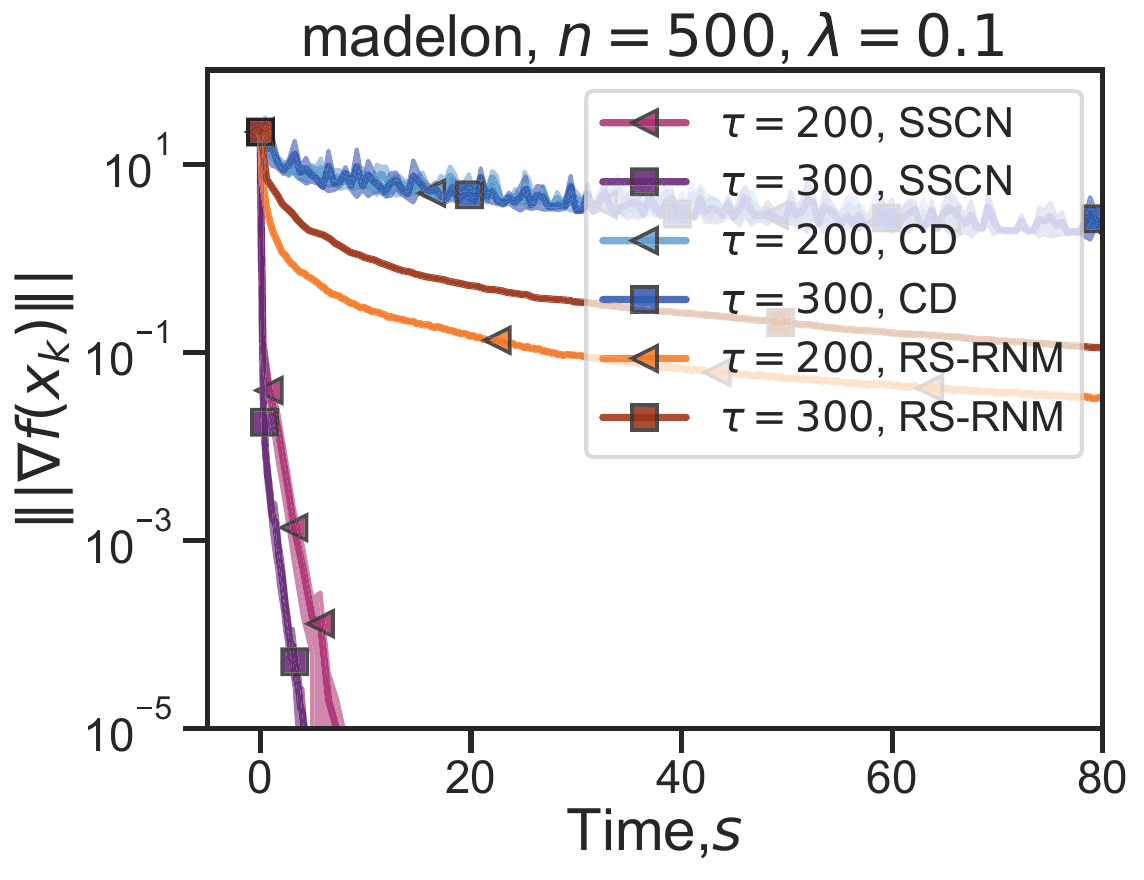}
    \end{tabular}
    \caption{
    Comparison of CD, SSCN and RS-RNM \citep{fuji2022randomized} for different constant coordinate schedules, where $\tau$ denotes the dimension of the subspace in the SSCN method. Performance is measured w.r.t. iterations (first column) and time (second column) averaged over three runs for logistic regression with non-convex regularization with $\lambda = 0.1$ for the datasets \textit{gisette} (first row) and \textit{madelon} (second row). Experiment details are described in Section~\ref{sec:experiments} and additional plots for the \textit{duke} dataset can be found in \cref{fig:logistic_regression_nonconv_CD_vs_SSCN_duke} in Appendix \ref{sec:additional_experiments}.}
    \label{fig:logistic_regression_nonconv_CD_vs_SSCN_gisette}
    
\end{figure}

We see in our experiments (Fig.~\ref{fig:logistic_regression_nonconv_CD_vs_SSCN_gisette}), that the combination of second-order information in a coordinate method, as in SSCN, indeed leads to a significant improvement
in performance, as compared to the first-order CD.
However, while \cite{hanzely2020stochastic} focus on convex objective functions, it has been challenging
to show a strict benefit from employing
second-order information in coordinate descent methods
in the \textit{non-convex case}, which is important for modern applications
in machine learning.
In this work, we demonstrate that one can obtain strong theoretical convergence guarantees for 
general, possibly non-convex functions. Toward this goal, we make the following contributions:
\begin{itemize}

\setlength\itemsep{0.1em}

\item \textbf{Convergence guarantees:} we prove that for non-convex functions, 
SSCN converges to a stationary point starting from an arbitrary initialization (global convergence), at least with the rate of first-order CD when using an \textit{approximation} of the curvature matrix,
and with \textit{strictly better rates} when using the true Hessian~\eqref{ChoiceHk}.
The method convergences for an arbitrary selection of $\tau$, and the
rate becomes better for larger $\tau$, matching the convergence
of the full Cubic Newton for $\tau = n$.
\item \textbf{Sampling schemes:} along with the constant sampling (fixed $\tau$),
we propose
an \textit{adaptive} theoretical sampling scheme ($\tau_k$ grows with $k$) that allows us to prove 
stronger convergence to a \textit{second-order stationary} point for SSCN at the full Cubic Newton rate.
Thus, this sampling scheme avoids being stuck at saddle points.
In practice, we demonstrate that adaptive sampling does not require to fully sample all coordinates,
while preserving the fast rate of convergence for the algorithm.
\item \textbf{Experimentally:} Finally, we provide experimental results that verify our theoretical analysis and demonstrate significant speed-ups compared to the classical coordinate descent.
\end{itemize}

%% file: 02_related_work.tex
\begin{table*}[ht]
  \caption{Comparison to previous work for non-convex optimization. $\tau_k := \tau(S_k)$ refers to the adaptive sampling scheme of the main result in~\cref{thm:convergence_||s||^2}. And $p'$ is defined as $p' = \frac{\tau^2}{n^2}$. \looseness=-1}
  \label{table:comparison_related_works}
  \centering
  \resizebox{\textwidth}{!}{
      \begin{tabular}{c c c c c}
    \toprule
    \textbf{Method} & \textbf{\# Coord.} &  \makecell{\textbf{Convergence rate} \\ to $\|\nabla f(\x)\| \leq \epsilon$} &
         \makecell{\textbf{2nd order} \\ \textbf{stationary pt?}} &
         \makecell{\textbf{Iter. cost} \\ (\textbf{exact step})} \\
     \cmidrule(r){1-5}
     \multicolumn{1}{l}
    {Gradient Descent~\citep{nesterov2013introductory}} & $n$ & $\mathcal{O}(k^{-1/2})$ & \xmark & $\mathcal{O}(n)$ \\
    \hline
    \multicolumn{1}{l}
    {Cubic Newton (CR)~\citep{nesterov2006cubic,cartis2011adaptive}} & $n$ & $\mathcal{O}(k^{-2/3})$ & \cmark & $\mathcal{O}(n^3)$ \\
    \hline
     \multicolumn{1}{l}{Stochastic CR~\citep{tripuraneni2018stochastic, kohler2017sub}} & $n$ & $\mathcal{O}(k^{-2/3})$ & \cmark & $\mathcal{O}(n^3)$ \\
     \hline
 \multicolumn{1}{l}{Coordinate Descent~\citep{richtarik2014iteration}} & $\tau \leq n$ & $\mathcal{O}(k^{-1/2})$  & \xmark & $\mathcal{O}(\tau)$\\
     \hline
     \multicolumn{1}{l}{Random subspace regularized Newton \citep{fuji2022randomized}} & $\tau \leq n$ & $\mathcal{O}(k^{-1/2})$ & \xmark & {$\mathcal{O}(\tau^3)$} \\
     \hline
     \multicolumn{1}{l}{\textbf{This work:}~\cref{RefTheoremAnyConvergenceImproved}}  & $\tau \leq n$ & $\mathcal{O}\left(k^{-2/3} + (n(1-p'))^{1/4} k^{-1/2}\right)$ & \xmark & $\mathcal{O}(\tau^3)$ \\
     \hline
     \multicolumn{1}{l}{\textbf{This work:}~\cref{thm:convergence_||s||^2}} & $\tau_k \leq n$ & $\mathcal{O}(k^{-2/3})$ & \cmark & $\mathcal{O}(\tau_k^3)$ \\
    \bottomrule
  \end{tabular}
  }
\end{table*}

\section{RELATED WORK}
\label{sec:related_work}
\textbf{Newton's method and cubic regularization (CR).}

Newton's method is the classical optimization algorithm 
that employs second-order information (the Hessian matrix)
about the objective function  \citep{jorge2006numerical,nesterov2018lectures}.
Even though the traditional Newton's method with a unit step size is effective at handling ill-conditioned problems, it might not achieve global convergence when starting from arbitrary initializations. Therefore, various regularization techniques
have been proposed to improve the convergence
properties of the Newton's method and
to ensure the global rates (see \cite{polyak2007newton}
for a detailed historical overview).

Among these regularization techniques is the well-established cubic regularization of Newton's method,
which achieves global complexity bounds that are provably better than those of the gradient descent
~\citep{nesterov2006cubic}. 
It uses a cubic over-estimator of the objective function as a regularization technique for the computation of a step to minimize the objective function. One limitation of this approach lies in its dependence on calculating the exact minimum of the cubic model. In an alternative approach, ~\cite{cartis2011adaptive} introduced the ARC method, which alleviates this demand by allowing for an approximation of the minimizer. Other methods have been proposed to reduce the computational complexity of CR, which we discuss below.

\textbf{Subsampled/stochastic Newton and CR.}
For objective functions that have a finite-sum structure, a popular method is to use sub-sampling techniques to approximate the Hessian matrix, such as in~\cite{byrd2011use}. These techniques have also been adapted to cubic regularization~\citep{kohler2017sub, tripuraneni2018stochastic} followed by various improvements such as more practical sampling conditions~\citep{wang2018note} or variance reduction~\citep{zhou2020stochastic}. We emphasize that there is a notable distinction between our work and these prior works. The latter addresses scenarios involving the sampling of data points, whereas our focus lies in the sampling of coordinates. These approaches are therefore orthogonal and could be combined with each other.

\textbf{Coordinate descent.}
Coordinate descent methods are particularly useful when dealing with high-dimensional optimization problems, as they allow for efficient optimization by only updating one coordinate at a time, which can be computationally less expensive than updating the entire vector. There is a wide literature focusing on the case where $\tau = 1$ with precise convergence rates derived for instance by~\cite{nesterov2012efficiency, richtarik2014iteration}. A generalization of CD known as subspace descent ~\citep{kozak2019stochastic} projects the gradient onto a random subspace at each iteration. The rates of convergence of coordinate descent for non-convex objectives were studied by \cite{patrascu2015efficient}.

\textbf{Subspace Newton.}
The subspace idea discussed above has also been extended to Newton's method by~\cite{gower2019rsn} resulting in the update rule
$\x_{k+1} = \x_k - \eta \Sm_k (\nabla_S^2 f(\x_k))^{-1} \nabla_S f(\x_k)$,
where the gradient and Hessian are computed over a subset of selected coordinates $S$.
Finally,~\cite{hanzely2020stochastic} extends this idea to cubic regularization, deriving rates of convergence in the more generic case where the objective function $f$ is convex. In contrast, we consider the case where the function $f$ is not necessarily convex. 
Recently,~\cite{fuji2022randomized} analyzed a randomized subspace version of a (differently regularized) Newton's method discussed in~\cite{ueda2010convergence} and obtained convergence to a first-order critical point with the same iteration complexity as CD. An overview of previous works for non-convex optimization can be found in Table \ref{table:comparison_related_works}.

%% file: 03_analysis.tex
\section{ALGORITHM}
\label{sec:algorithm}

\subsection{Notation and setting}

Our goal is to optimize a bounded below function $f: \R^n \to \R$.
We denote $f^{\star} := \inf_{\x \in \R^n} f(\x)$.
By $\| \cdot \|$ we denote the standard Euclidean norm for vectors
and the spectral norm for matrices.

\begin{assumption}\label{assumption:lipschitz_hessian}
    We assume that $f$ has Lipschitz continuous Hessians with constant $L_2 \geq 0$, i.e. $\forall \x, \y \in \R^n$, it holds 
   $\| \nabla^2 f(\x) - \nabla^2 f(\y) \| \leq L_2 \| \x - \y \|$.
\end{assumption}
Consequently, we have a global bound for second-order Taylor approximation of $f$,
for all $\x, \y \in \R^n$:
\begin{align}
\label{eq:global_bound_tayler_expansion}
    \nonumber
    |f&(\y) - f(\x) - \langle \nabla f(\x), \y - \x \rangle  \\
    &- \frac{1}{2} \langle \nabla^2 f(\x) (\y - \x), \y - \x \rangle | \leq  \frac{L_2}{6} \| \y - \x \|^3.
\end{align}
For a given subset of coordinates $\SM \subset [n] := \{1,\hdots,n\}$ and any vector $\x \in \R^n$ we denote by $\x_{[S]} \in \R^n$ a vector with nonzero elements whose indices are in $S$ and by $\Am_{[S]} \in \R^{n \times n}$ the matrix with nonzero elements whose both rows and columns are in $S$
\begin{align*}
    \begin{split}
    \label{eq:def_x_Sm}
    (\x_{[S]})_i \overset{\text{def}}{=} 
        \begin{cases}
            x_i \; & \text{if } i \in S \\
            0    \; & \text{else},
        \end{cases}
    \end{split}
    \begin{split}
    (\Am_{[S]})_{ij} \overset{\text{def}}{=}  
    \begin{cases}
        A_{ij} & \, \text{if } i,j \in S\\
        0   & \, \text{else}.
    \end{cases}
\end{split}
\end{align*}
We also denote the cardinality of set $S$ by $\tau(S) := |S|$.
Furthermore, we denote by 
$\x|_{S} \in \R^{\tau(S)}$ the vector which only contains the entries in $\x$, which are in $S$. Similarly $\Am|_{S} \in \R^{\tau(S) \times \tau(S)}$ contains only the entries $\Am_{ij}$ with $i \in S$ and $j \in S$.
Note that $\mat{A}_{[S]}|_{S} \equiv \mat{A}|_{S}$.

\subsection{Stochastic subspace cubic newton}

Inequality~\eqref{eq:global_bound_tayler_expansion}
motivates us to introduce 
the following cubic regularized model,
for a given $\x \in \R^n$, symmetric matrix $\Qm = \Qm^{\top} \in \R^{n \times n}$,
coordinate subset $S \subset [n]$,
and regularization parameter $M > 0$,
\begin{align*}
&\bar{m}_{\x, \Qm, S, M}(\h):= \\
&
f(\x) + \langle \nabla f(\x), \h_{[S]} \rangle  + \frac{1}{2} \langle \Qm \h_{[S]}, \h_{[S]}\rangle + \frac{M}{6} \| \h_{[S]} \|^3
\\[7pt]
&\equiv
f(\x) + \langle \nabla f(\x)_{[S]}, \h \rangle  + \frac{1}{2} \langle \Qm_{[S]} \h, \h \rangle + \frac{M}{6} \| \h_{[S]} \|^3,
\end{align*}
where $\h \in \R^n$. For simplicity and when it is clear from the context, we can omit
extra indices, denoting our model simply by $\bar{m}(\h) = \bar{m}_{\x, \Qm, S, M}(\h) : \R^n \to \R$.
Then, the next iterate of our method is:
\beq \label{OneStep}
\x_{k + 1} =
\x_k + \arg\min\limits_{\h \in \R^n} \bar{m}_{\x_k, \Qm_k, S_k, M}(\h), \; k \geq 0,
\eeq
where $S_k \subset [n]$ are random subspaces of a fixed size $\tau \equiv \tau(S_k)$, $\tau \in [n]$, so we update the $\tau$ coordinates of $\x_k$ which are in $S_k$ in each iteration.
Note that for $\tau = n$ %
we obtain the full cubic Newton step~\citep{nesterov2006cubic}.
However, we are interested in choosing $\tau \ll n$,
such that the corresponding optimization subproblem~\eqref{OneStep} can be solved efficiently 
when %
$n$ is large. 
\begin{remark}
    We note that Eq. \eqref{OneStep} is equivalent 
    to the following  update rule: 
    $
        \x_{k+1}|_{S_k} =
        \x_{k}|_{S_k} + \arg \min_{\h \in \R^\tau} 
        m_{\x_k, \Qm_k, S_k, M}(\h) 
    $
    with a model $m: \R^\tau \to \R$
    defined as 
    \begin{align} 
    \nonumber
        &m_{\x, \Qm, S, M}(\h) \\
        \label{NewModel}
        &:= f(\x) 
        + \langle \nabla f(\x)|_{S}, \h \rangle + \frac{1}{2} \langle \Qm |_{S}\h, \h \rangle + \frac{M}{6} \| \h \|^3,
    \end{align}
    where $\h \in \R^\tau$.
    This update rule implies that in practice we only need to solve a cubic subproblem of dimension $\tau \ll n$.
    At the same time, it is more convenient
    to work with the initial $\bar{m}$ in the theoretical analysis.
\end{remark}

\begin{algorithm}[!ht]
\begin{algorithmic}[1]
\STATE \textbf{Initialization:} $\x_0 \in \R^n$, distribution $\cD$ of random subsets $S \subset [n]$ of size $\tau \equiv \tau(S)$
\FOR{$k =  0, 1, \dots$}
\STATE Sample $S_k$ from distribution $\cD$
\STATE Estimate matrix $\Qm_k \approx \nabla^2 f(\x_k)_{[S_k]}$
\STATE Set $\x_{k+1}|_{S_k} \,=\, \x_k|_{S_k} + \arg\min\limits_{\h \in \R^\tau}  m_{\x_k, \Qm_k, S_k, M_k}(\h)$,
for some $M_k > 0$
\label{eq:x_update_CRCD}
\ENDFOR
\end{algorithmic}
\caption{SSCN: Stochastic Subspace Cubic Newton}
\label{alg:SSCN}
\end{algorithm}

The resulting optimization method is stated in Algorithm~\ref{alg:SSCN}, where we first sample $S_k$ from a chosen distribution $\cD$, and then perform an update by minimizing our model. Additional details concerning this minimization subproblem will be provided shortly. 

In the remaining part of this work, we will assume $\cD$ to be the uniform distribution.
In Algorithm~\ref{alg:SSCN}, we have the freedom of choosing 
the matrix $\Qm_k$
in every iteration.
Even though we mainly focus on employing
the true Hessian $\nabla^2 f(\x_k)$ for a selected subset of coordinates, there are several
interesting possibilities, that can be also considered within our framework.
\begin{itemize}
\setlength\itemsep{0.1em}
    \item \textit{Full Hessian matrix}: $\Qm_k = \nabla^2 f(\x_k)_{[S_k]}$.
    Then, our algorithm recovers the SSCN method from~\cite{hanzely2020stochastic}.
    This is the most powerful version, which we equip
    with new strong convergence guarantees, valid for general
    non-convex objective functions.
    
    \item \textit{No second-order information}: 
    $\Qm_k = \mat{0}$.
    In this case, Algorithm~\ref{alg:SSCN}
    and our analysis recovers the  
    rate of the coordinate descent (CD) method,
    even though our algorithm is slightly different
    due to the cubic regularization, which affects the step-size selection: one step becomes
    $\h_k^* = - \eta_k \nabla f(\x_k)|_{S_k} \in \R^{\tau}$,
    with 
    $\eta_k = \sqrt{ 2 / (M_k\| \nabla f(\x_k)|_{S_k}\|)}$.
    The ability to tackle this extreme case
    demonstrates the robustness of our iterations:
    we show that Algorithm~\ref{alg:SSCN} works even
    if the approximation $\Qm_k \approx \nabla^2 f(\x_k)_{[S_k]}$
    is not exact.

    \item \textit{Lazy Hessian updates}: $\Qm_k = \nabla^2 f(\x_t)_{[S_k]}$,
    where $\x_t$ is some point from the past, $0 \leq t \leq k$.
    In this case, we use the same Hessian for several steps, which improves the arithmetic complexity of the method~\citep{doikov2023second}.

    \item \textit{Quasi-Newton updates}, such as DFP, BFGS, 
    and L-BFGS~\citep{dennis1977quasi,nocedal1999numerical}.
    Some recent works combine quasi-Newton methods
    with the cubic regularization technique, including~\cite{kamzolov2023cubic,scieur2024adaptive}.

    \item \textit{Finite-difference approximation}. 
    When second-order information is not directly available,
    we can approximate the Hessian using only the gradients, as follows:
    $(\Qm_k)_{ij} = \frac{1}{\delta}( \nabla f(\x_k + \delta \e_i) - \nabla f(\x_k) )_j$, for $i, j \in S_k$, 
    with a possible symmetrization later on. Here, $\e_i$ is the basis vector, and $\delta > 0$ is a parameter.
    Choosing $\delta$ sufficiently small we ensure
    the rate of the cubic Newton~\citep{cartis2012oracle,grapiglia2022cubic,doikov2023first}.
    
\end{itemize}

\paragraph{Complexity of solving the cubic subproblem.}
Note that for non-convex functions the model $\bar{m}(\h)$ is in general non-convex.
However, its \textit{global minimum} is always well-defined and can be found by
standard techniques from linear algebra.
One step of our method can be rewritten
in the standard form, as in~\eqref{eq:subspace_2nd_order}:
$
(\x_{k + 1} - \x_k)|_{S_k} =
-\bigl( \Qm_k|_{S_k} + \alpha_k \Im \bigr)^{-1} \nabla f(\x_k)|_{S_k}$,
and the regularization constant $\alpha_k$
can be found as a solution
to the following
\textit{univariate concave} maximization,
\begin{align}
\label{eq:alpha_formula}
\ba{cl}
\max\limits_{\alpha > 0}
\Bigl[
&-\frac{1}{2}\la \bigl(\Qm_k|_{S_k} + \alpha \Im \bigr)^{-1}
\nabla f(\x_k)|_{S_k}, \\
&\nabla f(\x_k)|_{S_k} \ra
- \frac{2^4\alpha^3}{3M_k^2}
\; : \;
\Qm_k|_{S_k} + \alpha\Im \succ \mat{0}
\Bigr].
\ea
\end{align}
It can be done efficiently by means of any
one-dimensional procedure (e.g. the binary search or univariate Newton's method, 
see also Chapter 7 of~\cite{conn2000trust} and
Section 5 of~\cite{nesterov2006cubic}).
Then, the complexity of one step is $\mathcal{O}(\tau^3)$,
as for the standard matrix inversion.
\cite{cartis2011adaptive} show that one can retain the fast rate with an \emph{inexact model minimizer} which solves $m(\h)$ on a Krylov subspace. The subproblem can also be solved using gradient descent, as shown by~\cite{carmon2019gradient}, which means that only Hessian-vector products are required.
\section{CONVERGENCE ANALYSIS}
\subsection{General convergence rate}
First, let us establish a general convergence result for Algorithm~\ref{alg:SSCN}
when using an arbitrary symmetric matrix $\Qm_k = \Qm_k^{\top} \in \R^{n \times n}$.
To quantify the approximation error, we introduce parameter $\sigma \geq 0$ such that
\beq \label{SigmaDef}
\ba{rcl}
\| \nabla^2 f(\x_k)_{[S_k]} - \Qm_k \| & \leq & \sigma, \qquad \forall k.
\ea
\eeq
Note that $\sigma$ bounds the distance to the already subsampled Hessian $\nabla^2 f(\x_k)_{[S_k]}$, \textbf{\emph{not}} to the \emph{full Hessian} $\nabla^2 f(\x_k)$. Thus, in the simplest case one could choose $\Qm_k = \nabla^2 f(\x_k)_{[S_k]}$ (and we have $\sigma = 0$),
but our approach also accommodates inaccuracies in the Hessian estimation. 
For instance, if we use lazy Hessian updates ($\Qm_k = \nabla^2 f(\x_t)_{S_k}$, where $\x_t$ is some point from the past, $0 \leq t \leq k$), we can choose $\sigma := L_2 \| \x_k - \x_t \|_2$, where $L_2$ is the Lipschitz constant of the Hessian (\cref{assumption:lipschitz_hessian}).
Similar bounds can be derived if we use finite-difference approximation of the second derivatives. Then, $\sigma$ corresponds to the approximation error.
On the other extreme, when $\Qm_k = \mat{0}$ our method should recover the standard rate of the coordinate descent (CD).
For that, we have to use an additional assumption.
\begin{assumption}\label{assumption:lipschitz_gradient}
    Let $f$ have Lipschitz continuous gradient with constant $L_1 > 0$, i.e. $\forall \x, \y \in \R^n$, it holds 
   $\| \nabla f(\x) - \nabla f(\y) \| \leq L_1 \| \x - \y \|$.
\end{assumption}
In this case the error $\sigma$ is upper boudned by $\sigma \leq L_1$.
By choosing the regularization parameter $M_k$ in our cubic model sufficiently large, 
we can ensure that the following progress condition is satisfied, at each iteration $k \geq 0$:
\beq \label{M_k_Choice}
\ba{rcl}
f(\x_{k} + \h_k^{*}) & \leq & 
\bar{m}_{\x_k, \Qm_k, S_k, M_k}(\h_k^{*}).
\ea
\eeq
This inequality justifies that our method is \underline{monotone} (i.e. $f(\x_{k + 1}) \leq f(\x_k))$, and it is crucial 
for establishing that SSCN can achieve \textit{any} precision in terms of the gradient norm for \textit{any} size of a stochastic subspace, and for arbitrary selection of $\Qm_k$.
All missing proofs are provided in the appendix.

\begin{restatable}{theorem}{TheoremAnyConvergenceQ} \label{RefTheoremAnyConvergenceQ}
\label{thm:RefTheoremAnyConvergenceQ}
Let the sequence $\{\x_i\}$ be generated by Algorithm~\ref{alg:SSCN} 
with arbitrary $\mat{Q}_k$ satisfying~\eqref{SigmaDef},
and any fixed $\tau \equiv \tau(S_k) \in [n]$. Let
the regularization parameter at iteration $k \geq 0$ be chosen as
\beq \label{M_k_any_convergence_choice}
\ba{rcl}
M_k & = & 2L_2 + \frac{7^2 (\sigma + L_1)^2}{2\| \nabla f(\x_k)_{[S_k]} \|}.
\ea
\eeq
For a given accuracy level $\varepsilon > 0$, assume that
$\| \nabla f(\x_i) \| \geq \varepsilon$, for all $0 \leq i \leq K$.
Then, it holds
\beq \label{FixedGeneralCompl}
\ba{rcl}
\!\!\!\!\! K \leq 
\frac{n}{\tau} 
\Bigl[
\frac{(2 + \frac{7^2}{3}) (\sigma + L_1)  (f(\x_0) - f^{\star})}{\varepsilon^2}
+ \frac{4L_2 (f(\x_0) - f^{\star})}{3(\sigma + L_1) \varepsilon}
\Bigr].
\ea
\eeq
\end{restatable}

\subsection{The power of second-order information}

In this and the following sections, we assume that 
we use the exact second-order information, $\Qm_k = \nabla^2 f(\x_k)_{[S_k]}$.
In this case, we are able to ensure a faster rate of convergence,
thus showing the power of utilizing the second-order information,
for general possibly non-convex problems.

We would like to highlight that the analysis we employ is fundamentally different than in \citet{hanzely2020stochastic}, which means that our work can not be seen as a straightforward extension but it instead requires novel ideas to derive the proofs. 
Indeed, the non-convex analysis of the deterministic full-space cubic Newton method for non-convex objectives is based on the following progress for every iteration  (see, e.g. the proof of Theorem 1 in \citet{nesterov2006cubic}),
\begin{equation*}
    f(\x_k) - f(\x_{k+1}) \geq \frac{1}{12 \sqrt{L}} \| \nabla f(\x_{k+1}) \|^{3/2},
\end{equation*}
where $L$ is the Lipschitz constant of the Hessian. 
Note that contrary to the first-order gradient methods, we have the gradient norm at the \textit{new point} $\| \nabla f(\x_{k+1}) \|$ in the right-hand side of the last inequality. Therefore, when analysing a stochastic subspace version of the cubic Newton, it becomes challenging to generalize this proof, since $\x_{k+1}$ depends on the choice of the random subspace. In contrast, our new analysis is based on studying the global properties of the stochastic cubic model of the objective function and allows inexactness in the progress inequality (see sections \ref{sec:convergence_fixed_sampling} and \ref{sec:convergence_adaptive_sampling} in the appendix for the full proofs). To the best of our knowledge, our analysis is novel.\\
Using the model's optimality (see Section~\ref{SectionAppendixOptimality} in the appendix), 
we can establish that we consistently decrease the objective function at each step,
with the following progress.

\begin{restatable}{lemma}{ModelDecrease}
\label{lemma:model_decrease}
    For any $\x_k \in \R^n$
    and arbitrary  $\SM_k \subset [n]$, let
    $\h_k^* = \arg \min_\h \bar{m}_{\x_k, \nabla^2 f(\x_k), S_k, M}(\h)$.
    Then we have, for any $M \geq L_2$,
    \beq \label{eq:model_decrease}
    \ba{rcl}   
       f(\x_k) - f(\x_{k+1}) & \geq &  \frac{M}{12} \| \h_k^*\|^3.
    \ea
    \eeq
\end{restatable}

For our refined analysis,
we must characterize the distance between the coordinate-sampled gradient 
$\nabla f(\x)_{[S]}$ and the full gradient $\nabla f(\x)$, 
as well as the corresponding distance for the Hessian.

\begin{restatable}{lemma}{ConcentrationGradientNormAndHessian}
\label{lemma:concentration_bounds_gradient_norm_and_hessian}
    For any $\x \in \R^n$ and any subset $S \subset [n]$, we have with probability at least $1-\delta$: 
    \begin{align} 
    \label{eq:concentration_bounds_gradient_norm}
    &\!\!\!\| \nabla f(\x) - \nabla f(\x)_{[S]} \| 
        \leq \delta^{-1} \sqrt{1 - \frac{\tau(S)}{n} } \| \nabla f(\x) \|\\
    \label{eq:concentration_bounds_hess_vec}
    &\!\!\!\| \nabla^2 f(\x) - \nabla^2 f(\x)_{[S]} \| 
        \leq \delta^{-1} \sqrt{1 - p'} \| \nabla^2 f(\x) \|_F
    \end{align}
    where $p' := (\tfrac{\tau(S)}{n})^2$
    and $\| \cdot \|_F$ is the Frobenius norm of a matrix.
\end{restatable}

As expected, Lemma~\ref{lemma:concentration_bounds_gradient_norm_and_hessian} shows that increasing the sampling size $\tau(S)$ yields more accurate sampled gradients and Hessians. This effect will later be verified experimentally in Section~\ref{sec:experiments} where we will test different coordinate sampling schedules. Note that we can prove analogous bounds in expectation, as in
Eqs.~\eqref{eq:concentration_bounds_gradient_norm}-\eqref{eq:concentration_bounds_hess_vec} (see \cref{lemma:concentration_bounds_gradient_norm}, \ref{lemma:concentration_bounds_hess_vec}).
We denote by $\E[\cdot]$ the full expectation w.r.t. all randomness in the algorithm.
We are ready to state our main convergence rate for a fixed coordinate sample size.

\begin{restatable}{theorem}{TheoremAnyConvergenceImproved} \label{RefTheoremAnyConvergenceImproved}
\label{thm:RefTheoremAnyConvergenceImproved}
Let the sequence $\{\x_i\}$ be generated by Algorithm~\ref{alg:SSCN} 
with $\Qm_k = \nabla^2 f(\x_k)_{[S_k]}$,
and any fixed $\tau \equiv \tau(S_k) \in [n]$. 
For a given $\varepsilon > 0$, assume that the first $K$ gradients are such that
$\E \| \nabla f(\x_i) \| \geq \varepsilon, 0 \leq i \leq K$.
Then for a sufficiently large $M \geq L_2$, it holds
\begin{equation} \label{FixedNewtonComplexity}
\ba{rcl}
K & = &
\mathcal{O}
\Bigl(
\left[\frac{n}{\tau}\right]^{3/2}
\frac{\sqrt{L_2}(f(\x_0) - f^{\star})}{\varepsilon^{3/2}} \\
&& + \,
n^{1/2}(1 - p')^{1/2}
\left[ \frac{n}{\tau} \right]^{2} \frac{L_1 (f(\x_0) - f^{\star}}{\varepsilon^2}
\Bigr).
\ea
\end{equation}
\end{restatable}
We see that according to this result, SSCN achieves any desirable accuracy for the gradient norm after a finite number of iterations,
starting from an arbitrary initialization $\x_0$. The size of the stochastic subspaces $\tau = \tau(S_k)$ can be \textit{arbitrary}. Note that $p' = (\tfrac{\tau}{n})^2$ increases with increasing $\tau$ and we recover the rate of the full Cubic Newton for $\tau = n$. 
If the cubic subproblem is solved exactly in each iteration (with the cost of $\mathcal{O}(\tau^3))$,
then the total computational complexity is
 $\mathcal{O}
\Bigl(
(\tau n)^{3/2}
\frac{\sqrt{L_2}(f(\x_0) - f^{\star})}{\varepsilon^{3/2}}
+ 
(1 - p')^{1/2}
 \tau n^{5/2} \frac{L_1 (f(\x_0) - f^{\star}}{\varepsilon^2}
\Bigr)$. Note that the global computational complexity for CR is $\mathcal{O}
\Bigl(
n^3
\frac{\sqrt{L_2}(f(\x_0) - f^{\star})}{\varepsilon^{3/2}} \Bigr)$. 
Therefore, we see that SSCN is already \textit{strictly better} than CR for
$\tau (1 - p')^{1/2} < (\varepsilon n L_2)^{1/2} / L_1$.
Note that in practice we can already observe speed-ups of SSCN for smaller problem dimensions.
We also see that the second term in~\eqref{FixedNewtonComplexity}
matches the complexity of CD, up to the factor $n^{3/2}(1 - p')^{1/2} / \tau$, which tends to $0$ when $\tau \to n$.
Lastly, we highlight that the final complexity of SSCN
with constant sampling can be taken as the minimum of
both~\eqref{FixedGeneralCompl} and \eqref{FixedNewtonComplexity},
thus
achieving the best of these two bounds.

\subsection{Adaptive sampling scheme}
\label{sec:adaptive_sampling}
In the following, we present a scheme to sample the number of coordinates at each iteration that yields 
even faster convergence to a \textit{second-order stationary point} up to an arbitrary precision. This scheme is \emph{adaptive} in the sense that it depends on the gradient and Hessian measured at each iteration. We measure optimality using the standard first and second-order criticality measures $\| \nabla f(\x) \|$ and $\lambda_{\min}(\nabla^2 f(\x))$ (the minimum eigenvalue of the Hessian matrix). To do so, we introduce the following quantity~\citep{nesterov2006cubic}:
\begin{equation*}
\ba{rcl}
\mu(\x) & = & \max \left\{ \| \nabla f(\x) \|^{3/2}, \, [-\lambda_{\min}(\nabla^2 f(\x))]^3 \right\}.
\ea
\end{equation*}
Given $\epsilon_1 > 0$ and $\epsilon_2 > 0$, we would like to have with probability at least $1-\delta$
\begin{align} 
 \label{GradHessTight}
\ba{rcl}
\| \nabla f(\x_k) - \nabla f(\x_k)_{[S_k]} \| &\!\!\! \leq  \delta^{-1}\epsilon_1, \\
\| \nabla^2 f(\x_k) - \nabla^2 f(\x_k)_{[S_k]} \|  &\leq  \delta^{-1}\sqrt{\epsilon_2}.
\ea
\end{align}
\begin{figure}[tbh]
    \centering
        \begin{tabular}{cc}
     \includegraphics[width=0.23\textwidth]{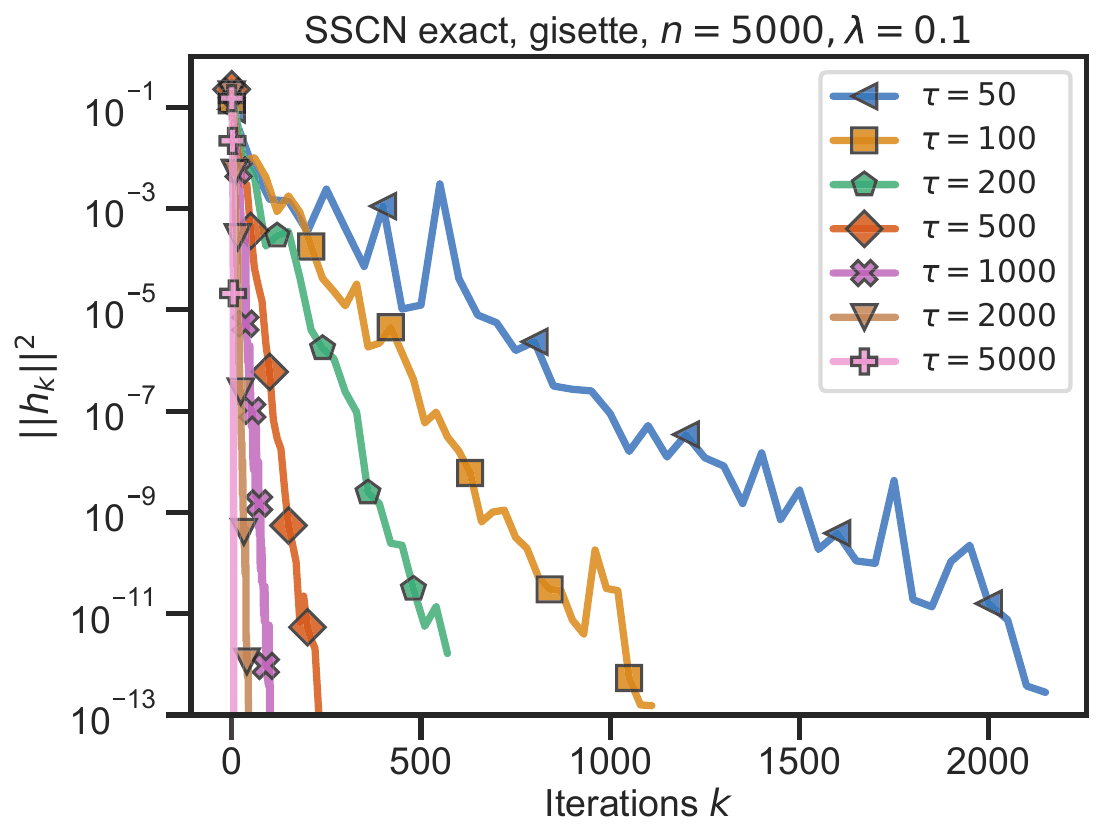}
    \includegraphics[width=0.23\textwidth]{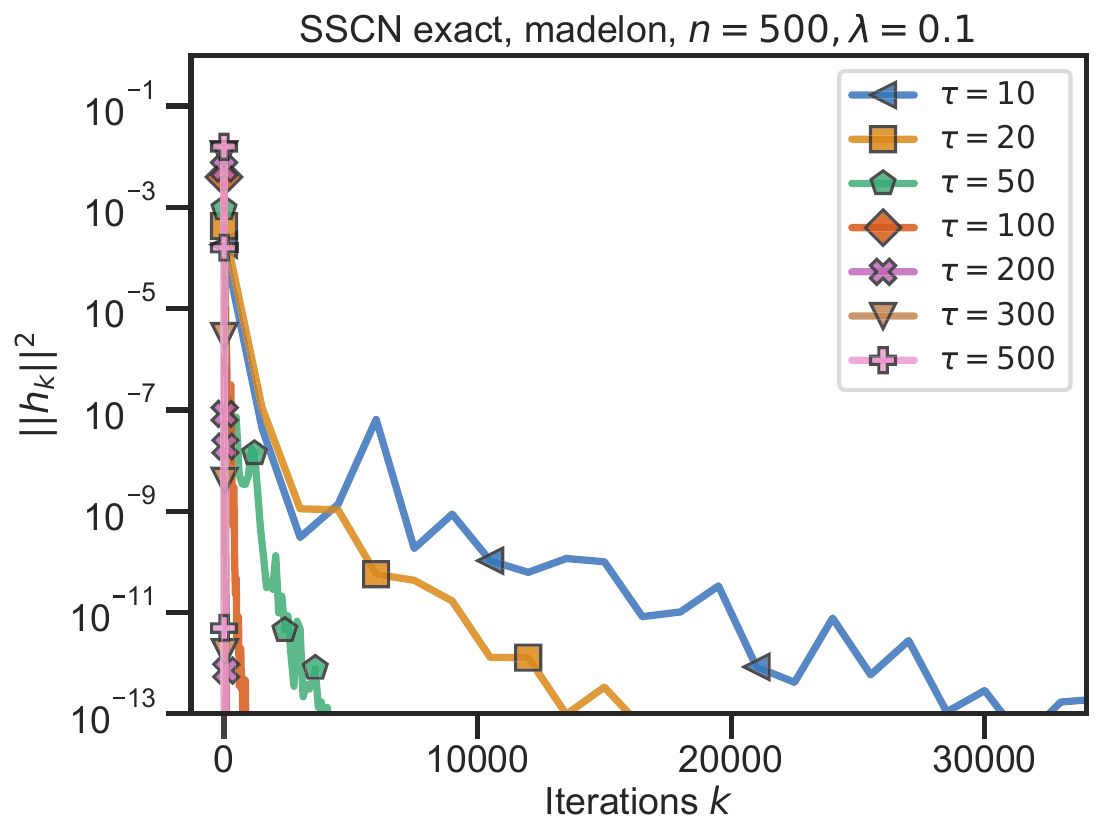}
        \end{tabular}
        
    \caption{Squared norm of the step  $\h_k$ for different constant coordinate schedules for logistic regression with non-convex regularization with $\lambda = 0.1$ for two different datasets (left: \textit{gisette}, right: \textit{madelon}). For the same plot for the  \textit{duke} dataset, see \cref{fig:logistic_regression_norm_h_k2_duke}. Note that the y-axis is plotted in log-scale.}
\label{fig:logistic_regression_norm_h_k2}
\end{figure}

Taking into account Lemma~\ref{lemma:concentration_bounds_gradient_norm_and_hessian},
to ensure~\eqref{GradHessTight}, it is enough to choose $\tau(S_k)$ such that
$$
\ba{rcl}
\sqrt{1 - \frac{\tau(S_k)}{n}} = \frac{\delta^{-1} \epsilon_1}{\| \nabla f(\x_k) \|}
\quad \text{and} \quad  
\sqrt{1 - p'}  =  \frac{\delta^{-1} \sqrt{\epsilon_2}}{\| \nabla^2 f(\x_k) \|_F}.
\ea
$$
Putting everything together, we obtain the following condition for our adaptive scheme 
at iteration $k \geq 0$:
\beq \label{eq:adaptive_tau}
\ba{rcl}
  \!\!\!\!\!\frac{\tau(S_k)}{n} &
  \!\!\!\!\!\geq\!\!\!\!\! & 
  \max \Bigl\{ 1  -  \frac{\delta^{-2}\epsilon_1^2}{\| \nabla f(\x_k) \|^2}, 
  \sqrt{1 -  \frac{\delta^{-2}\epsilon_2}{\| \nabla^2 f(\x_k) \|_F^2 }} \Bigr\}.
\ea
\eeq
Next, we will demonstrate that choosing $\epsilon_1 = \epsilon_2 = c_{k-1} \cdot \| \h_{k-1}^* \|^2$, for a given sequence $(c_k)$, allows us to recover the convergence rate of cubic regularization. Importantly, Theorem~\ref{thm:convergence_||s||^2} permits the choice of an arbitrary sequence $(c_k)$ enabling the adjustment of the number of coordinates as a function of the iteration $k$. Section~\ref{sec:experiments} will illustrate that this flexibility leads to an effective strategy, resulting in substantial speed-ups.

\begin{restatable}{theorem}{MainConvergence}
\label{thm:convergence_||s||^2}
Consider the sequence $\{\x_k\}_{k=0}^K$ generated by $\x_{k+1} = \x_k +  \h_k^*$ where $\tau$ satisfies Eq.~\eqref{eq:adaptive_tau} with $\epsilon_1 = \epsilon_2 = c_{k-1} \cdot \| \h_{k-1}^* \|^2$ for some $c_{k-1} > 0,\; \forall k=0,\hdots,K$. Let $M \geq L_2$.
Let $\Delta_0 = f(\x_0) - f^*$, and define the following constants (dependent on $M$):
\begin{align*}
  C_M &= \left((2M+1)^{3/2} + (4 \delta^{-3/4} \!+\! \sqrt{2} \delta^{-3/2}) \max_i c_i^{3/2} \right)^{-1}, \\
    D_M &= \left( \frac{27 M^3}{2} +  4\delta^{-3/2} \max_i c_i^{3/2} \right)^{-1}.  
\end{align*}
Then with probability at least $1-\delta$, we have 
\begin{equation*}
\ba{rcl}
\min\limits_{1 \leq j \leq K} \mu(\x_j) 
&\leq & 
\frac{1}{K} \Big( \frac{6}{M} \max (C_M^{-1}, D_M^{-1}) \Delta_0 \\
        && + \, 4 \delta^{-3/2} c_{-1}^{3/2} \| \h_{-1}^* \|^3 \Big),
\ea
\end{equation*}
where $\h_{-1}^*$ is such that \\
$\E_{[S]} \| \nabla^2 f(\x_0) - \nabla^2 f(\x_0)_{[S]} \| \leq \sqrt{c_{-1}} \| \h_{-1}^* \|$.

\end{restatable}
Theorem~\ref{thm:convergence_||s||^2} states that Algorithm~\ref{alg:SSCN} with an adaptive sampling scheme converges to an $\epsilon$-second-order stationary point at a rate of $\bigO(\epsilon^{-3/2}, \epsilon^{-3})$.

\begin{figure}[b!]
    \centering
    \begin{tabular}{cc}
        \includegraphics[width=0.24\textwidth]{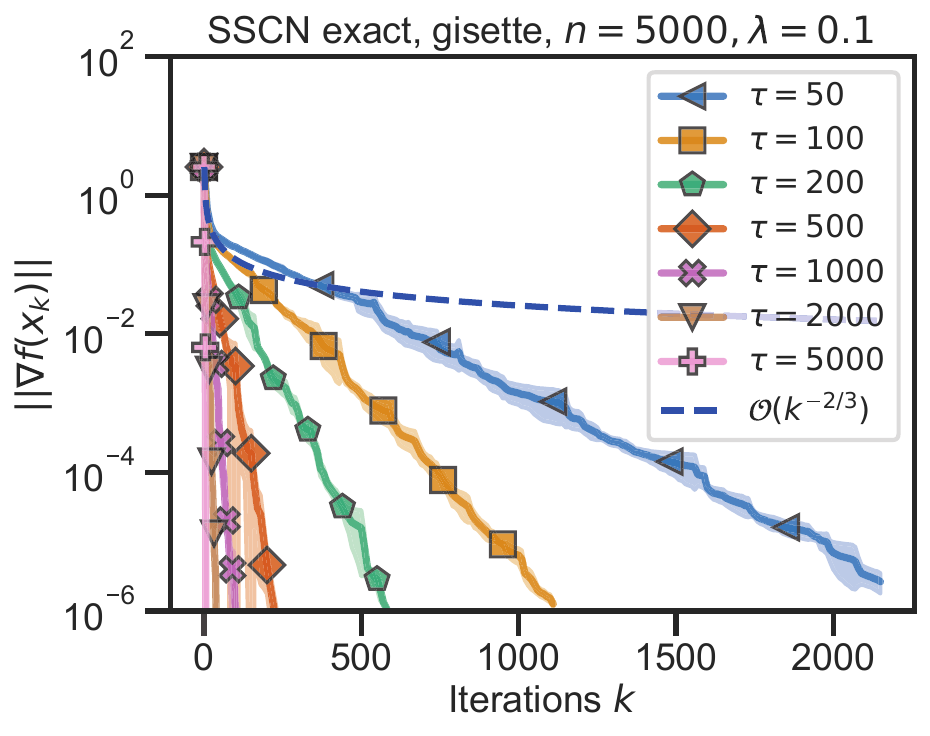}
        \includegraphics[width=0.24\textwidth]{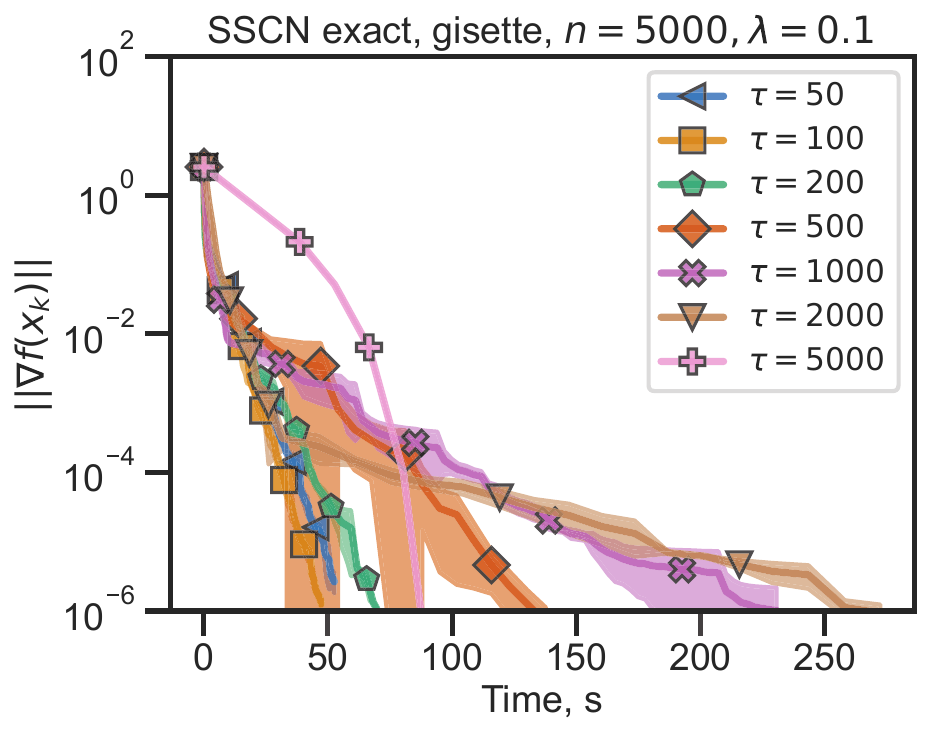}\\
        \includegraphics[width=0.25\textwidth]{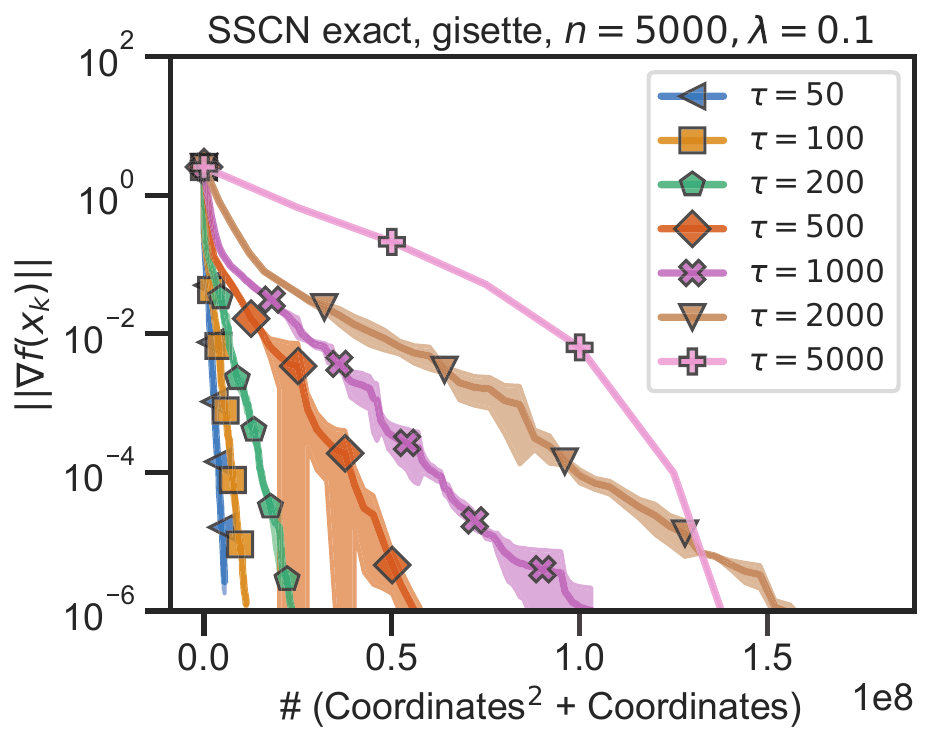}
    \end{tabular}

    \caption{Convergence of different constant coordinate schedules measured w.r.t. iterations (first column), time (second column) and \# (Coordinates$^2$ + Coordinates) evaluated (third column) averaged over three runs for logistic regression with non-convex regularization with $\lambda = 0.1$ for \textit{gisette} dataset. Similar plots for the \textit{duke} and \textit{madelon} datasets can be found in Fig.~\ref{fig:logistic_regression_nonconv_convergence_duke_madelon}.}
\label{fig:logistic_regression_nonconv_convergence_gisette}
\end{figure}

\paragraph{A practical scaling rule.}

The theoretical result of Theorem~\ref{thm:convergence_||s||^2} relies on using Eq.~\eqref{eq:adaptive_tau} which requires access to the gradient and Hessian norms. We note that one can use the estimates computed over a subset of coordinates $S \subset [n]$. We give further details regarding the validity of this scheme in the Appendix. Alternatively, we have discovered that a less complex schedule produces comparable results. Our starting point is the observation that $\| \h_{k-1}^* \|^2$ exhibits exponential growth on most datasets, as illustrated in Figure \ref{fig:logistic_regression_norm_h_k2}. We posit that this exponential increase allows us to employ a coordinate sampling schedule that also follows an exponential trend. By choosing a slowly growing constant for the exponential, we obtain significant speed-ups in practice. We substantiate the effectiveness of this straightforward approach in Appendix~\ref{subsec:constant_vs_exponential_schedule}.

\begin{table}[h]
    \centering
     \caption{Overview of the datasets used in the experimental section with non-convex regularizer $\lambda$.}
    \label{table:datasets_overview}
    \begin{tabular}{c|cccc}
         & \textbf{Type} & \textbf{n} & \textbf{\#samples} & \textbf{$\lambda$} \\
        \hline
        \textit{gisette} & Classification & 5.000 & 6.000  & $0.1$ \\
        \textit{duke} & Classification & 7.129 & 44  & $0.1$ \\
        \textit{madelon} & Classification & 500 & 2.000  & $0.1$ \\
        \textit{realsim} & Classification & 20.958 & 72.309 & $0.1$
    \end{tabular}
\end{table}

%% file: 04_experiments.tex
\section{EXPERIMENTS}
\label{sec:experiments}
We now verify our theoretical results numerically. Due to space limitations, we can only present a fraction of the experiments and refer the reader to the Appendix for the remaining experiments. We ran experiments with a logistic regression loss with a \textit{non-convex regularizer} $\lambda\cdot \sum_{i=1}^n \x_i^2/(1+\x_i^2)$~\citep{kohler2017sub} where $\lambda > 0$ controls the strength of the regularizer. An overview of the three datasets used in our experiments can be found in Table \ref{table:datasets_overview}. All runs were initialized in the origin at $\x_0 = \mathbf{0}$ and all the plots shown in this section are averaged over three runs. The shaded region corresponds to one standard deviation.
In the inner loop of Algorithm~\ref{alg:SSCN}, we solve the subproblem exactly up to a pre-specified tolerance of $1e^{-5}$. We use the exact same subsolver, which is discussed in section 3.3 in~\cite{kohler2017sub} \footnote{We used the implementation provided at \url{https://github.com/jonaskohler/subsampled_cubic_regularization/} under the Apache-2.0 license.}. The experiments were run on an Apple MacBook Pro with an Apple M1 Pro Chip and 16 GB RAM.

\textbf{Constant coordinate schedule.}
In Figure~\ref{fig:logistic_regression_nonconv_convergence_gisette} and~\ref{fig:logistic_regression_nonconv_convergence_duke_madelon} we plot the convergence of different constant coordinate schedules w.r.t. iterations, time and \# (coordinates$^2$ + coordinates). We chose the last measure to approximate the efficiency of the coordinate schedule, since the Hessian matrix scales quadratically and the gradient scales linearly with the number of sampled coordinates.
As expected, full cubic Newton is the fastest in terms of number of iterations and the fewer coordinates are sampled, the more iterations are required to reach the same gradient norm. 
However, in terms of average run time and number of evaluated coordinates, we observe that \textbf{smaller coordinate schedules are faster and more efficient up to some gradient norm} for the $\textit{gisette}$ and $\textit{duke}$ dataset, where the benefit is more pronounced if the problem is higher dimensional. We note that, even for the smaller madelon dataset, smaller coordinate schedules are more efficient w.r.t \# coordinate evaluations until $\|\nabla f(\x_k) \| \approx 1e^{-1}$.
This clearly underscores the potential for substantial computational savings through the utilization of a straightforward approach that samples a fixed number of coordinates. 
Next, we will compare our method to other first and second order subspace methods.

\textbf{Comparison to other methods.}
We compare SSCN to the random subspace regularized Newton method (RS-RNM) proposed by \citep{fuji2022randomized} for different constant coordinate schedules, which we believe is the fairest comparison to our method. Other inexact or stochastic Newton methods either study convex setups \citet{gower2019rsn, jiang2024krylov} or focus on settings, where the stochasticity arises from subsampling data points \citet{yao2018inexact, xu2020newton}, while the dimensionality of the problem remains unchanged. In comparison, \citep{fuji2022randomized} and our work both study subspace Newton methods in a smooth, nonconvex setting.
We also compare SSCN to a vanilla randomized coordinate descent (CD), where the step size $\eta$ was chosen via a line search procedure that guarantees that $f(x_{k+1})-f(x_k) \geq \frac{\eta}{2} \|\nabla f(x_k)\|^2_2$, which is the so-called Armijo condition~\citep{nocedal1999numerical}.

The results are shown in Figure~\ref{fig:logistic_regression_nonconv_CD_vs_SSCN_gisette} amd \ref{fig:logistic_regression_nonconv_CD_vs_SSCN_duke}. 
It can be seen that SSCN is much faster than RS-RNM on both datasets.
We attribute it to two different aspects:
1.) SSCN is more efficient as it does not need to build the full Hessian matrix before projecting it onto a random subspace. Consequently, the speed-up of RS-RNM compared to full space RNM is only w.r.t. the matrix inversion, which is performed in a lower dimension. This is particularly apparent on the \textit{gisette} dataset, which has 5000 coordinates. In contrast, SSCN only constructs the Hessian submatrix of size $\tau \times \tau$. This is why the runtime is much longer for RS-RNM, while the number of iterations for both methods is comparable.
2.) SSCN performs well on ill-conditioned problems, which is the case for the \textit{madelon} dataset (which has a condition number of the order of $1e7 \sim 1e8$), while RS-RNM seems to be affected by the condition number. For this reason, RS-RNM also needs more iterations to find a first-order stationary point.\\
For the \textit{duke} dataset, we observe that CD converges extremely fast in few iterations. Since the per-iteration cost of CD is much lower compared to SSCN, this results in much faster convergence for CD. However, CD converges very slowly to what seems to be a suboptimal solution for the \textit{madelon} dataset, while for the \textit{gisette} dataset the fastest schedule is clearly SSCN which only samples $2\%$ of the coordinates. This set of experiments highlights some interesting trade-offs between CD and SSCN. In the case of simpler objective functions, CD demonstrates a clear advantage. However, this advantage can diminish rapidly \textbf{when dealing with more intricate objective functions where SSCN tends to be more efficient.}
We believe that these experiments highlight that our proposed method is a competitive subspace Newton method in practice.

%% file: 05_conclusion.tex
\section{LIMITATIONS}
\label{sec:limitation}

In terms of theory, we emphasize that the adaptive scheme studied in Section~\ref{sec:adaptive_sampling} is developed assuming we have access to the partial derivatives
of the expected objective function. Thus, if the dataset contains a large number of data points, one might need to resort to a stochastic approximation, which would require adapting the analysis to work with high probability. This could be an interesting venue for future work where one can use both sampling of coordinates and datapoints in the same algorithm.

Based on experimental observations (refer to, for example, \cref{fig:logistic_regression_nonconv_CD_vs_SSCN_gisette} and \cref{fig:logistic_regression_nonconv_CD_vs_SSCN_duke}), we have found that the comparative advantage of SSCN over first-order CD is heavily contingent upon the complexity of the objective function. When dealing with a well-conditioned loss function, the lower per-iteration cost of CD results in notably faster convergence in terms of wall clock time. However, in scenarios where the loss function is ill-conditioned, SSCN manages to converge while CD struggles to do so within a reasonable timeframe. Thus, there appears to be a significant interest in studying the intrinsic ill-conditioning aspect of contemporary machine learning models. This would allow us to better understand the applicability of coordinate methods for such models.

\paragraph{Extension to arbitrary subspaces}
While this work focuses on coordinate subspaces, we believe that our overall analysis should also work for arbitrary subspaces $S_k$. 
The key step is to have analogous inequalities to \cref{lemma:concentration_bounds_gradient_norm_and_hessian} when $S_k$ is an arbitrary subspace. In this case, we just need to substitute $\nabla f(x_k)|_{S_k} \in \R^\tau$ with the vector $\nabla f_{S_k}(x_k) := S_k^\top \nabla f(x_k)$,
and replace $\nabla^2 f(x_k)|_{S_k} \in \R^{\tau \times \tau}$ with the matrix $\nabla^2 f(x_k)_{S_k} := S_k^{\top} \nabla^2 f(x_k) S_k$,
where the 'thin' matrix $S_k \in \R^{n \times \tau}$ describes the subspace. In the case when $S_k$ is a submatrix of the identity matrix, it corresponds to the coordinate subspace and our current method. Then, we can use the analogous inequality to \cref{lemma:concentration_bounds_gradient_norm_and_hessian} as the definition of 'effective subspace size' $\tau$, and our results will remain valid. However, we would like to emphasize that arbitrary subspaces are not that important from a practical perspective, and coordinate selection is typically the only viable choice in practice. This is why we focused mainly on this case in the paper. For all other choices, it is also not clear how the gradients and Hessian should be estimated effectively. This still remains an open question.

\section{CONCLUSION}
We analyzed the convergence rate of SSCN for the class of twice differentiable non-convex functions. From a theoretical perspective, to the best of our knowledge, we are the first to prove the global convergence guarantees of the subspace Newton method for a wide family of non-convex problems, ensuring that 
\begin{enumerate}
    \itemsep-0.1mm 
    \item for an arbitrary fixed $\tau$ our global rates interpolate between the rate of coordinate descent and the cubic Newton rate and recover the cubic Newton rate for full coordinate sampling and
    \item by using our novel adaptive sampling scheme for $\tau$, SSCN achieves rates that are provably better than previously known. 
\end{enumerate}
Furthermore, we have observed empirically that a more straightforward exponential schedule produces favorable results. 
Overall, our experiments demonstrated that one can sample a fraction of the coordinates while observing fast convergence. This results in significant computational gains compared to the vanilla cubic regularization algorithm. There are various interesting extensions to consider such as the use of importance sampling and combining coordinate sampling with datapoint sampling in the case of finite-sum objective functions.

\section*{Acknowledgments}

ND is supported by the Swiss State Secretariat for
Education, Research and Innovation (SERI) under contract number 22.00133.

%% file: 06_appendix.tex
\appendix
\onecolumn
\textbf{\Large APPENDIX}

{
  \tableofcontents
}

\newpage

\section{PROOF OF AUXILIARY LEMMAS}

\subsection{Model optimality}
\label{SectionAppendixOptimality}

As mentioned earlier, in our analysis, we will assume that each iteration 
is performed by computing $\h^*_k = \arg\min_{\h \in \R^n} \bigl[ 
\bar{m}(\h) := \bar{m}_{\x_k, \Qm_k, S_k, M}(\h) \bigr]$
exactly, and then we set $\x_{k + 1} = \x_k + \h_k^*$. Of course, in practice the algorithm needs to solve
the subproblem only along coordinates in $S_k$, which has dimensionality $\tau$.
Note that in general (for non-convex objective functions), the model $\bar{m}(\h)$ can also be non-convex.
However, its \textit{global minimum} is always well-defined and can be found efficiently by employing
standard techniques from linear algebra (see discussion in the end of the previous section).
Prior work~\citep{cartis2011adaptive} has shown that it is possible to retain the remarkable properties of the cubic regularization algorithm with an inexact model minimizer. For our purpose, 
we will simply rely on first-order and second-order optimality, i.e. 
$\nabla \bar{m}(\h^*_k) = 0$ and $\nabla^2 \bar{m}(\h_k^*) \succeq 0$, which respectively imply:
\beq \label{eq:first_order_optimality}
\ba{rcl}
    \nabla f(\x_k)_{[\SM_k]} + {\Qm_k}_{[\SM_k]} \h^*_k + \frac{M}{2} \|\h^*_k\| \cdot \h^*_k & = & 0, \\[10pt]
    {\Qm_k}_{[\SM_k]}
    + \frac{M}{2} \|\h^*_k\| \cdot \Im + \frac{M}{2 \| \h^{*}_k \|} \h^*_k (\h^*_k)^\top & \succeq & 0.
\ea
\eeq

Based on these optimality properties, we can establish the following lemmas, which will serve as the foundational bases for our main convergence theorem.

\begin{restatable}{proposition}{SecondOrderOptimality}
\label{prop:second_order_optimality}
For all global minimizers $\h^*_k$ of $\bar{m}(\h) = \bar{m}_{\x_k, \Qm_k, \SM_k, M}(\h)$ over $\R^n$ it holds that 

\beq \label{eq:pos_semidefiniteness_of_expression}
\ba{rcl}
    {\Qm_k}_{[\SM_k]} + \frac{M}{2} \| \h^*_k \| \cdot \Im & \succeq & 0,
\ea
\eeq
where $\Im$ denotes the identity matrix.
\end{restatable}
\begin{proof}
    This proof follows closely the proof of Theorem 3.1 in \cite{cartis2011adaptive}.
    From the second-order necessary optimality conditions at $\h^*_k$ we have

    $$
    \ba{rcl}
        \left< \left( {\Qm_k}_{[\SM_k]}  
        + \frac{M}{2} \|\h^*_k\| \cdot \Im 
        + \frac{M}{2 \| \h^*_k \| } \h^*_k (\h^*_k)^\top \right) \w, \w \right> 
        & \geq & 0 
    \ea
    $$
    for all vectors $\w \in \R^n$.

    If $\h^*_k = 0$, we immediately have the result. Thus, we only need to consider $\h^*_k  \neq 0$. 
   There are two cases to analyse. Firstly, suppose that $\la \w, \h^*_k \ra$ = 0. 
   Then it immediately follows  
   \beq
   \label{eq:w^Ts_eq_0}
   \ba{rcl}
       \left< \left( {\Qm_k}_{[\SM_k]}
       + \frac{M}{2} \|\h^*_k\| \cdot \Im \right) \w, \w \right>
       & \geq & 0 
       \qquad 
       \text{for all } \w \in \R^n \text{ s.t. } \la \w, \h^*_k \ra = 0.
   \ea
   \eeq
   It remains to consider vectors $\w$ for which $\la \w,  \h^*_k \ra \neq 0$. Since $\w$ and $\h^*_k$ are not orthogonal, the line $\h^*_k + \alpha \w$ intersects the ball of radius $\|\h^*_k\|$ at two points, $\h^*_k$ and $\u^*_k \neq \h^*_k$, and thus 
   \beq
    \label{eq:u_star_norm_=_y_star_norm}
   \ba{rcl}
       \| \u^*_k\| & = & \|\h^*_k\|.
   \ea
   \eeq
   Let $\w^*_k = \u^*_k - \h^*_k$, and note that $\w^*_k$ is parallel to $\w$, thus
   $\w = \beta \w^*_k$ for some $\beta \not= 0$.
   Since $\h^*_k$ is a global minimizer of $\bar{m}(\h)$, we have that 
   \beq \label{eq:second_order_cond}
   \ba{rcl}
       0 &\leq&  \bar{m}(\u^*_k) - \bar{m}(\h^*_k) \\
       \\
        &=& \langle \nabla f(\x_k)_{[S_k]}, (\u^*_k - \h^*_k) \rangle  
        + \frac{1}{2} \langle {\Qm_k}_{[\SM_k]} \u^*_k, \u^*_k \rangle 
        - \frac{1}{2} \langle {\Qm_k}_{[\SM_k]}  \h^*_k, \h^*_k \rangle \\
        \\
        & & \qquad 
        + \; \frac{M}{6} (\| \u^*_k \|^3 - \| \h^*_k \|^3) \\
        \\
         &\overset{\eqref{eq:u_star_norm_=_y_star_norm}}{=}&
         
         \langle \nabla f(\x_k)_{[S_k]}, (\u^*_k - \h^*_k) \rangle  
         + \frac{1}{2} \langle {\Qm_k}_{[\SM_k]} \u^*_k, \u^*_k \rangle 
         - \frac{1}{2} \langle {\Qm_k}_{[\SM_k]}  \h^*_k, \h^*_k \rangle.
    \ea
    \eeq
   But \eqref{eq:first_order_optimality} gives that
   \beq \label{eq:first_order_cond}
    \ba{rcl}
         \langle \nabla f(\x_k)_{[S_k]}, (\u^*_k - \h^*_k) \rangle 
         & = & \langle {\Qm_k}_{[\SM_k]} \h^*_k, (\h^*_k - \u^*_k) \rangle 
            + \frac{M}{2} \|\h^*_k\| \langle \h^*_k - \u^*_k, \h^*_k \rangle.
   \ea
   \eeq

   Further note that from \eqref{eq:u_star_norm_=_y_star_norm} it follows that 
   \beq \label{eq:w^*_norm}
   \ba{rcl}
       \langle \h^*_k - \u^*_k, \h^*_k \rangle 
       & = & \frac{1}{2} \la \h^*_k, \h^*_k \ra 
        + \frac{1}{2} \la \u^*_k, \u^*_k \ra 
        - \la \u^*_k, \h^*_k \ra 
       \;\; = \;\; \frac{1}{2} \langle \w^*_k, \w^*_k \rangle.
   \ea
   \eeq
   
   Now plugging \eqref{eq:w^*_norm} in \eqref{eq:first_order_cond}, and then plugging it in \eqref{eq:second_order_cond} we get 
   \beq \label{eq:w^Ts_neq_0}
   \ba{rcl}
    0 &\leq& 
    \frac{M}{4} \| \h^*_k \| \langle \w^*_k, \w^*_k \rangle 
    + \frac{1}{2}  \langle {\Qm_k}_{[\SM_k]} \h^*_k, \h^*_k \rangle 
    - \langle {\Qm_k}_{[\SM_k]} \h^*_k, \u^*_k \rangle \\
    \\
    & & \qquad 
    + \; \frac{1}{2} \langle {\Qm_k}_{[\SM_k]} \u^*_k, \u^*_k \rangle  \\
    \\
    &=&  
    \frac{M}{4} \| \h^*_k \| \langle \w^*_k, \w^*_k \rangle 
        + \frac{1}{2}\langle {\Qm_k}_{[\SM_k]} (\u^*_k - \h^*_k), (\u^*_k - \h^*_k) \rangle \\
    \\
    &\overset{\w_k^* = \u_k^* - \h_k^*}{=}& 
    \frac{1}{2}\left\langle \left( {\Qm_k}_{[\SM_k]} 
    + \frac{M}{2} \|\h^*_k\| \cdot \Im \right)(\u^*_k - \h^*_k), (\u^*_k - \h^*_k)\right\rangle \\
    \\
    &=&
    \frac{1}{2\beta}\left\langle \left( {\Qm_k}_{[\SM_k]} 
    + \frac{M}{2} \|\h^*_k\| \cdot \Im \right)\w, \w\right\rangle.
   \ea
   \eeq
   Therefore, it holds for any $\w \in \R^n$ such that $\la \w, \h^*_k \ra \neq 0$.
   Finally we can conclude from \eqref{eq:w^Ts_eq_0} and \eqref{eq:w^Ts_neq_0} that 
   $$
   \ba{rcl}
       {\Qm_k}_{[\SM_k]} + \frac{M}{2} \|\h^*_k\| \cdot \Im & \succeq & 0,
   \ea
   $$
   which completes the proof.
\end{proof}

\begin{restatable}{lemma}{GradTimesS}
    For all global minimizers $\h^*_k$ of $\bar{m}(\h) = \bar{m}_{\x_k, \Qm_k, \SM_k, M}(\h)$ over $\R^n$ it holds that 
    \beq \label{eq:grad_x_times_s^*_nonpos}
    \ba{rcl}
        \langle \nabla f(\x_k)_{[\SM_k]}, \h_k^* \rangle & \leq & 0.
    \ea
    \eeq
\end{restatable}
\begin{proof}
    Multiplying \eqref{eq:pos_semidefiniteness_of_expression} twice with $\h_k^*$ and multiplying \eqref{eq:first_order_optimality} once with $\h_k^*$ we get
    $$
    \ba{rcl}
        \langle {\Qm_k}_{[\SM_k]}  \h_k^*,\h_k^* \rangle + \frac{M}{2} \|\h_k^*\|^3 
        & \geq & 0, \\
        \\
        \langle \nabla f(\x_k)_{[S_k]}, \h_k^* \rangle + {\Qm_k}_{[\SM_k]}  \h_k^*, \h_k^*\rangle 
        + \frac{M}{2}\| \h_k^* \|^3 & = & 0.
    \ea
    $$
\end{proof}

Note that \eqref{eq:pos_semidefiniteness_of_expression}
is a stronger version of the standard second-order
optimality condition \eqref{eq:first_order_optimality},
which takes additionally into account that $\h^{*}$
is a \textit{global minimum}.

\ModelDecrease*

\begin{proof}[Proof Lemma~\ref{lemma:model_decrease}]
By \cref{assumption:lipschitz_hessian},%
we have for all $\x, \y \in \R^n$:
$$
\ba{rcl}
|f(\y) - f(\x) - \langle \nabla f(\x), \y - \x \rangle - \frac12 \langle \nabla^2 f(\x) (\y - \x), \y - \x \rangle | & \leq & \frac{L_2}{6} \| \y - \x \|^3.
\ea
$$

This implies that,
\beq \label{eq:arbitrary_rel_f_m}
\ba{rcl}
f(\x_{k+1}) &\leq& 
f(\x_k) + \langle \nabla f(\x_k), \h_k^* \rangle 
+ \frac12 \langle \nabla^2 f(\x_k) \h_k^*, \h_k^* \rangle + \frac{L_2}{6} \| \h_k^* \|^3  \\
\\
&\leq& f(\x_k) + \langle \nabla f(\x_k), \h_k^* \rangle + \frac12 \langle \nabla^2 f(\x_k) \h_k^*, \h_k^* \rangle + \frac{M}{6} \| \h_k^* \|^3  \\
\\
&=& f(\x_k) + \langle \nabla f(\x_k)_{[S_k]}, \h_k^* \rangle + \frac12 \langle \nabla^2 f(\x_k)_{[S_k]} \h_k^*, \h_k^* \rangle + \frac{M}{6} \| \h_k^* \|^3  \\
\\
&=& \bar{m}_{\x_k, \nabla^2 f(\x_k), S_k, M}(\h_k^*),
\ea
\eeq
where we used $M \geq L_2$ in the second inequality.

By first-order optimality of the cubic model, recall that
\begin{equation*}
\ba{rcl}
\nabla f(\x_k)_{[\SM_k]} + \nabla^2  f(\x_k)_{[\SM_k]} \h^*_k + \frac{M}{2} \|\h^*_k\| \cdot \h^*_k 
& = & 0.
\ea
\end{equation*}

Taking the inner product with $\h^*_k$, we obtain
\begin{equation} \label{eq:arbitrary:hessian}
\ba{rcl}
\langle \nabla^2  f(\x_k)_{[\SM_k]} \h^*_k, \h^*_k \rangle 
& = & 
-\langle \nabla f(\x_k)_{[\SM_k]}, \h^*_k \rangle - \frac{M}{2} \|\h^*_k\|^3.
\ea
\end{equation}

Combining Eq.~\eqref{eq:arbitrary_rel_f_m} and~\eqref{eq:arbitrary:hessian}, we obtain
\beq \label{eq:arbitrary_decrease_f}
\ba{rcl}
f(\x_{k+1}) - f(\x_k) &\leq& \frac12 \langle \nabla f(\x_k)_{[S_k]}, \h_k^* \rangle - \frac{M}{12} \| \h_k^* \|^3 \nonumber \\
\\
&\leq& - \frac{M}{12} \| \h_k^* \|^3,
\ea
\eeq
where the last inequality holds since $\langle \nabla f(\x_k)_{[\SM_k]}, \h_k^* \rangle \leq 0$ by Eq.~\eqref{eq:grad_x_times_s^*_nonpos}.
\end{proof}

\section{CONVERGENCE OF A FIXED SAMPLING SCHEME} \label{sec:convergence_fixed_sampling}

In this section, we assume that 
in each iteration we sample a coordinate subspace $S_k$
of the same fixed sample size $\tau(S_k) \equiv \tau \in [n]$.
We establish the global convergence rates for our method,
for an arbitrary initialization $\x_0$.

\subsection{General convergence rate}

Let us prove our general result that allows for an approximate matrix $\Qm \approx \nabla^2 f(\x_k)_{[S_k]}$.

\TheoremAnyConvergenceQ*
\begin{proof}
First, we want to choose $M_k > 0$ such that the following progress condition is satisfied:
\beq
\label{M_k_Choice_Progress}
\ba{rcl}
f(\x_{k} + \h_k^{*}) & \leq & 
\bar{m}_{\x_k, \Qm_k, S_k, M_k}(\h_k^{*}).
\ea
\eeq
We can ensure that this condition holds for a sufficiently large value of $M_k$. Indeed,
by Lipschitzness of the Hessian, we have
$$
\ba{rcl}
f(\x_k + \h_k^{*}) & \leq &
f(\x_k) + \la \nabla f(\x_k)_{[S_k]} \h_k^{*} \ra
+ \frac{1}{2} \la \nabla^2 f(\x_k)_{[S_k]} \h_k^{*}, \h_k^{*} \ra
+ \frac{L_2}{6}\| \h_k^* \|^3 \\
\\
& = & 
\bar{m}_{\x_k, \Qm_k, S_k, M_k}(\h_k^*) 
+ \frac{1}{2} \la (\nabla^2 f(\x_k)_{[S_k]} - \mat{Q}_k) \h_k^{*}, \h_k^{*} \ra
+ \frac{L_2 - M_k}{6}\| \h_k^{*} \|^3 \\
\\
& \overset{\eqref{SigmaDef}}{\leq} &
\bar{m}_{\x_k, \Qm_k, S_k, M_k}(\h_k^*) + \frac{\sigma}{2}\| \h_k^* \|^2 + \frac{L_2 - M_k}{6}\| \h_k^{*} \|^3,
\ea
$$
and to satisfy~\eqref{M_k_Choice_Progress}, it is sufficient to have
\beq \label{Sufficient_M_k}
\ba{rcl}
\frac{M_k}{6}\| \h_k^* \|^3 & \geq & \frac{\sigma}{2}\|\h_k^*\|^2 + \frac{L_2}{6}\|\h_k^*\|^3.
\ea
\eeq
Note that the stationary condition for $\h_k^*$ gives
\beq \label{QStatCond}
\ba{rcl}
\nabla f(\x_k)_{[S_k]} + \mat{Q}_k \h_k^* + \frac{M_k}{2} \| \h_k^* \| \h_k^* & = & 0, \\
\\
\mat{Q}_k + \frac{M_k}{2} \mat{I} & \succeq & 0.
\ea
\eeq

Then, observe that
\begin{equation}
\mat{Q}_k \overset{\eqref{SigmaDef}}{\preceq }
\sigma \mat{I} + 
\nabla^2 f(\x_k)_{[S_k]} 
\;\; \preceq \;\; (\sigma + L_1)\mat{I},
\label{eq:bound_Q}
\end{equation}

where the second inequality is due to Assumption \ref{assumption:lipschitz_gradient}.\\
Denoting $r := \| \h_k^* \|$, 
we get
$$
\ba{rcl}
\| \nabla f(\x_k)_{[S_k]} \|^2
& \overset{\eqref{QStatCond}}{=} &
\la ( \mat{Q}_k + \frac{M_kr}{2} \mat{I}  )^2 \h_k^*, \h_k^* \ra
\;\; \overset{\eqref{eq:bound_Q}}{\leq} \;\;
(\sigma + L_1 + \frac{M_k r}{2})^2 r^2.
\ea
$$
Hence,
\beq \label{RIneq}
\ba{rcl}
\frac{M_k}{2} r^2 + (\sigma + L_1) r & \geq & \| \nabla f(\x_k)_{[S_k]} \|.
\ea
\eeq
Resolving the quadratic equation, we conclude from~\eqref{RIneq} that
\beq \label{M_k_square}
\ba{rcl}
M_k\| \h_k^* \| & = & M_k r \;\; \overset{\eqref{RIneq}}{\geq} \;\; 
\sqrt{ (\sigma + L_1)^2 + 2M_k\| \nabla f(\x_k)_{[S_k]} \|  } - \sigma - L_1.
\ea
\eeq
Therefore, we can estimate the left hand side in~\eqref{Sufficient_M_k} as follows,
$$
\ba{rcl}
\frac{M_k}{6}\| \h_k^* \|^3
& = & 
\frac{M_k}{12}\| \h_k^* \|^3 
+ \frac{M_k}{12}\| \h_k^* \|^3 \\
\\
& \geq & 
\frac{\sqrt{ (\sigma + L_1)^2 + 2M_k\| \nabla f(\x_k)_{[S_k]} \|  } - \sigma - L_1}{12}\| \h_k^* \|^2 
+ \frac{M_k}{12}\| \h_k^* \|^3,
\ea
$$
and to ensure the inequality in~\eqref{M_k_Choice_Progress}, it is sufficient to choose
any $M_k \geq 2 L_2 + \frac{(7 \sigma + L_1)^2}{2\| \nabla f(\x_k)_{[S_k]} \| }$.
Let us make the following simple choice:
\beq \label{M_k_Q_Choice}
\boxed{
\ba{rcl}
M_k & := & 2 L_2 + \frac{7^2 (\sigma + L_1)^2}{2\| \nabla f(\x_k)_{[S_k]} \| }.
\ea
}
\eeq
Note that without loss of generality we assume that $\| \nabla f(\x_k)_{[S_k]} \| > 0$.
Otherwise, the algorithm simply does not move for the current step,
since it appears to be already in a local optimum in the sampled subspace.

Therefore, for this choice of $M_k$, we have established, for arbitrary $\y \in \R^n$:
$$
\ba{rcl}
f(\x_{k+1}) & \leq & 
\bar{m}_{\x_k, \Qm_k, S_k, M_k}(\h_k^{*}) 
\;\; \leq \;\;
\bar{m}_{\x_k, \Qm_k, S_k, M_k}(\y - \x_k), 
\ea
$$
where the last inequality holds since $\h_k^{*}$ is a global minimum of the model.
Hence, for any $\y \in \R^n$, we have:
$$
\ba{rcl}
f(\x_{k + 1}) & \leq & f(\x_k) + \la \nabla f(\x_k)_{[S_k]}, \y - \x_k \ra
+ \frac{1}{2}\la {\Qm_k}_{[S_k]}(\y - \x_k), \y - \x_k \ra
+ \frac{M_k}{6}\|\y - \x_k\|^3 \\
\\
& \overset{\eqref{eq:bound_Q}}{\leq} &
f(\x_k) + \la \nabla f(\x_k)_{[S_k]}, \y - \x_k \ra
+ \frac{\sigma + L_1}{2} \| \y - \x_k \|^2
+ \frac{M_k}{6}\|\y - \x_k\|^3.
\ea
$$
Let us take $\y := \x_k - \alpha \nabla f(\x_k)_{[S_k]}$ (the coordinated descent
step), with some $\alpha > 0$. Then,
$$
\ba{rcl}
f(\x_{k + 1}) & \leq & f(\x_k) - \alpha \| \nabla f(\x_k)_{[S_k]} \|^2
+ \frac{\alpha^2 (\sigma + L_1)}{2}\| \nabla f(\x_k)_{[S_k]} \|^2
+ \frac{\alpha^3 M_k}{6} \| \nabla f(\x_k)_{[S_k]} \|^3,
\ea
$$
or, rearranging the terms,
$$
\ba{rcl}
f(\x_k) - f(\x_{k + 1})
& \geq & \alpha \| \nabla f(\x_k)_{[S_k]} \|^2 \cdot 
\Bigl(  1 - \frac{\alpha (\sigma + L_1)}{2} 
- \frac{\alpha^2 M_k}{6} \| \nabla f(\x_k)_{[S_k]}\|  \Bigr).
\ea
$$
Note that $\alpha > 0$ is an arbitrary stepsize. We choose it in a way to have
$$
\ba{rcl}
1 - \frac{\alpha (\sigma + L_1)}{2} 
- \frac{\alpha^2 M_k}{6} \| \nabla f(\x_k)_{[S_k]}\| & = & \frac{1}{2},
\ea
$$
which is achieved by finding the positive root of the quadratic equation.
We obtain, denoting $g := \| \nabla f(\x_k)_{[S_k]}\|$ and $\beta := L_1 + \sigma$:
$$
\ba{rcl}
\alpha & := & 
\frac{-\beta + \sqrt{\beta^2 + \frac{4}{3} M_k g} }{ \frac{2}{3} M_k g } 
\;\; = \;\; \frac{3\beta}{2M_k g} 
\cdot \Bigl( \sqrt{1 + \frac{4M_k}{3\beta^2}g} \; -\; 1 \Bigr) \\
\\
& = & 
\frac{3\beta}{2M_k g} \cdot \frac{1 + \frac{4M_k}{3\beta^2}g - 1}{\sqrt{1 + \frac{4M_k}{3\beta^2}g} \; + \; 1 } 
\;\; = \;\;
\frac{2}{\beta} \cdot \frac{1}{\sqrt{1 + \frac{4M_k}{3\beta^2}g} \; + \; 1 } \\
\\
& \geq & 
\frac{2}{\beta} \cdot \frac{1}{2 + \frac{2M_k}{3\beta^2} g},
\ea
$$
where in the last bound we used the inequality $\sqrt{1 + t} \leq 1 + \frac{t}{2}$, valid for $t \geq 0$. 

Thus, we obtain the following progress for one step of our method:
$$
\ba{rcl}
f(\x_k) - f(\x_{k + 1}) & \geq & 
\frac{1}{\beta} \| \nabla f(\x_k)_{[S_k]} \|^2
\cdot \frac{1}{ 2 + \frac{2 M_k}{3\beta^2} \| \nabla f(\x_k)_{[S_k]} \| } \\
\\
& = & \frac{1}{2 \beta} 
\frac{\| \nabla f(\x_k)_{[S_k]} \|^2}{1 + \frac{M_k}{3\beta^2} \| \nabla f(\x_k)_{[S_k]} \| } \\
\\
& \overset{\eqref{M_k_Q_Choice}}{=} & 
\frac{1}{2 \beta} 
\frac{\| \nabla f(\x_k)_{[S_k]} \|^2}{1 + \frac{7^2}{6} + \frac{2L_2}{3\beta^2} \| \nabla f(\x_k)_{[S_k]} \| } \\
\\
& \geq &
\frac{1}{2 \beta} 
\frac{\| \nabla f(\x_k)_{[S_k]} \|^2}{1 + \frac{7^2}{6} + \frac{2L_2}{3\beta^2} \| \nabla f(\x_k) \| },
\ea
$$
where the last inequality holds due to
the trivial observation,
$$
\ba{rcl}
\| \nabla f(\x_k)_{[S_k]} \|^2
& = & 
\sum\limits_{i \in S_k} (\nabla f(\x_k))_i^2
\;\; \leq \;\;
\sum\limits_{i = 1}^n (\nabla f(\x_k))_i^2
\;\; = \;\; \| \nabla f(\x_k) \|^2.
\ea
$$
Taking the expectation w.r.t. $S_k$, we get
$$
\ba{rcl}
\E_{S_k} \Bigl[ f(\x_k) - f(\x_{k + 1}) \Bigr] 
&\geq& 
\frac{1}{2 \beta} 
\frac{  \E_{S_k }\| \nabla f(\x_k)_{[S_k]} \|^2}{1 + \frac{7^2}{6} + \frac{2L_2}{3\beta^2} \| \nabla f(\x_k)  \| }
\;\; = \;\;
\frac{1}{2 \beta} 
\frac{  ( \frac{\tau}{n} ) \| \nabla f(\x_k) \|^2}{1 + \frac{7^2}{6} + \frac{2L_2}{3\beta^2} \| \nabla f(\x_k)  \| }.
\ea
$$

The last equality holds because
\begin{equation}
\E_S \| \nabla f(\x_k)_{[S_k]} \|^2 = \sum\limits_{i=1}^n \E_S (\nabla f(\x_k)^{(i)} 1_{i \in S_k})^2 = \frac{\tau(S_k)}{n} \| \nabla f(\x_k) \|^2,
\end{equation}
where we used $\E_S[(1_{i \in S})^2] = \E_S[(1_{i \in S})] = \frac{\tau(S)}{n}$ if $S$ is uniformly sampled. 

Assuming that $\|\nabla f(\x_k)\| \geq \varepsilon$ and using
monotonicity of the last expression w.r.t $\|\nabla f(\x_k)\|$, we obtain
$$
\ba{rcl}
\E_{S_k}\Bigl[ f(\x_k) - f(\x_{k + 1}) \Bigr] 
& \geq & 
\frac{ (\frac{\tau}{n} ) \varepsilon^2}{2 \beta (1 + \frac{7^2}{6})  + \frac{4 L_2}{3\beta} \varepsilon }.
\ea
$$
Taking the full expectation and telescoping this bound for $K$ iterations,
we obtain
$$
\ba{rcl}
f(\x_0) - f^{\star} & \geq & 
f(\x_0) - \E f(\x_{K})
\;\; \geq \;\;
\frac{ (\frac{\tau}{n} ) \varepsilon^2 K}{2 \beta (1 + \frac{7^2}{6})  + \frac{4 L_2}{3\beta} \varepsilon },
\ea
$$
which leads to the following global complexity:
$$
\ba{rcl}
K & \leq & 
\frac{n}{\tau}  \cdot \frac{2(1 + \frac{7^2}{6}) \beta (f(\x_0) - f^{\star})}{\varepsilon^2} + \frac{n}{\tau} \cdot \frac{4 L_2 (f(\x_0) - f^{\star})}{3 \beta \varepsilon}.
\ea
$$

\end{proof}

\subsection{The power of second-order information}

In this section, we use exact second-order information, $\Qm_k = \nabla^2 f(\x_k)_{[S_k]}$. {For better readability, let us introduce the short notation $\tilde{m}(\h) := \tilde{m}_{\x_k,S_k,M_k}(\h) := \bar{m}_{\x_k,S_k, \nabla^2 f(\x_k)_{[S_k]},M_k}(\h)$.}
We show that in this case, it is possible to prove better convergence rates
that interpolate between CD and the full Cubic Newton method.

By Lemma~\ref{lemma:concentration_bounds_gradient_norm}, we have that
\beq \label{GradBound}
\ba{rcl}
\E_S \| \nabla f(\x) - \nabla f(\x)_{[S]} \| 
& \leq & \sqrt{1 - \frac{\tau}{n}} \| \nabla f(\x) \|.
\ea
\eeq
And from the proof of Lemma~\ref{lemma:concentration_bounds_hess_vec}, we get
\beq \label{HessBound}
\ba{rcl}
\E_S \| \nabla^2 f(\x) - \nabla^2 f(\x)_{[S]} \|^2_2
& \leq & (1 - p') \| \nabla^2 f(\x) \|_F^2,
\ea
\eeq
where $p' = \tfrac{\tau^2}{n^2}$.

Let us repeat the analysis of one step of our method, employing
bounds~\eqref{GradBound} and~\eqref{HessBound} directly.
The goal is to refine Lemma~\ref{lemma:bounding_full_grad_by_s_k}.
We fix $\x_k$ and consider one step of the method:
\beq \label{StepDef}
\ba{rcl}
\h_k^{*} & = & \argmin_{\h} {\tilde{m}_{\x_k,S_k,M_k}(\h)}
\ea
\eeq
for an arbitrary $S_k \subset [n]$ and $M \geq L_2$.
Then by first order optimality of $h_k^*$, the new gradient norm is bounded as
$$
\ba{rcl}
\| \nabla f(\x_k + \h_k^{*} ) \|
& \leq & 
\| \nabla f(\x_k) + \nabla^2 f(\x_k) \h_k^* \|
+ \frac{L_2}{2}\| \h_k^* \|^2 \\
\\
& \leq &
\| \nabla f(\x_k) - \nabla f(\x_k)_{[S_k]} \| 
+ \| (\nabla^2 f(\x_k) - \nabla^2 f(\x_k)_{[S_k]}) \h_k^* \| \\
\\
& & \quad + \; \frac{L_2}{2}\| \h_k^* \|^2 
+ \| \nabla f(\x_k)_{[S_k]} + \nabla^2 f(\x_k)_{[S_k]} \h_k^* \| \\
\\
& \overset{\eqref{eq:first_order_optimality}}{=} & 
\| \nabla f(\x_k) - \nabla f(\x_k)_{[S_k]} \| 
+ \| (\nabla^2 f(\x_k) - \nabla^2 f(\x_k)_{[S_k]}) \h_k^* \| \\
\\
& & \quad + \;
\frac{L_2 + M}{2}\| \h_k^* \|^2 \\
\\
& \leq & 
\| \nabla f(\x_k) - \nabla f(\x_k)_{[S_k]} \| 
+ \| \nabla^2 f(\x_k) - \nabla^2 f(\x_k)_{[S_k]} \| \cdot \| \h_k^* \| \\
\\
& & \quad + \;
\frac{L_2 + M}{2}\| \h_k^* \|^2.
\ea
$$
Now, we take (the full) expectation $\E[\cdot]$ and use
the Cauchy-Schwartz inequality for random variables:
$$
\ba{rcl}
\E \| \nabla f(\x_k + \h_k^*) \|
& \leq & 
\E \| \nabla f(\x_k) - \nabla f(\x_k)_{[S_k]} \|  \\
\\
& & \qquad  + \; 
\sqrt{ \E\bigl[ \| \nabla^2 f(\x_k) - \nabla^2 f(\x_k)_{[S_k]} \|^2\bigr]
\cdot
\E\bigl[\| \h_k^* \|^2 \bigr] } \\
\\
& & \qquad + \;
\frac{L_2 + M}{2} \E[ \| \h_k^* \|^2 ] \\
\\
& \overset{\eqref{GradBound}, \eqref{HessBound}}{\leq} &
\sqrt{1 - \frac{\tau}{n}} \cdot \E \| \nabla f(\x_k) \| \\
\\
& & 
\qquad 
+ \; \sqrt{ (1 - p') \E \| \nabla^2 f(\x_k) \|_F^2 \cdot \E[ \|\h_k^*\|^2 ]  } 
+ \frac{L_2 + M}{2} \E[ \| \h_k^* \|^2 ] \\
\\
& \overset{(*)}{\leq} & 
\sqrt{1 - \frac{\tau}{n}} \cdot \E \| \nabla f(\x_k) \|
+ \frac{(1 - p') \E \| \nabla^2 f(\x_k) \|_F^2}{2M}
+ \frac{L_2 + 2M}{2} \E[ \| \h_k^* \|^2 ] \\
\\
& \leq & 
\sqrt{1 - \frac{\tau}{n}} \cdot \E \| \nabla f(\x_k) \|
+ \frac{(1 - p') \E \| \nabla^2 f(\x_k) \|_F^2}{2M}
+ \frac{3M}{2} \E[ \| \h_k^* \|^2 ],
\ea
$$
where we used Young's inequality in $(*)$.

Since the function $g(x) = x^{3/2}, x \geq 0$ is convex, we have,
for an arbitrary $\lambda \in (0, 1)$, and $a, b\geq 0$:
\beq \label{Conv32}
\ba{rcl}
(a + b)^{3/2} & = & 
\frac{1}{\lambda^{3/2}} 
\Bigl( \lambda a 
+ (1 - \lambda)\bigl[ \frac{\lambda}{1 - \lambda}(b)\bigr] \Bigr)^{3/2} \\
\\
& \leq & 
\frac{1}{\lambda^{1/2}} a^{3/2} + \frac{1}{(1 - \lambda)^{1/2}} b^{3/2}.
\ea
\eeq
Let us fix some $\lambda_* \in (0, 1)$, and use inequality~\eqref{Conv32}
with $\lambda := \lambda_* \in (0, 1)$, obtaining:
$$
\ba{rcl}
& & \!\!\!\!\!\!\!\!\!\!\!\!\!\!\!
(\E \| \nabla f(\x_k + \h_k^{*}) \| )^{3/2} \\
\\
& \leq &
\frac{1}{\lambda_*^{1/2}} \bigl[1 - \frac{\tau}{n} \bigr]^{3/4} ( \E_S \| \nabla f(\x_k) \| )^{3/2}
+ \bigl[\frac{1}{1 - \lambda_*}\bigr]^{1/2}
\cdot \Bigl( 
\frac{(1 - p') \E \| \nabla^2 f(\x_k) \|_F^2}{2M}
+ \frac{3M}{2} \E[ \| \h_k^* \|^2 ]
\Bigr)^{3/2} \\
\\
& \leq & 
\frac{1}{\lambda_*^{1/2}} 
\bigl[1 - \frac{\tau}{n} \bigr]^{3/4} ( \E \| \nabla f(\x_k) \| )^{3/2}
+ \bigl[\frac{2}{1 - \lambda_*}\bigr]^{1/2} \cdot
\frac{(1 - p')^{3/2}}{(2M)^{3/2}} \E \| \nabla^2 f(\x_k) \|_F^3 \\
\\
& & \qquad
+ \; \bigl[\frac{2}{1 - \lambda_*}\bigr]^{1/2}
\bigl[\frac{3M}{2}\bigr]^{3/2} \E[ \| \h_k^* \|^3 ]
\ea
$$
where in the last inequality we used~\eqref{Conv32} again, but
with $\lambda := \frac{1}{2}$, as well as Jensen's inequality for the expectation.
Hence, we obtain that
$$
\ba{rcl}
\E[ \|\h_k^*\|^3 ]
& \geq & 
\bigl[ \frac{1 - \lambda_*}{2} \bigr]^{1/2}
\bigl[ \frac{2}{3M} \bigr]^{3/2}
\Bigl(  (\E\| \nabla f(\x_{k + 1}) \|)^{3/2}  
- \frac{1}{\lambda_*^{1/2}} \bigl[ 1 - \frac{\tau}{n} \bigr]^{3/4} ( \E \| \nabla f(\x_k) \| )^{3/2}  \Bigr) \\
\\
& & \quad - \;
\frac{(1 - p')^{3/2}}{3^{3/2} M^3} \E \| \nabla^2 f(\x_k) \|_F^3.
\ea
$$
It remains to combine it with the progress of one step in terms
of the functional residual (Lemma~\ref{lemma:model_decrease}):
$$
\ba{rcl}
\E f(\x_k) - \E f(\x_{k + 1}) & \geq & \frac{M}{12} \E[ \|\h_k^*\|^3 ]
\ea
$$
We have proved the following inequality.
\begin{lemma}
For one step of our method, with an arbitrary $S_k \subset [n]$
of size $\tau_k = |S_k|$, $1 \leq \tau_k \leq n$, and $M \geq L_2$, it holds, for any $\lambda_* \in (0, 1)$:
\beq \label{NewGradNewBound}
\ba{cl}
&\E f(\x_k) - \E f(\x_{k + 1}) \\
\\
& \geq \;
\frac{1}{12}
\bigl[ \frac{1 - \lambda_*}{2 M} \bigr]^{1/2}
\bigl[ \frac{2}{3} \bigr]^{3/2}
\Bigl(  (\E\| \nabla f(\x_{k + 1}) \|)^{3/2}  
- \frac{1}{\lambda_*^{1/2}}\bigl[ 1 - \frac{\tau_k}{n} \bigr]^{3/4} ( \E \| \nabla f(\x_k) \| )^{3/2}  \Bigr) \\
\\
& \qquad - \;
\frac{(1 - p_{2, k})^{3/2}}{12 \cdot 3^{3/2} M^2} \E \| \nabla^2 f(\x_k) \|_F^3,
\ea
\eeq
where $p_{2,k} = \frac{\tau_k(\tau_k - 1)}{n(n - 1)}$.
\end{lemma}

Let us denote by $L_{1}$ the Lipschitz constant of the gradient,
which provides us a uniform bound
for the Hessian in Frobenius norm: $\| \nabla^2 f(\x) \|_F \leq  \sqrt{n} L_{1}$, $\forall \x$. 

Then, using the same subset size $\tau_k \equiv \tau$ for all iterations,
we can telescope \eqref{NewGradNewBound} as follows:
$$
\ba{rcl}
f(\x_0) - f^{\star} & \geq & 
\frac{1}{12}
\bigl[ \frac{1 - \lambda_*}{2 M} \bigr]^{1/2}
\bigl[ \frac{2}{3} \bigr]^{3/2}
\sum\limits_{i = 1}^k
\Bigl[ 
(\E\| \nabla f(\x_i) \|)^{3/2}
- \frac{1}{\lambda_*^{1/2}}\bigl[ 1 - \frac{\tau}{n} \bigr]^{3/4}(\E\| \nabla f(\x_{i - 1}) \|)^{3/2}
\Bigr] \\
\\
& &
\;\; - \;
\frac{(1 - p_{2})^{3/2} n^{3/2} L_1^3}{12 \cdot 3^{3/2} M^2} k \\
\\
& = &
\frac{1}{12}
\bigl[ \frac{1 - \lambda_*}{2M} \bigr]^{1/2}
\bigl[ \frac{2}{3} \bigr]^{3/2} \Bigl(1 - \frac{1}{\lambda_*^{1/2}}\bigl[ 1 - \frac{\tau}{n} \bigr]^{3/4} \Bigr)
\sum\limits_{i = 1}^{k - 1} (\E \| \nabla f(\x_i) \|)^{3/2} \\
\\
& & \;\;
+  \;
\frac{1}{12}
\bigl[ \frac{1 - \lambda_*}{2M} \bigr]^{1/2}
\bigl[ \frac{2}{3} \bigr]^{3/2}( \E \| \nabla f(\x_k) \| )^{3/2} \\
\\
& & \;\;
-  \;
\frac{1}{12}
\bigl[ \frac{1 - \lambda_*}{2M\lambda_*} \bigr]^{1/2}
\bigl[ \frac{2}{3} \bigr]^{3/2} \bigl[1 - \frac{\tau}{n}\bigr]^{3/4} 
\| \nabla f(\x_0) \|^{3/2} 
\; - \; \frac{(1 - p_{2})^{3/2} n^{3/2} L_1^3 }{12 \cdot 3^{3/2} M^2} k.
\ea
$$
Therefore, we obtain the following convergence result.
\begin{theorem}\label{thm:convergence_interpolating_rate}
For any $M \geq L_2$, $\lambda_* \in (0, 1)$, and $1 \leq \tau \leq n$, it holds
$$
\ba{cl}
& \min\limits_{1 \leq i \leq k} ( \E \| \nabla f(\x_i) \|  )^{3/2} \\
\\
& \; \leq \; 
12 \bigl[ \frac{3}{2} \bigr]^{3/2} \bigl[ \frac{2}{1 - \lambda_*}  \bigr]^{1/2}
 \Bigl(1 - \frac{1}{\lambda_*^{1/2}}\bigl[ 1 - \frac{\tau}{n} \bigr]^{3/4} \Bigr)^{-1}
\cdot
\Bigl[  \frac{\sqrt{M}(f(\x_0) - f^{\star})}{k}  
+ \frac{(1 - p')^{3/2} n^{3/2}}{12 \cdot 3^{3/2} M^{3/2}} L_{1}^3   \Bigr]
+ \frac{R_0}{k},
\ea
$$
where 
$$
\ba{rcl}
R_0 & := & 
\frac{1}{\lambda_*^{1/2}}\bigl[1 - \frac{\tau}{n} \bigr]^{3/4}
\cdot
 \Bigl(1 - \frac{1}{\lambda_*^{1/2}}\bigl[ 1 - \frac{\tau}{n} \bigr]^{3/4} \Bigr)^{-1}
 \| \nabla f(\x_0) \|^{3/2}.
\ea
$$
\end{theorem}
The constant $R_0$ is a technical term which is not important.
Let us ignore it for simplicity.

Note that we have a freedom in choosing parameter $M \geq L_2$
and $\lambda_* \in (0, 1)$.
Let us assume that all gradients are sufficiently large, for some given
tolerance $\varepsilon > 0$:
$$
\ba{rcl}
\E \| \nabla f(\x_i) \| & \geq & \varepsilon, \qquad \forall 0 \leq i \leq k.
\ea
$$
We can choose $M \geq L_2$ such that
the constant term with $L_1$ be sufficiently small,
namely
$$
\ba{rcl}
12 \bigl[ \frac{3}{2} \bigr]^{3/2} \bigl[ \frac{2}{1 - \lambda_*}  \bigr]^{1/2}
 \Bigl(1 - \frac{1}{\lambda_*^{1/2}} \bigl[ 1 - \frac{\tau}{n} \bigr]^{3/4} \Bigr)^{-1}
 \cdot \frac{(1 - p')^{3/2} n^{3/2}}{12 \cdot 3^{3/2} M^{3/2}} L_{1}^3 
 & \leq & 
\bigl(\frac{\varepsilon}{2}\bigr)^{3/2}.
\ea
$$

E.g., this condition will be satisfied for the choice
\beq \label{MChoiceNew}
\boxed{
\ba{rcl}
M & := & L_2 + 
\bigl[ \frac{2}{1 - \lambda_*} \bigr]^{1/3}
\Bigl(1 - \frac{1}{\lambda_*^{1/2}}\bigl[1 - \frac{\tau}{n} \bigr]^{3/4} \Bigr)^{-2/3}
\frac{(1 - p')n}{\varepsilon} L_1^2.
\ea
}
\eeq

In this case, the number of iterations $k$ required to reach $\varepsilon$ accuracy is:
\beq \label{ComplexityNew}
\ba{rcl}
k & = & \mathcal{O}\Bigl( 
\bigl[ \frac{1}{1 - \lambda_*} \bigr]^{1/2}
\Bigl( 1 - \frac{1}{\lambda_*^{1/2}} \bigl[ 1 - \frac{\tau}{n} \bigr]^{3/4} \Bigr)^{-1}
\frac{\sqrt{M} (f(\x_0) - f^{\star})}{\varepsilon^{3/2}}
+ \frac{R_0}{\varepsilon^{3/2}}
\Bigr) \\
\\
& \overset{\eqref{MChoiceNew}}{=} &
\mathcal{O}\biggl(
\bigl[ \frac{1}{1 - \lambda_*} \bigr]^{1/2}
\Bigl( 1 - \frac{1}{\lambda_*^{1/2}} \bigl[ 1 - \frac{\tau}{n} \bigr]^{3/4} \Bigr)^{-1}
\cdot 
\frac{\sqrt{L_2}(f(\x_0) - f^{\star})}{\varepsilon^{3/2}} \\
\\
& &
\;
+ \;
\bigl[ \frac{1}{1 - \lambda_*} \bigr]^{2/3}
\Bigl( 1 - \frac{1}{\lambda_*^{1/2}}\bigl[1 - \frac{\tau}{n} \bigr]^{3/4} \Bigr)^{-4/3}
\cdot
n^{1/2}(1 - p')^{1/2}
\frac{L_1(f(\x_0) - f^{\star})}{\varepsilon^2}
+ \frac{R_0}{\varepsilon^{3/2}}
\biggr).
\ea
\eeq
It remains to choose $\lambda_* \in (0, 1)$.

Let us consider the following simple choice: $\lambda_* := 1 - \frac{\tau}{n}$.
Then,
$$
\ba{rcl}
1 - \frac{1}{\lambda_*^{1/2}}\bigl[1 - \frac{\tau}{n} \bigr]^{3/4}
& = &
1 - \bigl[1 - \frac{\tau}{n}\bigr]^{1/4}
\;\; = \;\;
1 - \frac{(n - \tau)^{1/4}}{n^{1/4}} 
\;\; = \;\;
\frac{n^{1/4} - (n - \tau)^{1/4}}{n^{1/4}} \\
\\
& \geq & 
\frac{[n - n + \tau]}{4n^{3/4}} \frac{1}{n^{1/4}}
\;\; = \;\;
\frac{\tau}{4n},
\ea
$$
where we used concavity of the function $\phi(x) := x^{1/4}, x \geq 0$,
which implies that, for any $x, y \geq 0$:
$$
\ba{rcl}
y^{1/4} \;\; = \;\; \phi(y) & \leq & \phi(x) + \phi'(x) (y - x)
\;\; = \;\;
x^{1/4} + \frac{y - x}{4 x^{3/4}}
\quad \Leftrightarrow \quad
x^{1/4} - y^{1/4} \;\; \geq \;\; \frac{x - y}{4x^{3/4}}.
\ea
$$
At the same time,
$$
\ba{rcl}
\frac{1}{1 - \lambda_*} & = & 
\frac{n}{\tau}.
\ea
$$
Therefore, for this choice of $\lambda_* := 1 - \frac{\tau}{n}$ we obtain the following 
interpolating complexity, valid for any $1 \leq \tau \leq n$:
\beq \label{Compl}
\boxed{
\ba{rcl}
k & = & 
\mathcal{O}
\Bigl(
\bigl[\frac{n}{\tau}\bigr]^{3/2}
\frac{\sqrt{L_2}(f(\x_0) - f^{\star})}{\varepsilon^{3/2}}
+ 
n^{1/2}(1 - p')^{1/2}
\bigl[ \frac{n}{\tau} \bigr]^{2} \frac{L_1 (f(\x_0) - f^{\star}}{\varepsilon^2}
\Bigr).
\ea
}
\eeq

\section{CONVERGENCE ADAPTIVE SCHEME - HIGH-PROBABILITY VERSION}
\label{sec:convergence_adaptive_sampling}

\subsection{General second-order convergence}

For instructional purposes, we first conduct an analysis under some general conditions for the quality of the approximation of the sampled coordinate gradients. In section~\ref{sec:app_adaptive_sampling}, we will demonstrate that one can remove this condition using an adaptive sampling scheme.

\begin{condition}[Condition only used in Theorem~\ref{RefTheoremAnyConvergenceQ}]
For $\epsilon_1, \epsilon_2 \geq 0$, for all $\x \in \R^n$ and any subset $\SM \subset [n]$ from $\cD$ with probability at least $1=\delta$,
\begin{equation}
\label{eq:concentration_bound_gradient_norm}
\| \nabla f(\x) - \nabla f(\x)_{[S]} \| \leq \delta^{-\frac12}\epsilon_1 \text{  and  } \| \nabla^2 f(\x) - \nabla^2 f(\x)_{[S]} \| \leq \delta^{-\frac12} \sqrt{\epsilon_2}.
\end{equation}
\label{condition:sampling}
\end{condition}

This condition is established in Lemma~\ref{lemma:concentration_bounds_gradient_norm_and_hessian}, with its proof provided in the subsequent two lemmas.
\begin{restatable}{lemma}{ConcentrationGradientNorm}
\label{lemma:concentration_bounds_gradient_norm}
    For any $\x \in \R^n$ and any subset $S \subset [n]$, we have under an uniform sampling scheme
    \begin{align}
        \E_S \| \nabla f(\x) - \nabla f(\x)_{[S]} \| \leq \sqrt{1 - \frac{\tau(S)}{n} } \| \nabla f(\x) \| =: \epsilon_1.
    \end{align}
    This implies a high-probability bound 
    \begin{align}
    \label{eq:high_prob_bound_gradient_concentration}
    \Pr \left(\| \nabla f(\x) - \nabla f(\x)_{[S]} \| \geq \delta^{-\frac12} \underbrace{\sqrt{1 - \frac{\tau(S)}{n}} \| \nabla f(\x) \|}_{= \epsilon_1} \right) \leq \delta.
    \end{align}
\end{restatable}

\begin{proof}
    For any $\x \in \R^n$, we have
    $$
    \ba{rcl}
        \| \nabla f(\x) - \nabla f(\x)_{[S]} \|^2 &=& \sum\limits_{i=1}^n (\nabla f(\x)^{(i)} - \nabla f(\x)^{(i)} 1_{i \in S})^2  \\
        \\
        &=& \sum\limits_{i=1}^n (\nabla f(\x)^{(i)})^2 + (\nabla f(\x)^{(i)} 1_{i \in S})^2 - 2 (\nabla f(\x)^{(i)})^2 1_{i \in S}.
    \ea
    $$
    
    Taking expectation over $S$, we get
    $$
    \ba{rcl}
        \E_S \| \nabla f(\x) - \nabla f(\x)_{[S]} \|^2 &=& 
        \sum\limits_{i=1}^n \E_S (\nabla f(\x)^{(i)} - \nabla f(\x)^{(i)} 1_{i \in S})^2  \\
        \\
        &=& \sum\limits_{i=1}^n (\nabla f(\x)^{(i)})^2 + \E_S (\nabla f(\x)^{(i)} 1_{i \in S})^2 - 2 \E_S (\nabla f(\x)^{(i)})^2 1_{i \in S}  \\
        \\
        &\stackrel{(i)}{=}& \sum\limits_{i=1}^n (\nabla f(\x)^{(i)})^2 (1 + \frac{\tau(S)}{n} - 2 \frac{\tau(S)}{n})  \\
        \\
        &=& (1 - \frac{\tau(S)}{n}) \| \nabla f(\x) \|^2,
    \ea
    $$
    where $(i)$ used $\E_S[(1_{i \in S})^2] = \E_S[(1_{i \in S})] = \frac{\tau(S)}{n}$.\\
    
    We conclude the expectation bound using Jensen's inequality,
    \begin{equation}
        \E_S \| \nabla f(\x) - \nabla f(\x)_{[S]} \| \leq (\E_S \| \nabla f(\x) - \nabla f(\x)_{[S]} \|^2)^{1/2} 
        \leq \sqrt{1 - \frac{\tau(S)}{n}} \| \nabla f(\x) \|.
    \end{equation}

    Finally, we can obtain a high-probability bound using Markov's inequality as follows:
    \begin{align}
    \Pr(\| \nabla f(\x) - \nabla f(\x)_{[S]} \| \geq \mu) = \Pr(\| \nabla f(\x) - \nabla f(\x)_{[S]} \|^2 \geq \mu^2) \leq \mu^{-2} \E_S \| \nabla f(\x) - \nabla f(\x)_{[S]} \|^2.
    \end{align}
    Setting $\mu = \delta^{-\frac12} \sqrt{1 - \frac{\tau(S)}{n}} \| \nabla f(\x) \|$ for $\delta \in (0,1)$ yields:
    \begin{align}
    \Pr \left(\| \nabla f(\x) - \nabla f(\x)_{[S]} \| \geq \delta^{-\frac12} \underbrace{\sqrt{1 - \frac{\tau(S)}{n}} \| \nabla f(\x) \|}_{= \epsilon_1} \right) \leq \delta.
    \end{align}
\end{proof}

\begin{restatable}{lemma}{ConcentrationHessian}
\label{lemma:concentration_bounds_hess_vec}
    For any $\x \in \R^n$ and any subset $S \subset [n]$, we have under an uniform sampling scheme
    \begin{equation}
        \E_S \| \nabla^2 f(\x)_{[S]} - \nabla^2 f(\x) \|_2 
        \leq \sqrt{1-p'} \| \nabla^2 f(\x) \|_F =: \sqrt{\epsilon_2},
    \end{equation}
    where $p' := \left(\frac{\tau(S)}{n}\right)^2$. This implies a high-probability bound
    \begin{align}
    \label{eq:high_prob_bound_hessian_concentration}
    \Pr \left(\| \nabla^2 f(\x) - \nabla^2 f(\x)_{[S]} \|_2 \geq \delta^{-\frac12} \underbrace{\sqrt{1 - p'} \| \nabla^2 f(\x) \|_F}_{= \sqrt{\epsilon_2}} \right) \leq \delta.
    \end{align}
\end{restatable}

\begin{proof}
    \begin{align}
        \E_S \| \nabla^2 f(\x)_{[S]} - \nabla^2 f(\x) \|_F^2
        &= \E_S \sum_{i,j} \big( (\nabla^2 f(\x))_{ij}^2 1^2_{i,j \in S} + (\nabla^2 f(\x))_{ij}^2 - 2 (\nabla^2 f(\x))_{ij}^2 1_{i,j \in S} \big) \\
        &= \sum_{i,j} (\nabla^2 f(\x))_{ij}^2 \E_S[1_{i,j \in S}^2 + 1 - 2 \cdot 1_{i,j \in S}] \\
        &= (1 - p') \sum_{i,j} (\nabla^2 f(\x))_{ij}^2 \\
        &= (1-p') \| \nabla^2 f(\x) \|_F^2,
    \end{align}
    where $p' := \E_S[1_{i,j \in S}] = p^2 = \left(\frac{\tau(S)}{n}\right)^2$.

    Since $\|\Am\|_2 \leq \|\Am\|_F$ for any matrix $\Am$, we get the desired result, again using Jensen's inequality: 
    \begin{equation*}
        \E_S \| \nabla^2 f(\x)_{[S]} - \nabla^2 f(\x) \|_2 
        \leq \sqrt{1-p'} \| \nabla^2 f(\x) \|_F.
    \end{equation*}
    We can derive a high probability bound by Markov's inequality since
    \begin{align}
    \nonumber
    \Pr(\| \nabla^2 f(\x)_{[S]} - \nabla^2 f(\x) \|_2 \geq \mu) = \Pr(\| \nabla^2 f(\x)_{[S]} - \nabla^2 f(\x) \|^2_2 \geq \mu^2) &\leq \frac{\E_S \| \nabla^2 f(\x)_{[S]} - \nabla^2 f(\x) \|^2_2}{\mu^2} \\
    &\leq \frac{\E_S \| \nabla^2 f(\x)_{[S]} - \nabla^2 f(\x) \|^2_F}{\mu^2}
    \end{align}

    Setting $\mu = \delta^{-\frac12} \sqrt{1 - \left(\frac{\tau(S)}{n}\right)^2} \| \nabla^2 f(\x) \|_F$ for $\delta \in (0,1)$ yields:
    \begin{align}
    \Pr \left(\| \nabla^2 f(\x) - \nabla^2 f(\x)_{[S]} \|_2 \geq \delta^{-\frac12} \underbrace{\sqrt{1 - p'} \| \nabla^2 f(\x) \|_F}_{= \sqrt{\epsilon_2}} \right) \leq \delta.
    \end{align}
\end{proof}

We are now ready to state the first main result of this section.

\begin{restatable}{theorem}{MainConvergenceBall}
\label{thm:convergence_ball}
    Let the sequence $\{\x_i\}$ be generated by $\x_{k+1} = \x_k + \arg\min_\h {\tilde{m}_{\x_k,S_k,M_k}(\h)}$ and let $M \geq L_2$. Assume that the objective function $f(\x)$ is bounded below:
\begin{align*}
    f(\x) \geq f^* \quad \forall \x \in \R^n.
\end{align*}
Let $\Delta_0 = f(\x_0) - f^*$, and define the following constants (dependent on $M$):
\begin{align*}
C &= \tfrac{12\sqrt{3}}{M} \left( \tfrac{2M+1}{2} \right)^{3/2}, D = 162 M^2.
\end{align*}
Then, under Condition~\ref{condition:sampling} we have with probability at least $1-\delta$,
\begin{equation*}
\min_{1 \leq j \leq k} \mu(\x_j) \leq \max(C, D) \frac{\Delta_0}{k} + \sqrt{3} \delta^{-3/4} \epsilon_1^{3/2} + 4 \delta^{-3/2}\epsilon_2^{3/2}.
\end{equation*}

\end{restatable}

Theorem~\ref{thm:convergence_ball} states that Algorithm~\ref{alg:SSCN} converges to an $\epsilon$-second-order stationary point at a rate of $\bigO(\epsilon^{-3/2}, \epsilon^{-3})$, up to a ball whose radius is determined by $\epsilon_1$ and $\epsilon_2$. We also expect that the constants in the bound could be made tighter even if it results in a somewhat less readable proof. However, the interesting aspect of this theorem is that it shows that we obtain the same convergence as cubic regularization up to a ball. Next, we turn our attention to characterizing how $\epsilon_1$ and $\epsilon_2$ depend on the number of sampled coordinates.

In order to prove the theorem, we will first prove two lemmas, Lemma \ref{lemma:bounding_full_grad_by_s_k} and~\ref{lemma:bounding_full_hessian_by_s_k}, that relate the gradient and the Hessian of the objective function $f$ with the norm of the step $\h_k^*$.

\begin{restatable}{lemma}{BoundingFullGradByS}
\label{lemma:bounding_full_grad_by_s_k}    
    For any $\x_k \in \R^n$, let $\h_k^* = \arg \min_\h m_3(\h; \x_k,\SM_k)$ for an arbitrary $\SM_k \subset [n] $ and let $M \geq L_2$. Then, the \emph{full} gradient norm $\| \nabla f(\x_k + \h_k^*) \|$ at the new iterate can be bounded with probability at least $1-\delta$ as
    \begin{equation*}
        \| \nabla f(\x_k + \h_k^*) \| \leq \frac{2M+1}{2} \| \h_k^* \|^2 + \delta^{-1/2} \epsilon_1 + \frac12 \delta^{-1} \epsilon_2.
    \end{equation*}
\end{restatable}
\begin{proof}

By the triangle inequality, we have:
\begin{align}\label{eq:bound_full_gradient_norm}
\| \nabla f(\x_k + \h_k^*) \| \leq \| \nabla f(\x_k + \h_k^*) - \nabla \phi(\h_k^*) \| +  \| \nabla \phi(\h_k^*) \|,
\end{align}
where $\nabla \phi(\h_k^*) = \nabla f(\x_k)_{[S_k]} + \nabla^2 f(\x_k)_{[S_k]} \h_k^*$.

From the first-order optimality condition in \eqref{eq:first_order_optimality} it follows  that $\| \nabla \phi(\h_k^*) \| = \frac{M}{2} \| \h_k^* \|^2$. Therefore we focus on the first term, for which
\beq \label{eq:bound_full_gradient_norm_first_term}
\ba{rcl}
& & \!\!\!\!\!\!\!\!\!\!\!\!\!\!\!\!\!\!
\| \nabla f(\x_k + \h_k^*) - \nabla f(\x_k)_{[S_k]} - \nabla^2 f(\x_k)_{[S_k]} \h_k^* \| \\
\\
&\leq&  \| \nabla f(\x_k + \h_k^*) - \nabla f(\x_k) - \nabla^2 f(\x_k) \h_k^* \|  \\
\\
& & \qquad + \; \| \nabla f(\x_k)_{[S_k]} - \nabla f(\x_k) \| + \| \nabla^2 f(\x_k)_{[S_k]} \h_k^* - \nabla^2 f(\x_k) \h_k^* \|  \\
\\
&\overset{\cref{assumption:lipschitz_hessian}}{\leq}&
\frac{L_2}{2} \| \h_k^* \|^2 + \| \nabla f(\x_k)_{[S_k]} - \nabla f(\x_k) \| + \| \nabla^2 f(\x_k)_{[S_k]} \h_k^* - \nabla^2 f(\x_k) \h_k^* \|.
\ea
\eeq

From Eq. \eqref{eq:bound_full_gradient_norm} and Eq. \eqref{eq:bound_full_gradient_norm_first_term} we have that 
$$
\ba{rcl}
\| \nabla f(\x_k + \h_k^*) \| 
&\leq& \frac{M+L_2}{2} \|\h_k^*\|^2 + \| \nabla f(\x_k)_{[S_k]} - \nabla f(\x_k) \| + \| \nabla^2 f(\x_k)_{[S_k]} \h_k^* - \nabla^2 f(\x_k) \h_k^* \| \\
\\
&\leq& \frac{M+L_2}{2} \|\h_k^*\|^2 + \| \nabla f(\x_k)_{[S_k]} - \nabla f(\x_k) \| + \| \nabla^2 f(\x_k)_{[S_k]} - \nabla^2 f(\x_k) \| \| \h_k^* \| \\
\\
&\stackrel{(i)}{\leq}& \frac{M+L_2}{2} \|\h_k^*\|^2 + \delta^{-1/2} \epsilon_1 + \sqrt{\delta^{-1} \epsilon_2} \| \h_k^* \| \\
\\
&\stackrel{(ii)}{\leq}& \frac{M+L_2}{2} \|\h_k^*\|^2 + \delta^{-1/2} \epsilon_1 + \frac12 \delta^{-1} \epsilon_2 + \frac12 \| \h_k^* \|^2 \\
\\
&\leq& \frac{2M+1}{2} \|\h_k^*\|^2 + \delta^{-1/2} \epsilon_1 + \frac12 \delta^{-1} \epsilon_2,
\ea
$$
where $(i)$ uses Eq.~\eqref{eq:high_prob_bound_gradient_concentration} and ~\eqref{eq:high_prob_bound_hessian_concentration}  and $(ii)$ uses Young's inequality for products.
\end{proof}

\begin{restatable}{lemma}{BoundingFullHessianByS}
\label{lemma:bounding_full_hessian_by_s_k}    
    For any $\x_k \in \R^n$, let $\h_k^* = \arg \min_\h m_3(\h; \x_k,\SM_k)$ for an arbitrary $\SM_k \subset [n] $. %
    Then, the smallest eigenvalue at the new iterate can be bounded with probability at least $1-\delta$ as
    \begin{equation}
    - \lambda_{\min}(\nabla^2 f(\x_{k} + \h_{k}^*)) \leq \frac{3M}{2} \|\h_k^*\| + \delta^{-1} \sqrt{\epsilon_2}.
    \end{equation}
\end{restatable}

\begin{proof}

Recall that
\begin{equation}
    \label{eq:lipschitz_hessian_submatrix}
    \| ( \nabla^2 f(\x) - \nabla^2 f(\y) ) \| \leq L_2 \| \x - \y \|,
\end{equation}

By Eq.~\eqref{eq:pos_semidefiniteness_of_expression}, $\nabla^2 f(\x_k)_{[S_k]} + \frac{M}{2} \| \h_k^* \| \cdot \Im \succeq 0$.
Therefore
$$
\ba{rcl}
\nabla^2 f(\x_{k} + \h_{k}^*) 
&\overset{\eqref{eq:lipschitz_hessian_submatrix}}{\succeq}& 
\nabla^2 f(\x_k) - L_2 \| \h_k^* \| \Im  \\
\\
&\overset{\eqref{eq:high_prob_bound_hessian_concentration}}{\succeq}& \nabla^2 f(\x_k)_{[S_k]} - \sqrt{\delta^{-1} \epsilon_2} \Im - L_2 \| \h_k^* \| \Im  \\
\\
&\overset{Prop. \ref{prop:second_order_optimality}}{\succeq}& 
-\left( \frac{1}{2}M + L_2 \right) \|\h_k^*\| \Im - \sqrt{\delta^{-1} \epsilon_2} \Im  \\
\\
&\succeq& -\frac{3M}{2} \|\h_k^*\| \Im - \sqrt{\delta^{-1} \epsilon_2} \Im,
\ea
$$
where we used $L_2 \leq M$.
Therefore, we have
\begin{equation}
\|\h_k^*\| \geq \frac{2}{3M} \left( - \lambda_{\min}(\nabla^2 f(\x_{k} + \h_{k}^*)) - \sqrt{\delta^{-1} \epsilon_2} \right),
\end{equation}
which implies
\begin{equation}
- \lambda_{\min}(\nabla^2 f(\x_{k} + \h_{k}^*)) \leq \frac{3M}{2} \|\h_k^*\| + \sqrt{\delta^{-1} \epsilon_2}.
\end{equation}

\end{proof}

Now we are ready to prove~\cref{thm:convergence_ball}.
\begin{proof}[Proof of~\cref{thm:convergence_ball}]

By the convexity of the function $g(x) = x^{3/2}, x \geq 0$, Jensen's inequality yields $(\sum_{i=1}^d x_i)^{3/2} \leq \sqrt{d} \sum_{i=1}^d (x_i)^{3/2} $.
Applied to the result of Lemma~\ref{lemma:bounding_full_grad_by_s_k}, we obtain
\beq \label{eq:bound_norm_sk}
\ba{rcl}
\| \nabla f(\x_k + \h_k^*) \|^{3/2} 
&\leq& \left( \frac{M+L_2+1}{2} \| \h_k^* \|^2 + \delta^{-1/2}\epsilon_1 + \delta^{-1}\epsilon_2 \right)^{3/2} \\
\\
&\leq& \sqrt{3} \left( \frac{M+L_2+1}{2} \right)^{3/2} \| \h_k^* \|^3 + \sqrt{3} \delta^{-3/4}\epsilon_1^{3/2} + \sqrt{3} \delta^{-3/2}\epsilon_2^{3/2} \\
\\
&\leq& \sqrt{3} \left( \frac{2M+1}{2} \right)^{3/2} \| \h_k^* \|^3 + \sqrt{3} \delta^{-3/4}\epsilon_1^{3/2} + \sqrt{3} \delta^{-3/2}\epsilon_2^{3/2}.
\ea
\eeq

Let $C_M = \frac{M}{12 \sqrt{3}} \left( \frac{2}{2M+1} \right)^{3/2}$. Using a telescoping argument, we have
$$
\ba{rcl}
    f(\x_0) - f^* 
    &\geq& \sum\limits_{i=0}^{k-1} \left( f(\x_i) - f(\x_{i+1}) \right) \\
    \\
    &\overset{\text{Lem.}~\ref{lemma:model_decrease}}{\geq}& 
    \sum\limits_{i=0}^{k-1} \frac{M}{12} \| \h_i^* \|^3 \\
    \\
    &\overset{\text{Lem.}~\ref{lemma:bounding_full_grad_by_s_k}}{\geq}&
  C_M \sum\limits_{i=0}^{k-1} \|  \nabla f (\x_i + \h_i^*) \|^{3/2} - \sqrt{3} k \delta^{-3/4} C_M \epsilon_1^{3/2} - \sqrt{3} k \delta^{-3/2} C_M \epsilon_2^{3/2} \\
  \\
    &\geq& C_M k \min\limits_{0 \leq i \leq k-1} \|  \nabla f (\x_i + \h_i^*) \|^{3/2} - \sqrt{3} k \delta^{-3/4} C_M \epsilon_1^{3/2} - \sqrt{3} k \delta^{-3/2} C_M \epsilon_2^{3/2}  \\
    \\
    &\geq& C_M k \min\limits_{1 \leq i \leq k} \|  \nabla f (\x_i) \|^{3/2} - \sqrt{3} k \delta^{-3/4} C_M \epsilon_1^{3/2} - \sqrt{3} k \delta^{-3/2} C_M \epsilon_2^{3/2}.
\ea
$$
Now rearranging for $\|\nabla f(\x_j)\|$ we get
\begin{equation}
\min_{1 \leq i \leq k} \|  \nabla f (\x_i) \|^{3/2} \leq \frac{1}{C_M k} (f(\x_0) - f^*) + \sqrt{3} \delta^{-3/4} \epsilon_1^{3/2} + \sqrt{3} \delta^{-3/2} \epsilon_2^{3/2}.
\end{equation}

We can also obtain a guarantee in terms of second-order optimality as follows. By Lemma~\ref{lemma:bounding_full_hessian_by_s_k}, we have $- \lambda_{\min}(\nabla^2 f(\x_{k} + \h_{k}^*)) \leq \frac{3M}{2} \|\h_k^*\| + \sqrt{\delta^{-1} \epsilon_2}$.
By the convexity of the function $g(x) = x^3$ on $\R^+$, Jensen's inequality yields $(\sum_{i=1}^d x_i)^3 \leq d^2 \sum_{i=1}^d (x_i)^3$, therefore
\begin{align}
& (- \lambda_{\min}(\nabla^2 f(\x_{k} + \h_{k}^*)))^3 \leq \frac{27 M^3}{2} \|\h_k^*\|^3 + 4 \delta^{-3/2}\epsilon_2^{3/2} \nonumber \\
\implies & \|\h_k^*\|^3 \geq \frac{2}{27 M^3} \left( (- \lambda_{\min}(\nabla^2 f(\x_{k} + \h_{k}^*)))^3 - 4 \delta^{-3/2}\epsilon_2^{3/2} \right).
\label{eq:bound_norm_of_stepsize_by_Hessian}
\end{align}

Using a telescoping argument, we have
$$
\ba{rcl}
    f(\x_0) - f^* &\geq& \sum\limits_{i=0}^{k-1} f(\x_i) - f(\x_{i+1}) \\ 
    \\ 
&\overset{\text{Lem.}~\ref{lemma:model_decrease}}{\geq}& 
\sum\limits_{i=0}^{k-1} \frac{M}{12} \| \h_i^* \|^3 \\ 
\\
&\overset{\eqref{eq:bound_norm_of_stepsize_by_Hessian}}{\geq}&
  \frac{1}{162 M^2} \sum\limits_{i=0}^{k-1} \left( (- \lambda_{\min}(\nabla^2 f(\x_{i} + \h_{i}^*)))^3 - 4 \delta^{-3/2} \epsilon_2^{3/2} \right) \\
  \\
&\geq&   \frac{1}{162 M^2} \sum\limits_{i=0}^{k-1} (- \lambda_{\min}(\nabla^2 f(\x_{i} + \h_{i}^*)))^3 - {\frac{2}{81 M^2}} k \delta^{-3/2} \epsilon_2^{3/2} \\
\\
&\geq&   \frac{1}{162 M^2} \min\limits_{1 \leq i \leq k} (- \lambda_{\min}(\nabla^2 f(\x_{i} + \h_{i}^*)))^3 - {\frac{2}{81 M^2}}  k \delta^{-3/2} \epsilon_2^{3/2}.
\ea
$$
By rearranging, we get
\begin{align}
\min_{1 \leq i \leq k} (- \lambda_{\min}(\nabla^2 f(\x_{k} + \h_{k}^*)))^3 \leq \frac{162 M^2}{k} (f(\x_0) - f^*) + 4 \delta^{-3/2} \epsilon_2^{3/2}.
\end{align}

\end{proof}

\subsection{Adaptive sampling}
\label{sec:app_adaptive_sampling}

We now present our main convergence result when using the adaptive sampling scheme described in Section~\ref{sec:adaptive_sampling}.\\
\begin{remark}
    Note that the number of sampled coordinates, $\tau(S)$, does not need to be the same for every iteration. Since we analyze the progress of each iteration $k \geq 0$ separately, we allow for different probability distributions when selecting the coordinates at each iteration. In Lemma 4.4, we consider the distribution of selecting a particular coordinate subset, conditioning on the previous iterate. Then, in the final bound we consider the randomness over the sigma-algebra of all the iterates $x_k$. In practice, one can choose the distribution for  $\tau_k$ to be uniform for each $k \geq 0$, but we emphasize once more that $\tau_k$ can vary for different $k$ (for instance we can sample fewer coordinates at the beginning of the optimization process).
\end{remark}

\MainConvergence*

\begin{proof}

By Lemma~\ref{lemma:bounding_full_grad_by_s_k},
\begin{equation}
\| \nabla f(\x_k + \h_k^*) \| \leq \frac{2M+1}{2} \|\h_k^*\|^2 + \frac{(2\delta^{-1/2}+\delta^{-1}) c_{k-1}}{2} \| \h_{k-1}^* \|^2.
\end{equation}

By the convexity of the function $g(x) = x^{3/2}$, Jensen's inequality yields $(\frac{1}{d} \sum_{i=1}^d \x_i)^{3/2} \leq \frac{1}{d} \sum_{i=1}^d (\x_i)^{3/2}$. Therefore
\beq \label{eq:bound_nablaf_1}
\ba{rcl}
\| \nabla f(\x_k + \h_k^*) \|^{3/2} &\leq& \frac{1}{2} \left( (2M+1)^{3/2}  \|\h_k^*\|^3 + ((2\delta^{-1/2}+\delta^{-1}) c_{k-1})^{3/2} \| \h_{k-1}^* \|^3 \right).
\ea
\eeq

Let $\h^*_{-1}$ be such that $\E_{[S]} \| \nabla^2 f(\x_0) - \nabla^2 f(\x_0)_{[S]} \| \leq \sqrt{c_{-1}} \| \h_{-1}^* \|$. Summing Eq.~\eqref{eq:bound_nablaf_1} over $K$, we get
$$
\ba{rcl}
\sum\limits_{i=0}^{K-1} \| \nabla f(\x_i + \h_i^*) \|^{3/2} 
&\leq& 
\frac{1}{2} \left( (2M+1)^{3/2} \sum\limits_{i=0}^{K-1} \|\h_i^*\|^3 + \sum\limits_{i=0}^{K-1} ((2\delta^{-1/2}+\delta^{-1}) c_{i-1})^{3/2} \| \h_{i-1}^* \|^3 \right)  \\
\\
&\leq& \frac{1}{2} \Big( (2M+1)^{3/2} \sum\limits_{i=0}^{K-1} \|\h_i^*\|^3 +  \sum\limits_{i=0}^{K-1} ((2\delta^{-1/2}+\delta^{1}) c_i)^{3/2} \| \h_i^* \|^3 \\
&&\quad +((2\delta^{-1/2}+\delta^{-1}) c_{-1})^{3/2} \| \h_{-1}^* \|^3 \Big) \\
\\
 &\leq& \frac{1}{2} \Big( (2M+1)^{3/2} \sum\limits_{i=0}^{K-1} \|\h_i^*\|^3 + (2\delta^{-1/2}+\delta^{-1})^{3/2} \max c_i^{3/2} \sum\limits_{i=0}^{K-1} \| \h_i^* \|^3 \\ 
&& \quad  + ((2\delta^{-1/2}+\delta^{-1}) c_{-1})^{3/2} \| \h_{-1}^* \|^3 \Big) \\
&\leq& \frac{1}{2} \Big( \left((2M+1)^{3/2} + (4 \delta^{-3/4} + \sqrt{2} \delta^{-3/2}) \max c_i^{3/2} \right) \sum\limits_{i=0}^{K-1} \|\h_i^*\|^3 \\
&& \quad + ((2\delta^{-1/2}+\delta^{-1}) c_{-1})^{3/2} \| \h_{-1}^* \|^3 \Big)
\ea
$$
which implies

\begin{equation}
\sum\limits_{i=0}^{K-1} \|\h_i^*\|^3 \geq 2 C_M \sum\limits_{i=0}^{K-1} \| \nabla f(\x_i + \h_i^*) \|^{3/2} - ((2\delta^{-1/2}+\delta^{-1}) c_{-1})^{3/2} C_M \| \h_{-1}^* \|^3.
\label{eq:lower_bound_sum_s_norm_cube}
\end{equation}
where $C_M := \left((2M+1)^{3/2} + (4 \delta^{-3/4} + \sqrt{2} \delta^{-3/2}) \max c_i^{3/2} \right) ^{-1}$.

Using a telescoping argument, we have
$$
\ba{rcl}
f(\x_0) - f^* &\geq& \sum\limits_{i=0}^{K-1} f(\x_i) - f(\x_{i+1})  \\
\\
&\overset{Eq.~\eqref{eq:model_decrease}}{\geq}& 
\frac{M}{12} \sum\limits_{i=0}^{K-1} \| \h_i^* \|^3 \\
&\overset{Eq.~\eqref{eq:lower_bound_sum_s_norm_cube}}{\geq}& 
\frac{M C_M}{6} \sum\limits_{i=0}^{K-1} \| \nabla f(\x_i + \h_i^*) \|^{3/2} - \frac{M}{12} ((2\delta^{-1/2}+\delta^{-1}) c_{-1})^{3/2} C_M \| \h_{-1}^* \|^3
\ea
$$

Rearranging terms, we obtain
\begin{align}
(f(\x_0) - f^*) + \frac{M}{12} ((2\delta^{-1/2}+\delta^{-1})  c_{-1})^{3/2} C_M \| \h_{-1} \|^3 &\geq \frac{M C_M}{6} \sum\limits_{i=0}^{K-1} \| \nabla f(\x_i + \h_i^*) \|^{3/2} \nonumber \\
&\geq \frac{M C_M}{6} K \min\limits_{0 \leq i \leq K-1} \| \nabla f(\x_i + \h_i^*) \|^{3/2} \nonumber \\
&= \frac{M C_M}{6} K \min\limits_{1 \leq i \leq K} \| \nabla f(\x_i) \|^{3/2}.
\end{align}

Now rearranging for $\|\nabla f(\x_i)\|$ we get
\begin{align}
\min_{1 \leq i \leq K} \|  \nabla f (\x_i) \|^{3/2} &\leq 
\frac{1}{K} \left( \frac{6}{M C_M} (f(\x_0) - f^*) + \frac{1}{2} ((2\delta^{-1/2}+\delta^{-1}) c_{-1})^{3/2} \| \h_{-1}^* \|^3 \right),
\end{align}
which implies $\min_{1 \leq j \leq K} \|  \nabla f (\x_j) \| = \mathcal{O}(K^{-2/3})$.

We can also obtain a guarantee in terms of second-order optimality as follows. By Lemma~\ref{lemma:bounding_full_hessian_by_s_k},
\begin{equation}
- \lambda_{\min}(\nabla^2 f(\x_{k} + \h_{k}^*)) \leq \frac{3M}{2} \|\h_k^*\| + \sqrt{\delta^{-1} \epsilon_2}.
\end{equation}

By the convexity of the function $g(x) = x^3$ on $\R^+$, Jensen's inequality yields $(\frac{1}{d} \sum_{i=1}^d \x_i)^3 \leq \frac{1}{d} \sum_{i=1}^d (\x_i)^3$, therefore

\begin{align}
(- \lambda_{\min}(\nabla^2 f(\x_{k} + \h_{k}^*)))^3 &\overset{\text{monotonicity of } x^3}{\leq} \left( \frac{1}{2} \left( {3M} \|\h_k^*\| + 2\sqrt{\delta^{-1} \epsilon_2} \right) \right) ^{3} \nonumber \\
&\overset{\text{Jensen}}{\leq} \frac{1}{2} \left( {27M^3} \|\h_k^*\|^3 + 8 \delta^{-3/2} \epsilon_2^{3/2}\right) \nonumber \\
&= \frac12 \left({27 M^3} \|\h_k^*\|^3 + 8 \delta^{-3/2} c_{k-1}^{3/2} \| \h_{k-1}^* \|^3 \right).
\label{eq:bound_norm_of_stepsize_by_Hessian_2}
\end{align}

Let $\h_{-1}^*$ be such that $\E_S \| \nabla^2 f(\x) - \nabla^2 f(\x)_{[S]} \| \leq \sqrt{c_{-1}} \| \h_{-1}^* \|$. Summing Eq.~\eqref{eq:bound_norm_of_stepsize_by_Hessian_2} over $K$, we get
$$
\ba{rcl}
\sum\limits_{i=0}^{K-1} (- \lambda_{\min}(\nabla^2 f(\x_{i} + \h_{i}^*)))^3 
&\leq& \frac{27 M^3}{{2}} \sum\limits_{i=0}^{k-1} \|\h_i^*\|^3 + {4}\sum\limits_{i=0}^{k-1} \delta^{-3/2} c_{i-1}^{3/2} \| \h_{i-1}^* \|^3  \\
\\
&\leq& \frac{27 M^3}{{2}} \sum\limits_{i=0}^{k-1} \|\h_i^*\|^3 + {4}\left( \sum\limits_{i=0}^{k-1} \delta^{-3/2} c_i^{3/2} \| \h_i^* \|^3 + \delta^{-3/2} c_{-1}^{3/2} \| \h^*_{-1} \|^3 \right)  \\
\\
&\leq& \left( \frac{27 M^3}{{2}} +  {4}\delta^{-3/2} \max_i c_i^{3/2} \right) \sum\limits_{i=0}^{k-1} \| \h_i^* \|^3 + {4}\delta^{-3/2} c_{-1}^{3/2} \| \h^*_{-1} \|^3,
\ea
$$

which implies
\begin{equation}
\sum_{i=0}^{K-1} \|\h_i^*\|^3 \geq D_M \sum_{i=0}^{K-1} (- \lambda_{\min}(\nabla^2 f(\x_{i} + \h_{i}^*)))^3 - {4}\delta^{-3/2}D_M c_{-1}^{3/2} \| \h_{-1}^* \|^3,
\label{eq:lower_bound_sum_s_norm_cube_2}
\end{equation}
where $D_M = \left( \frac{27 M^3}{{2}} +  {4}\delta^{-3/2} \max_i c_i^{3/2} \right)^{-1}$.

Using a telescoping argument, we have
$$
\ba{rcl}
f(\x_0) - f^* &\geq& \sum\limits_{i=0}^{K-1} f(\x_i) - f(\x_{i+1})  \\
 \\
&\overset{Eq.~\eqref{eq:model_decrease}}{\geq}& \frac{M}{12} \sum\limits_{i=0}^{K-1} \| \h_i^* \|^3 \\ \\
&\overset{Eq.~\eqref{eq:lower_bound_sum_s_norm_cube_2}}{\geq}&
\frac{M}{12} D_M \sum_{i=0}^{K-1} (- \lambda_{\min}(\nabla^2 f(\x_{i} + \h_{i}^*)))^3 - \frac{M}{{3}} \delta^{-3/2} D_M c_{-1}^{3/2} \| \h_{-1}^* \|^3.
\ea
$$

By rearranging terms, we obtain
$$
\ba{rcl}
(f(\x_0) - f^*) + \frac{M}{{3}} \delta^{-3/2} D_M c^{3/2}_{-1} \| \h_{-1}^* \|^3 &\geq& 
\frac{M}{12} D_M \sum_{i=0}^{K-1} (- \lambda_{\min}(\nabla^2 f(\x_{i} + \h_{i}^*)))^3  \\
\\
&\geq& \frac{M}{12} D_M K \min\limits_{0 \leq i \leq K-1} (- \lambda_{\min}(\nabla^2 f(\x_{i} + \h_{i}^*)))^3 \\ \\
&\geq& \frac{M}{12} D_M K \min\limits_{1 \leq i \leq K} (- \lambda_{\min}(\nabla^2 f(\x_{i})))^3.
\ea
$$

Thus, we get
\begin{align}
\nonumber
\min_{1 \leq i \leq K} (- \lambda_{\min}(\nabla^2 f(\x_{i})))^3 \leq 
\frac{1}{K} \left( \frac{12}{M D_M} (f(\x_0) - f^*) + {4} \delta^{-3/2} c^{3/2}_{-1} \| \h_{-1}^* \|^3 \right),
\end{align}
which implies $\min_{1 \leq j \leq K} (- \lambda_{\min}(\nabla^2 f(\x_{j}))) = \mathcal{O}(K^{-1/3})$.

\end{proof}

\newpage

\section{EXTENDED RELATED WORK}

In the literature random subspace methods are also widely known as sketch-and-project methods. In the convex setting, the sketched Newton method proposed in~\cite{hanzely2023sketch} is shown to convergence at a global rate of $\mathcal{O}(k^{-2})$. More recently,~\cite{derezinski2024sharp} derive sharp convergence rates for sketch-and-project methods to iteratively solve linear systems via a connection to randomized singular value decomposition. They extend their setting to the minimization of convex function with stochastic Newton methods. In~\cite{lacotte2021adaptive}, the authors make a connection between the sketch size and the efficient Hessian dimensionality and show quadratic convergence for self-concordant, strongly convex functions.\\
Extensions of stochastic Newton methods to inexact tensor methods for convex objectives were proposed in~\cite{lucchi2023sub, agafonov2024inexact, doikov2020inexact}. 
In the context of non-convex optimization,~\cite{cartis2022randomised} derive a high-probability bound for convergence to a first-order stationary bound for randomised subspace methods, which are safeguarded by trust region or quadratic regularization with a complexity of $\mathcal{O}(k^{-1/2})$, which is the same rate as for gradient-based methods~\citep{nesterov2018lectures}.

\newpage 
 
\section{ADDITIONAL EXPERIMENTS}
\label{sec:additional_experiments}

\paragraph{Used datasets and license}
The datasets used in the experiments are all taken from LibSVM ~\citep{CC01a}, which are provided under a modified BSD license.

\subsection{Constant coordinate schedule on additional datasets}

We verified our theoretical results also on three other datasets \textit{duke} ($n = 7129, d = 44$), \textit{madelon} ($n = 500, d=2000)$ and \textit{realsim} ($n = 20.958, d=72.309$). The convergence results can be found in Figure \ref{fig:logistic_regression_nonconv_convergence_duke_madelon}. 

\begin{figure*}[ht]
    \centering
    \begin{tabular}{cc}

        \includegraphics[width=0.31\textwidth]{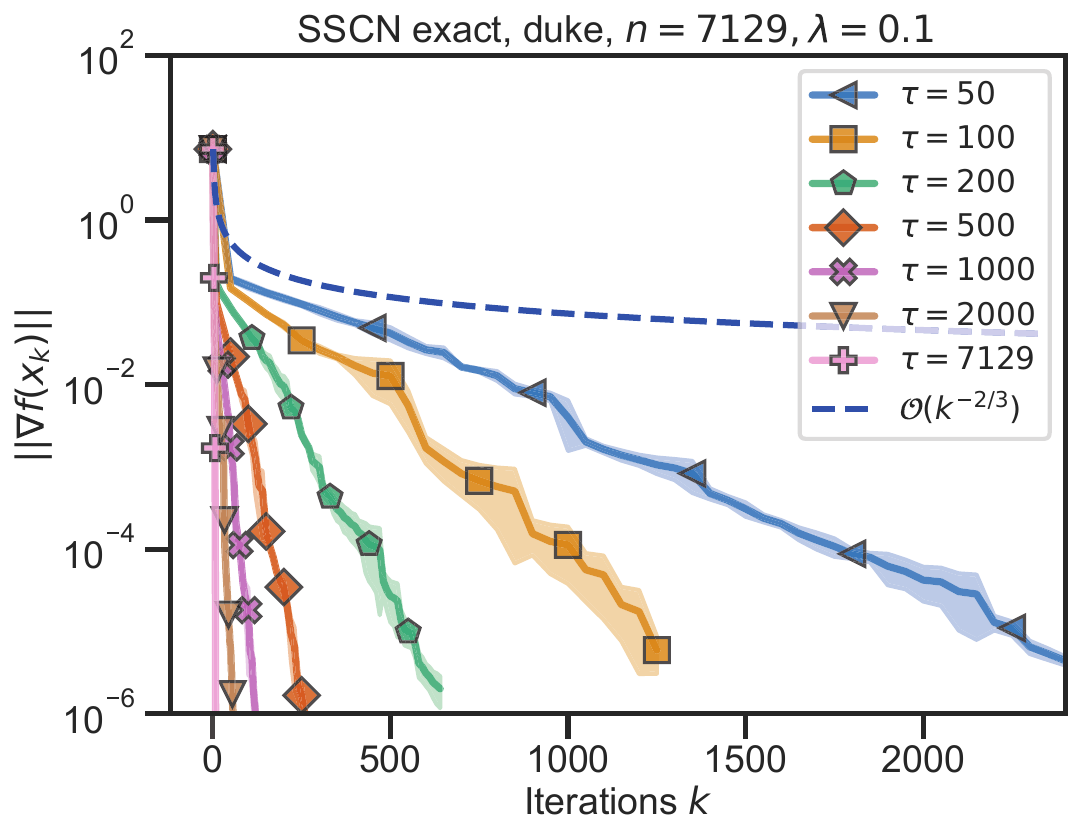}
        \includegraphics[width=0.31\textwidth]{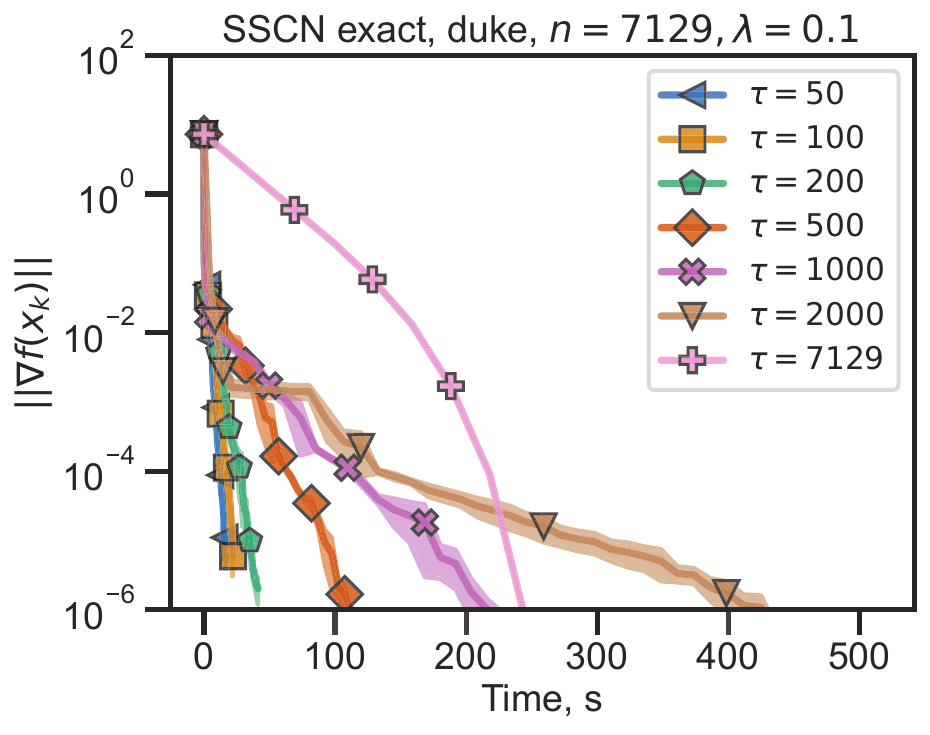}
        \includegraphics[width=0.31\textwidth]{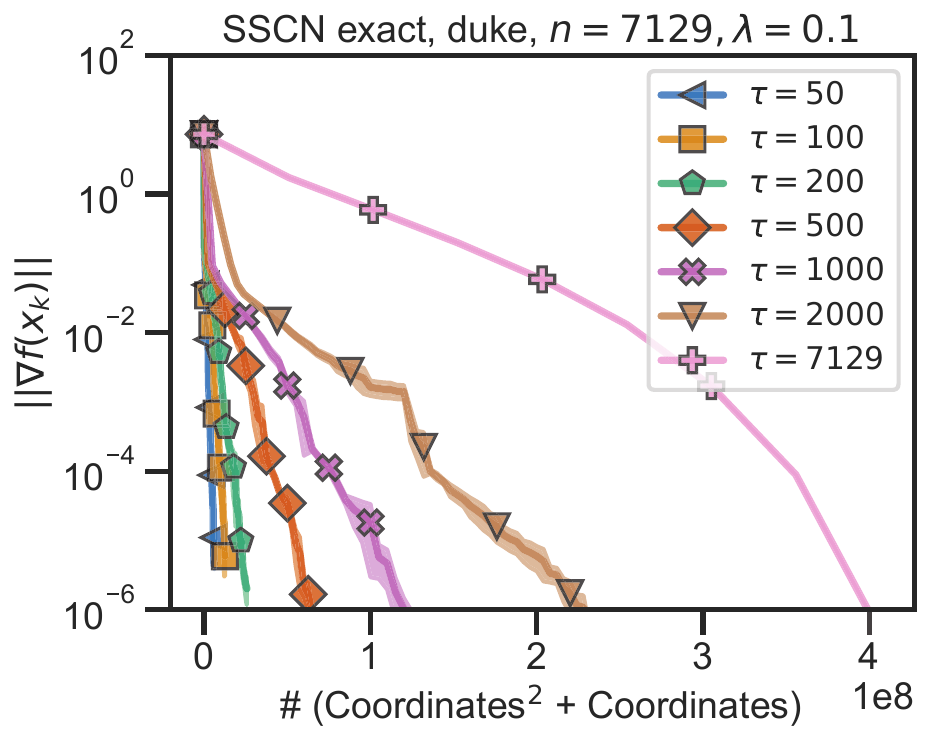}\\
    
        \includegraphics[width=0.31\textwidth]{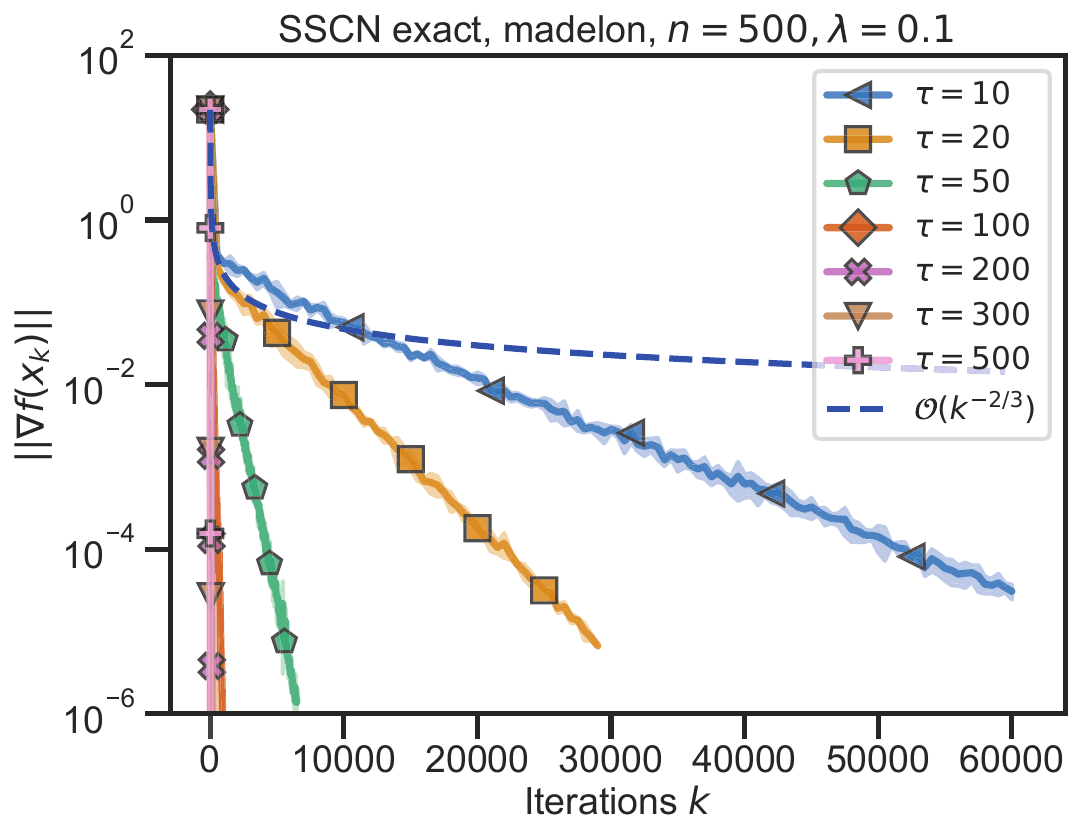}
        \includegraphics[width=0.31\textwidth]{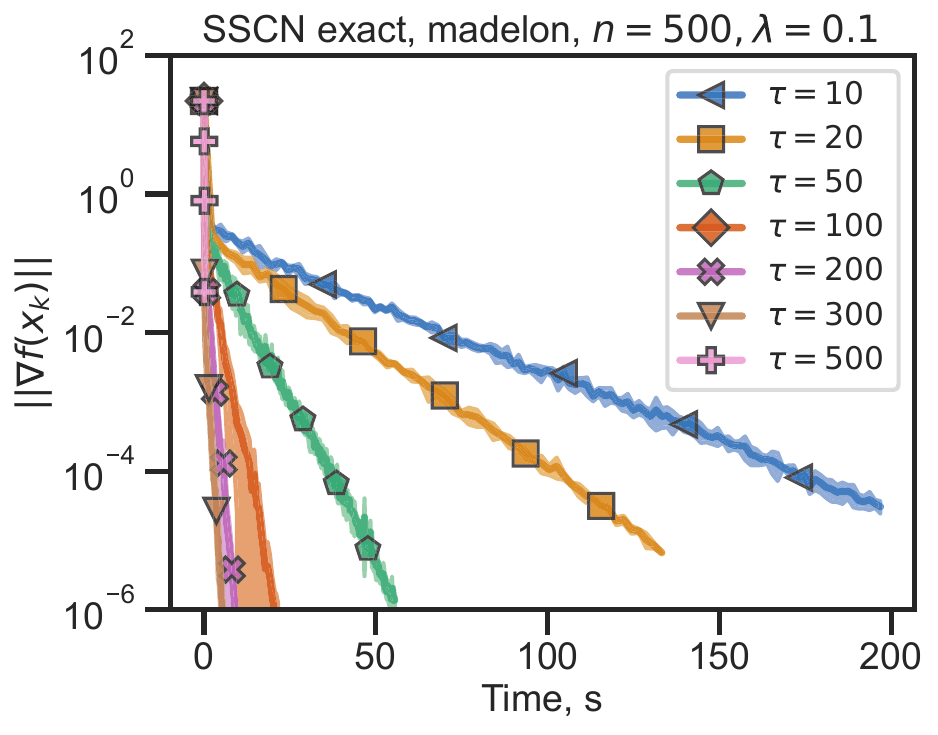}
        \includegraphics[width=0.31\textwidth]{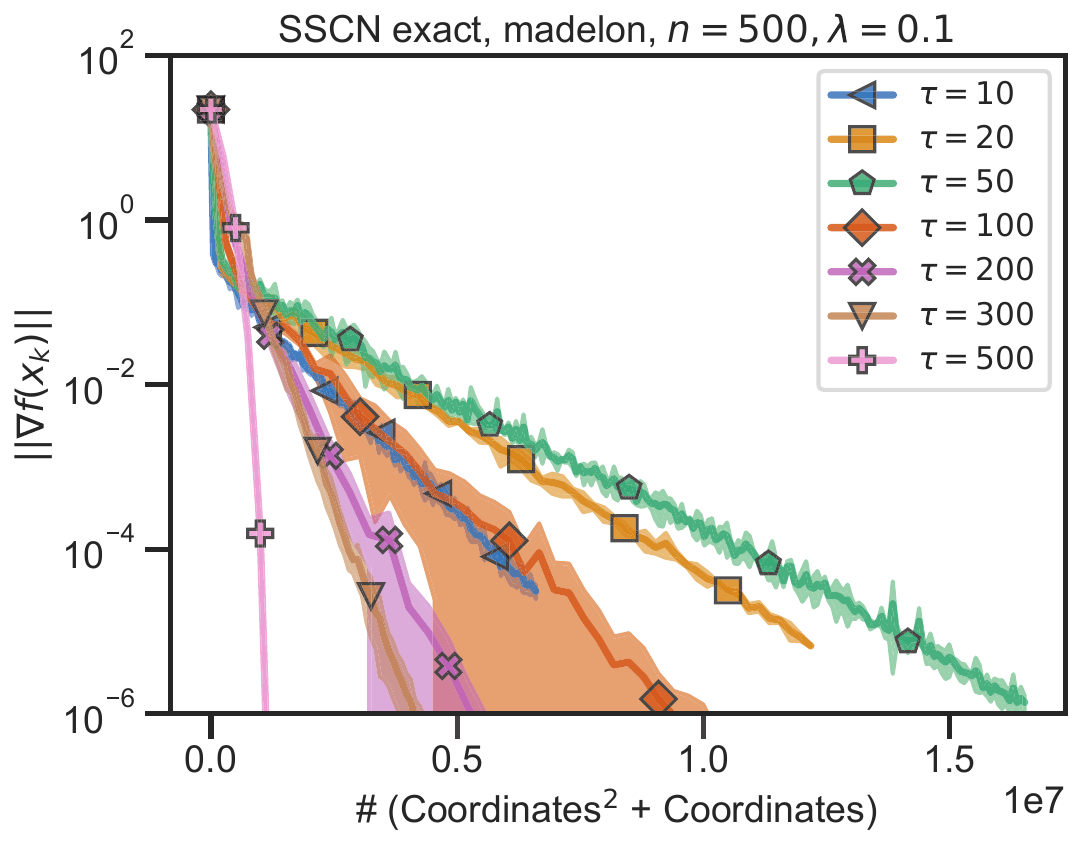} \\

        \includegraphics[width=0.31\linewidth]{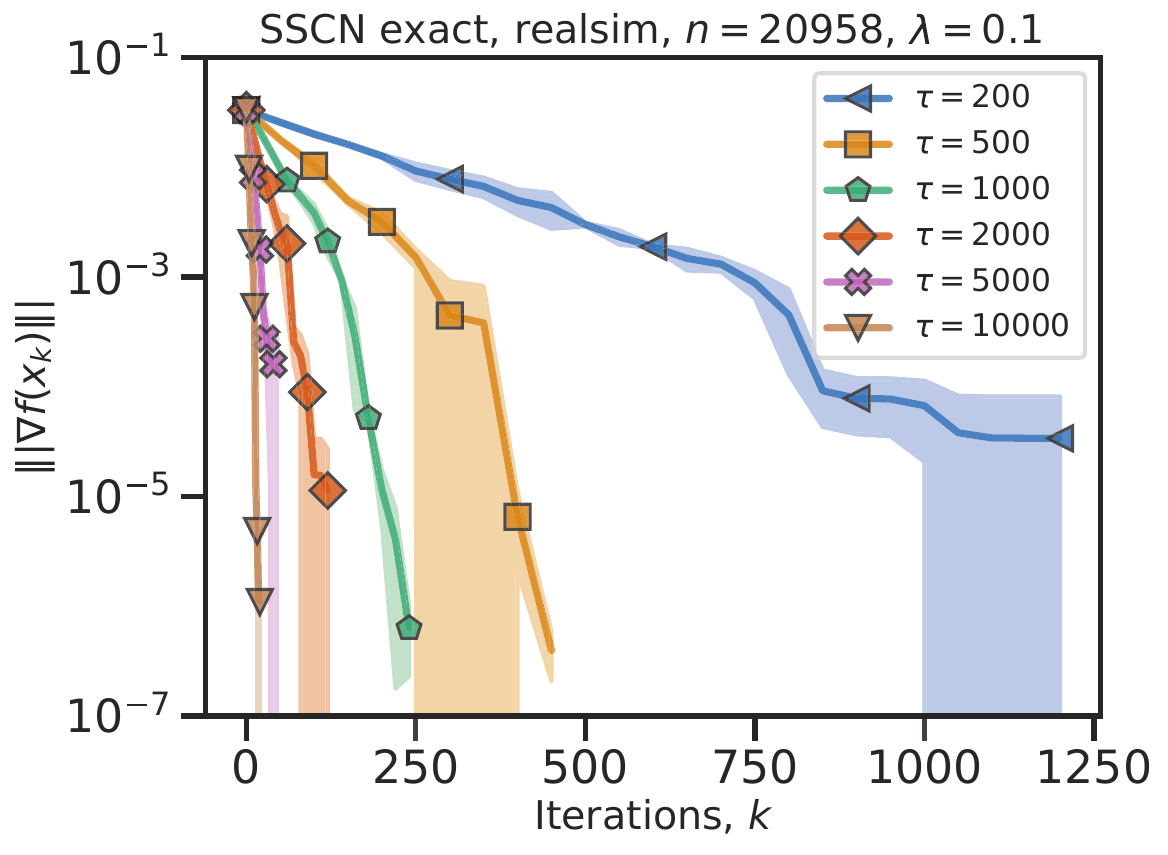}

        \includegraphics[width=0.31\linewidth]{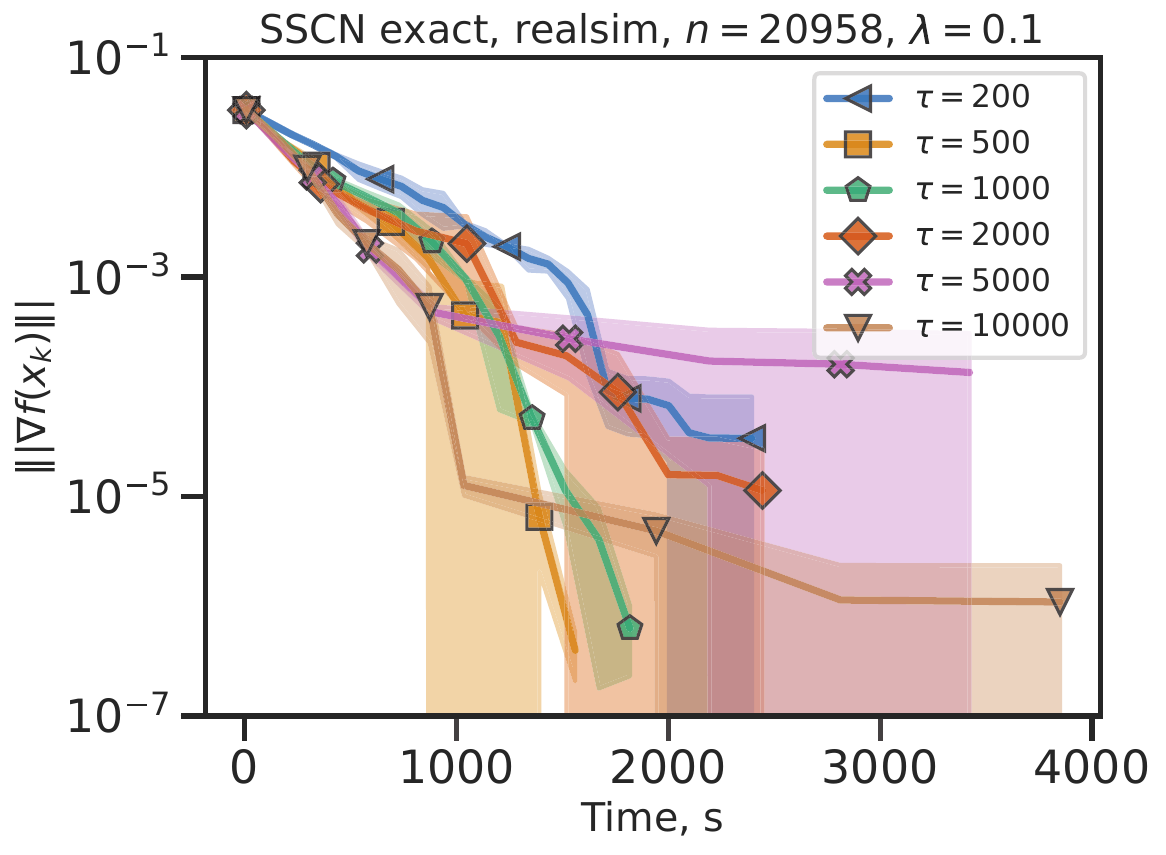}

        \includegraphics[width=0.31\linewidth]{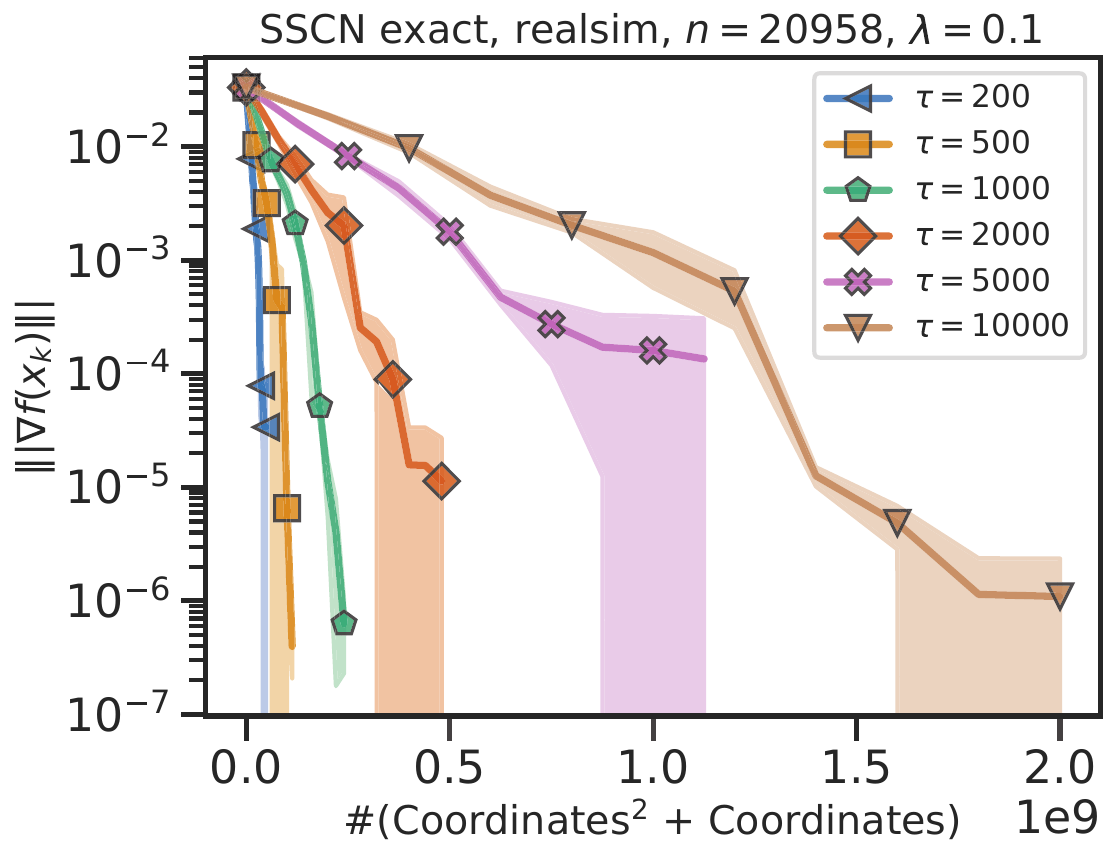}

    \end{tabular}
    \caption{Convergence of different constant coordinate schedules measured w.r.t. iterations (first column), time (second column) and \# (Coordinates$^2$ + Coordinates) evaluated (third column) averaged over three runs for logistic regression with non-convex regularization with $\lambda = 0.1$ for two datasets. First row: \textit{duke}, second row: \textit{madelon}, third row: \textit{realsim}.}
    \label{fig:logistic_regression_nonconv_convergence_duke_madelon}
\end{figure*}

\clearpage

\subsection{Constant vs. exponential schedule}
\label{subsec:constant_vs_exponential_schedule}

As discussed earlier, the adaptive sampling scheme presented in Section~\ref{sec:adaptive_sampling} suggests using an exponential schedule to sample coordinates. We compare this schedule to a constant schedule in Figure~\ref{fig:logistic_regression_exp_vs_const_Schedule}. We observe that for the two high-dimensional datasets \textit{gisette} and \textit{duke}, \textbf{the best constant schedule and the best exponential schedule perform on par both in terms of time and \# coordinates evaluated}. 
We conclude that in some cases a simple constant schedule can already perform sufficiently well. However, an exponential schedule might still be more appropriate if one needs a high-accuracy solution.

\begin{figure*}[h!]
    \centering
    \begin{tabular}{cc}
        \includegraphics[width=0.31\textwidth]{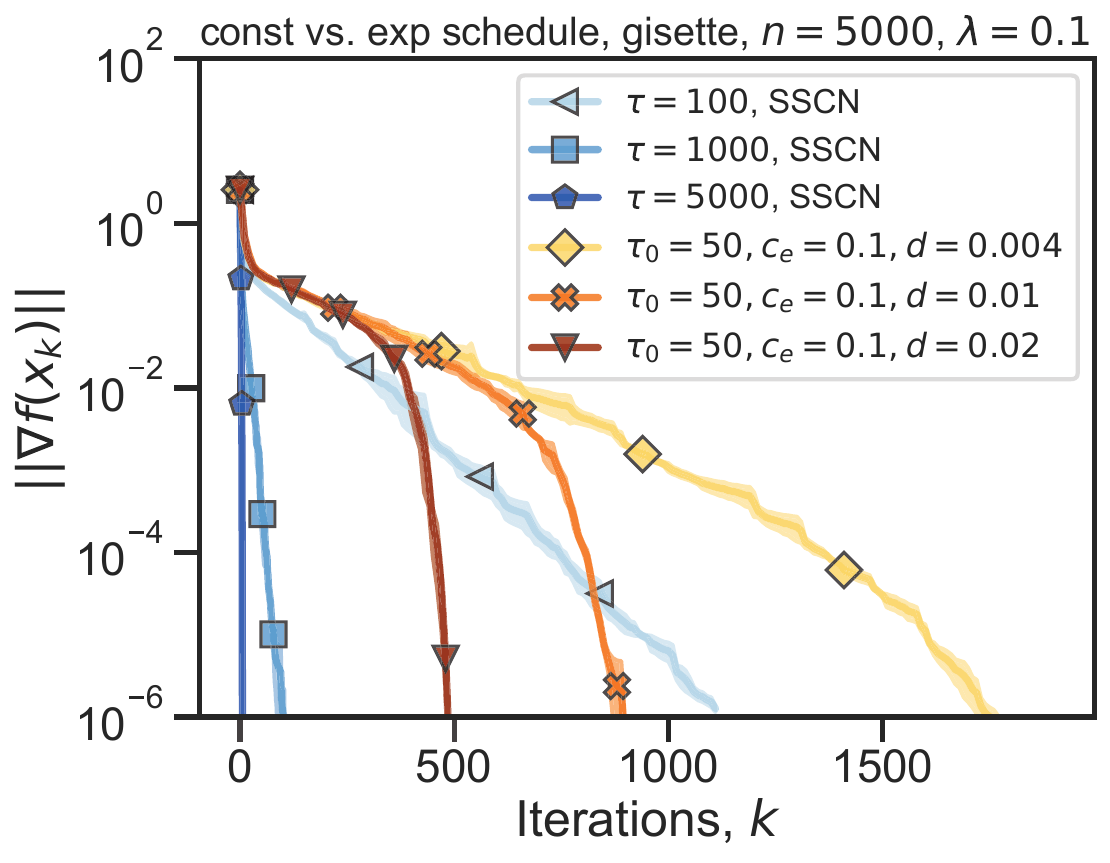}
        \includegraphics[width=0.31\textwidth]{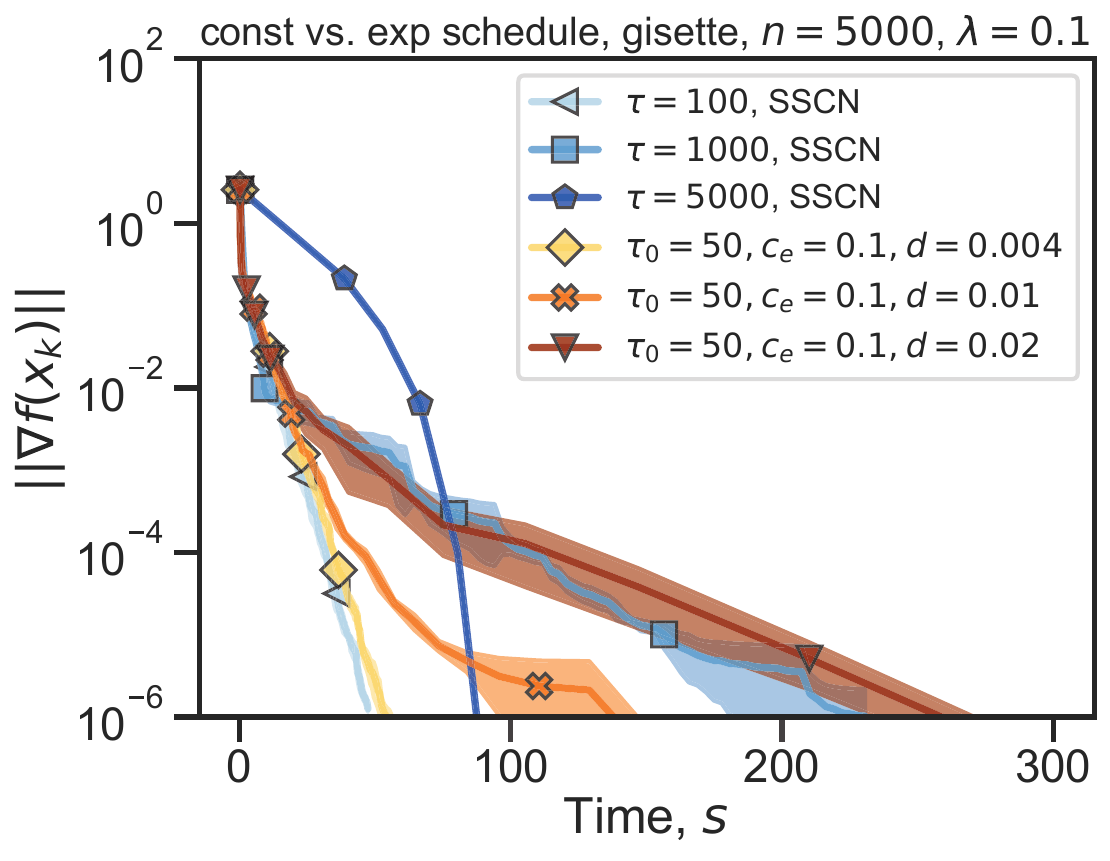}
        \includegraphics[width=0.31\textwidth]{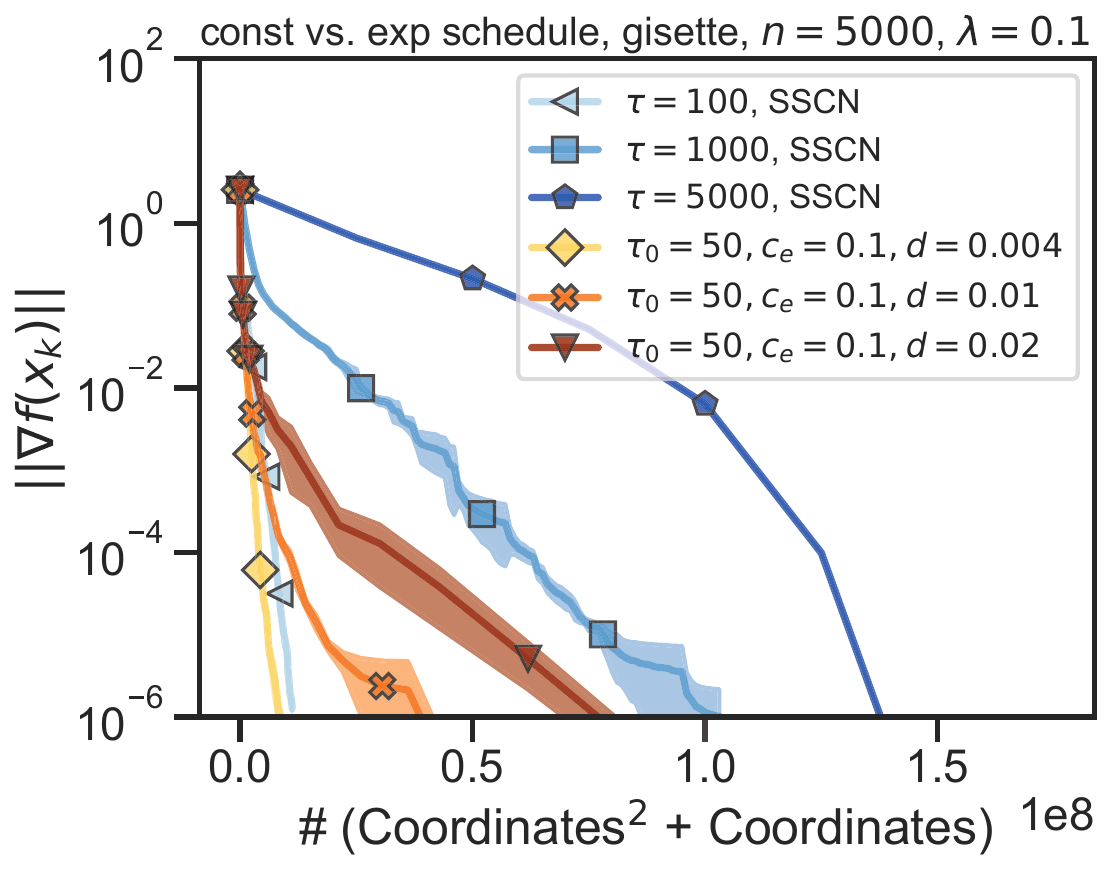} \\
        \includegraphics[width=0.31\linewidth]{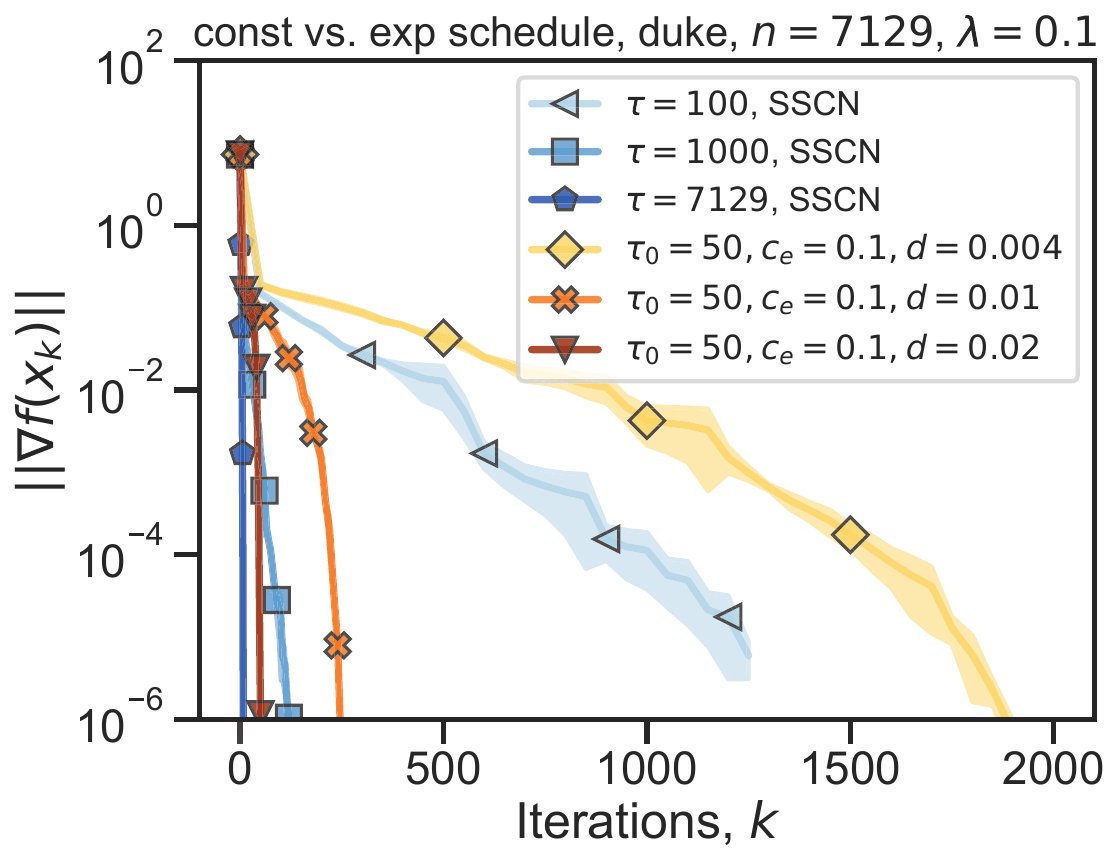}
        \includegraphics[width=0.31\linewidth]{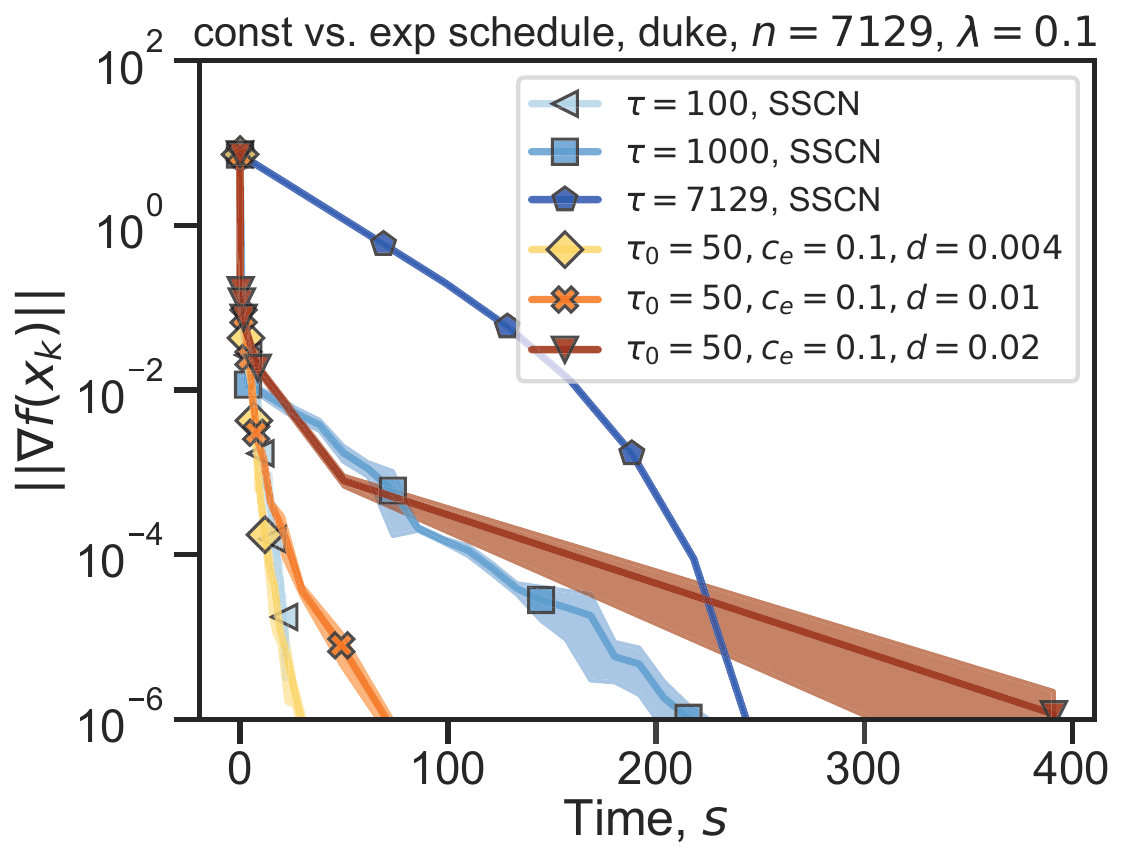}
        \includegraphics[width=0.31\linewidth]{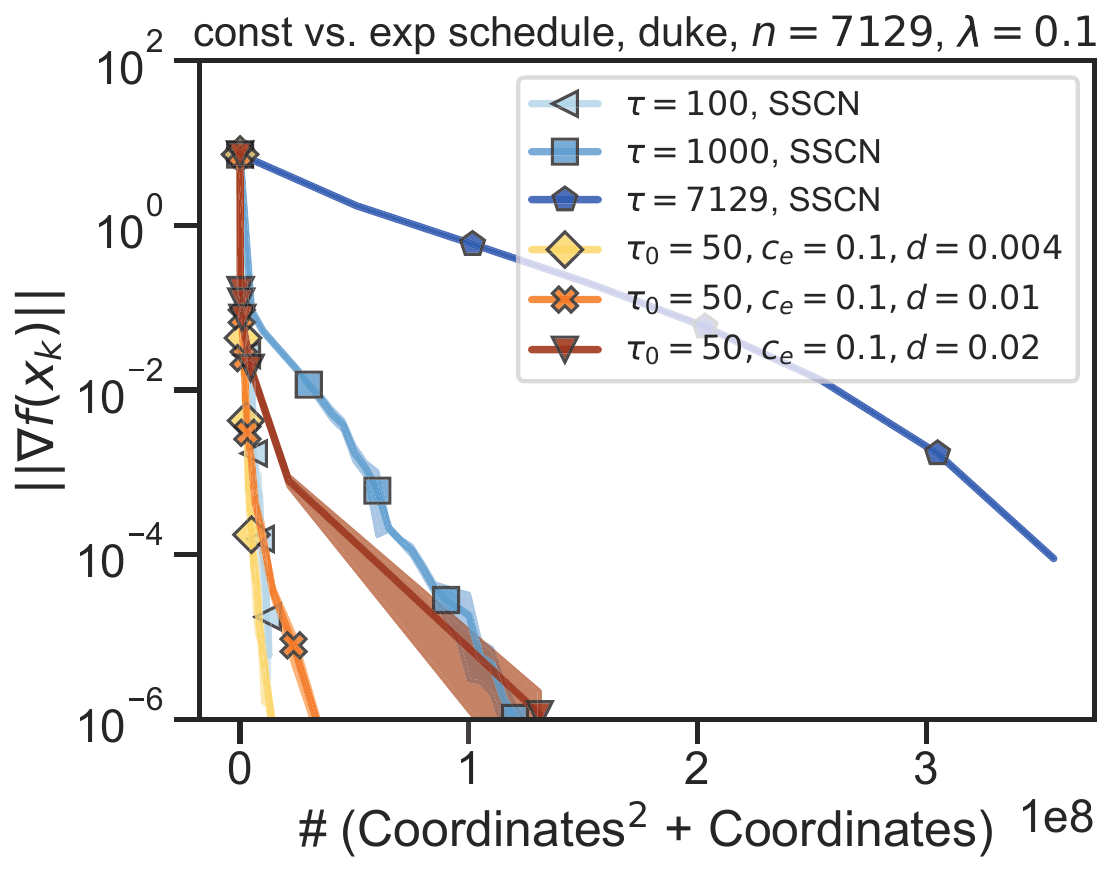}\\
        
        \includegraphics[width=0.31\linewidth]{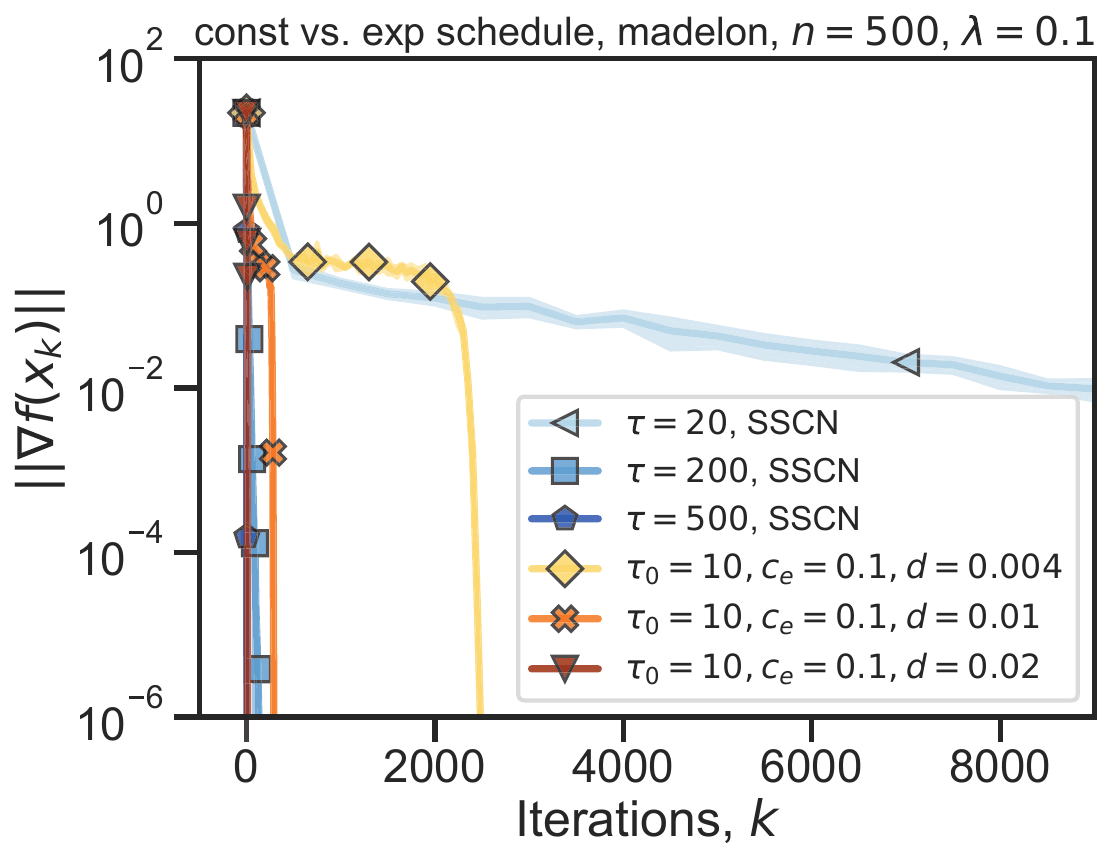}
        \includegraphics[width=0.31\linewidth]{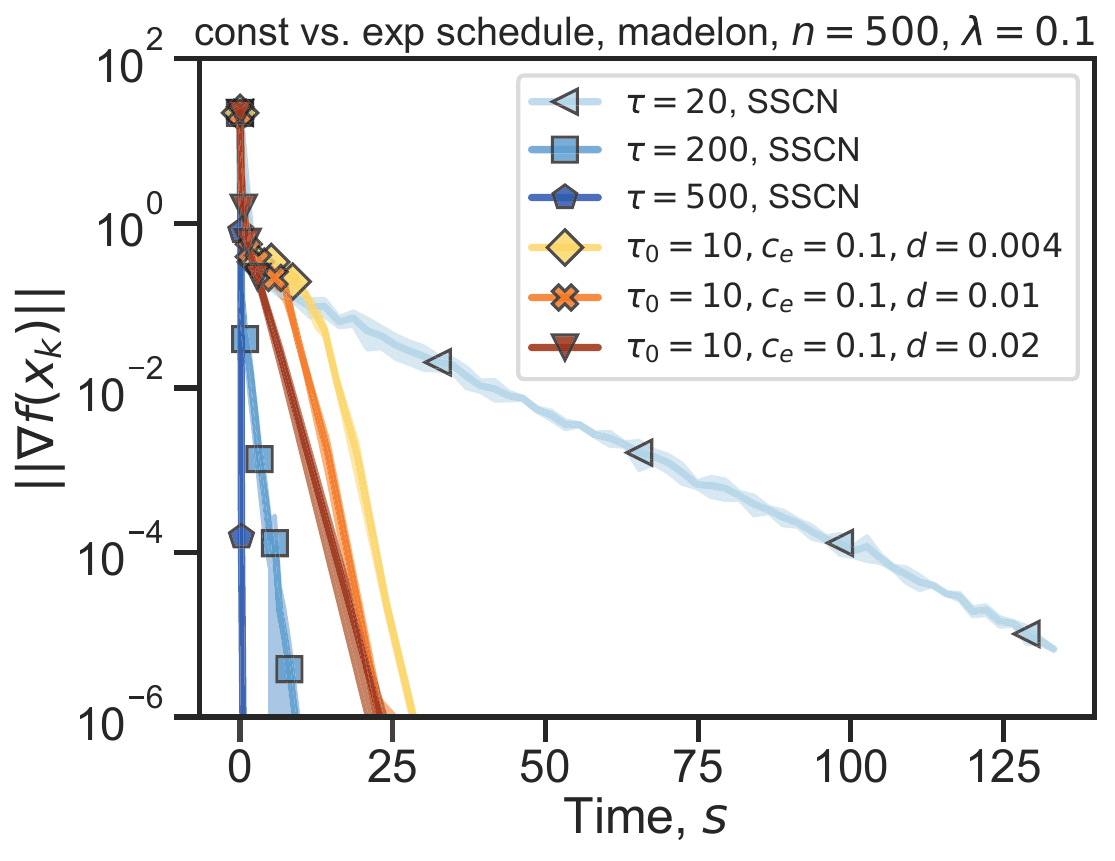}
        \includegraphics[width=0.31\linewidth]{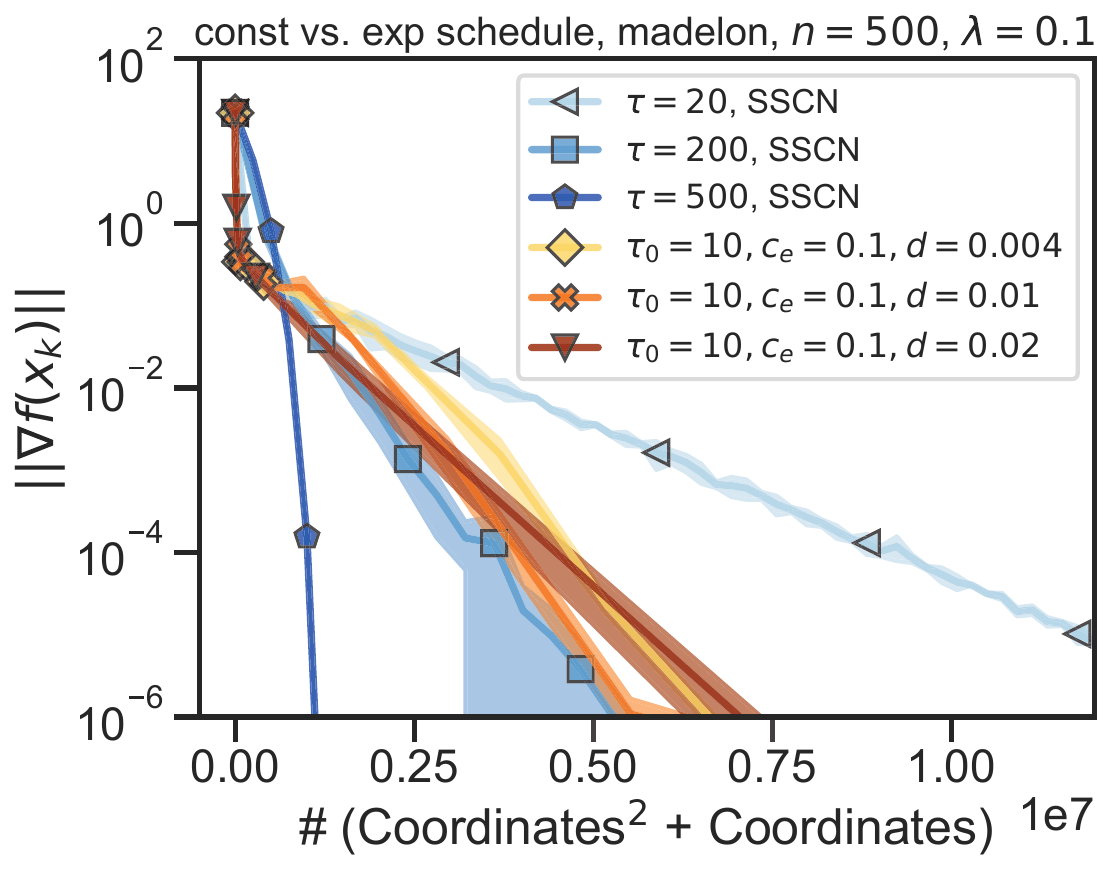}
    \end{tabular}
    \caption{Comparison of constant vs. exponential schedules $\tau(S_k) = \tau_0 + c_e \exp(dk)$ for different parameters w.r.t. iterations (first column) and time (second column) and \# (Coordinates$^2$ + Coordinates) evaluated (third column) averaged over three runs for logistic regression with non-convex regularization with $\lambda = 0.1$ for the \textit{gisette}, \textit{duke} and \textit{madelon} dataset.}
    \label{fig:logistic_regression_exp_vs_const_Schedule}
\end{figure*}

\subsection{Comparison to other methods for more datasets}

We also compared SSCN to CD and RS-RNM \citep{fuji2022randomized} on logistic regression with non-convex regularization on two other datasets, namely the \textit{duke} dataset and the \textit{realsim} dataset. The results can be found in \cref{fig:logistic_regression_nonconv_CD_vs_SSCN_duke}. 
We observe that CD converges significantly faster in time to a first-order stationary point for the \textit{duke} dataset and also slightly faster on the \textit{realsim} dataset, while RS-RNM is significantly slower due to the high per-iteration cost.\\
We trace back this differences in performance to the complexity of the loss landscape. In a separate run, we calculated the condition number of the Hessian of the loss function over the full space as an estimate for complexity of the landscape. The results can be found in \cref{fig:logistic_regression_nonconv_condnum}. We note that the loss is extremely ill-conditioned on the \textit{madelon} dataset, which explains while CD fails to converge while SSCN is not affected. In better conditioned problems such as \textit{duke} or \textit{realsim}, the higher per-iteration cost of SSCN leads to an overall slower convergence. 
In general, we expect more gains from using SSCN over CD on more complex loss landscape.

\begin{figure*}[h!]
    \centering
    \begin{tabular}{cc}
    
        \includegraphics[width=0.29\linewidth]{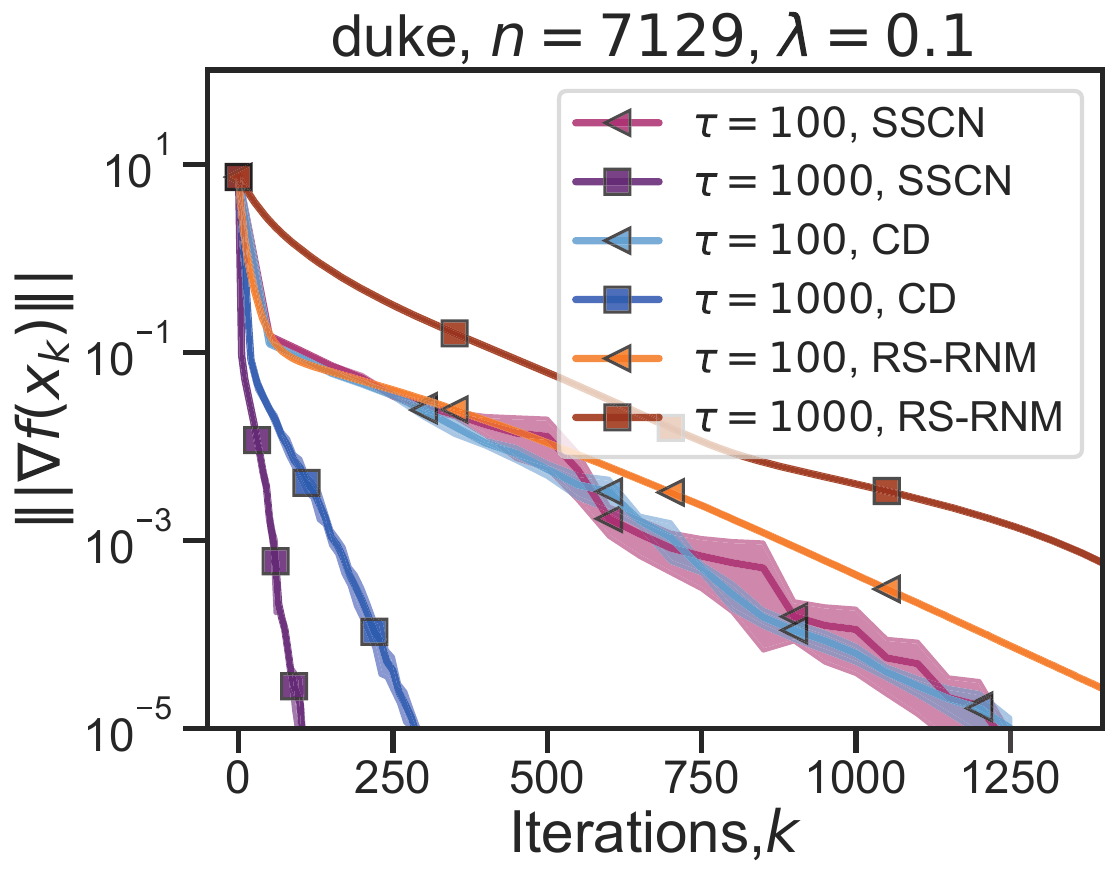}
        \includegraphics[width=0.29\linewidth]{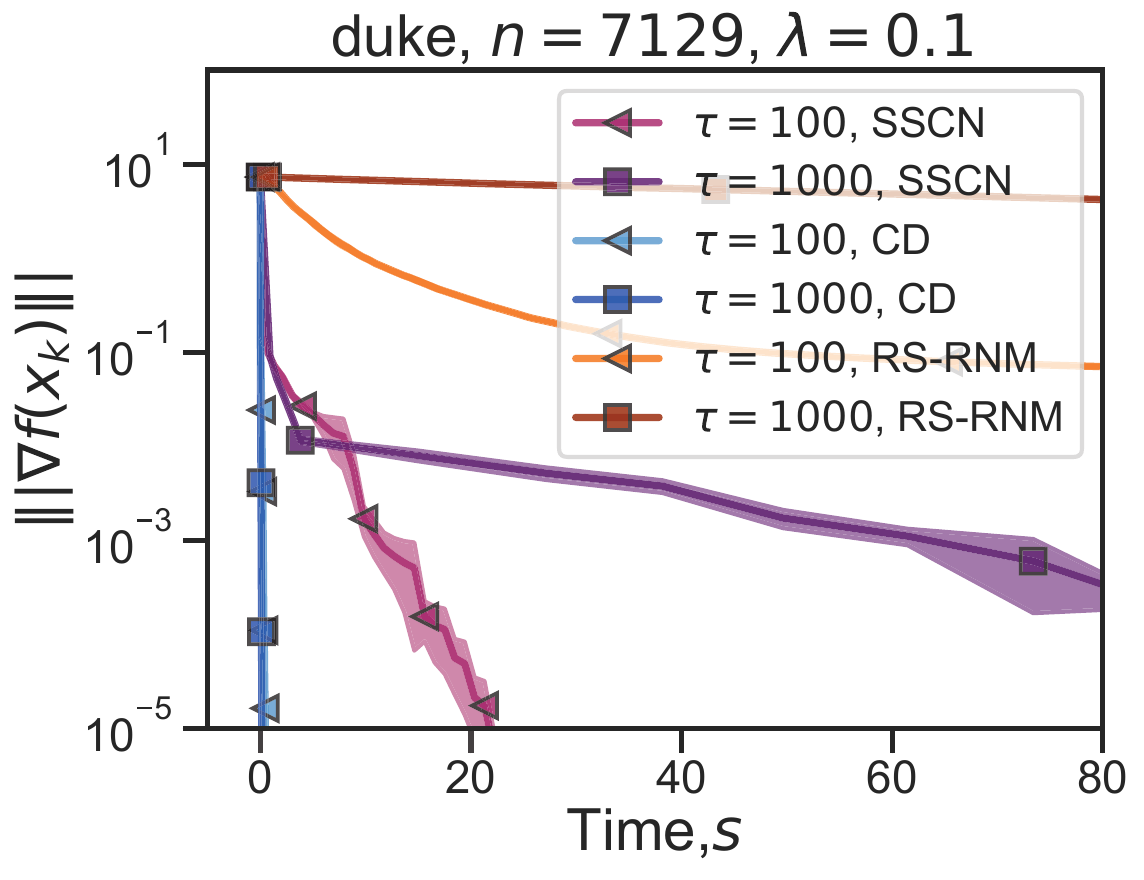}
        \includegraphics[width=0.31\linewidth]{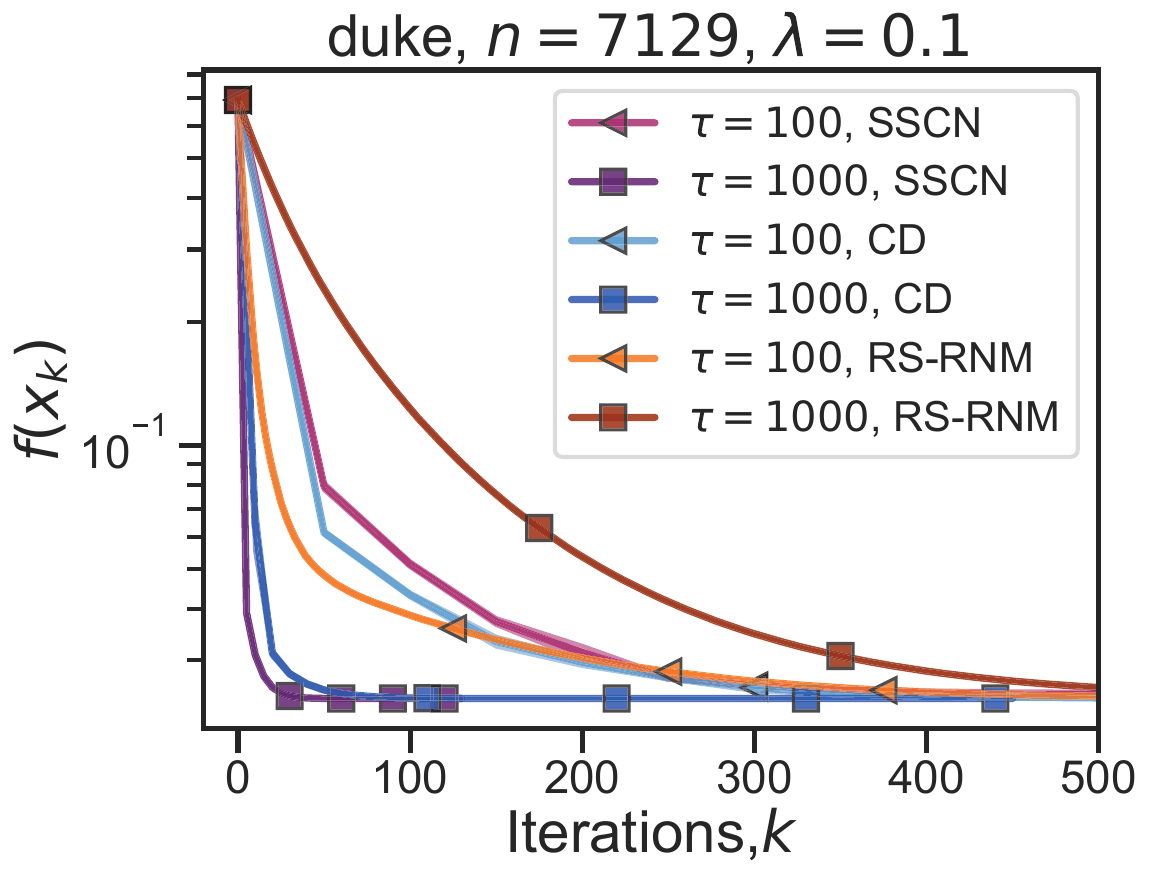}\\
        \includegraphics[width=0.29\linewidth]{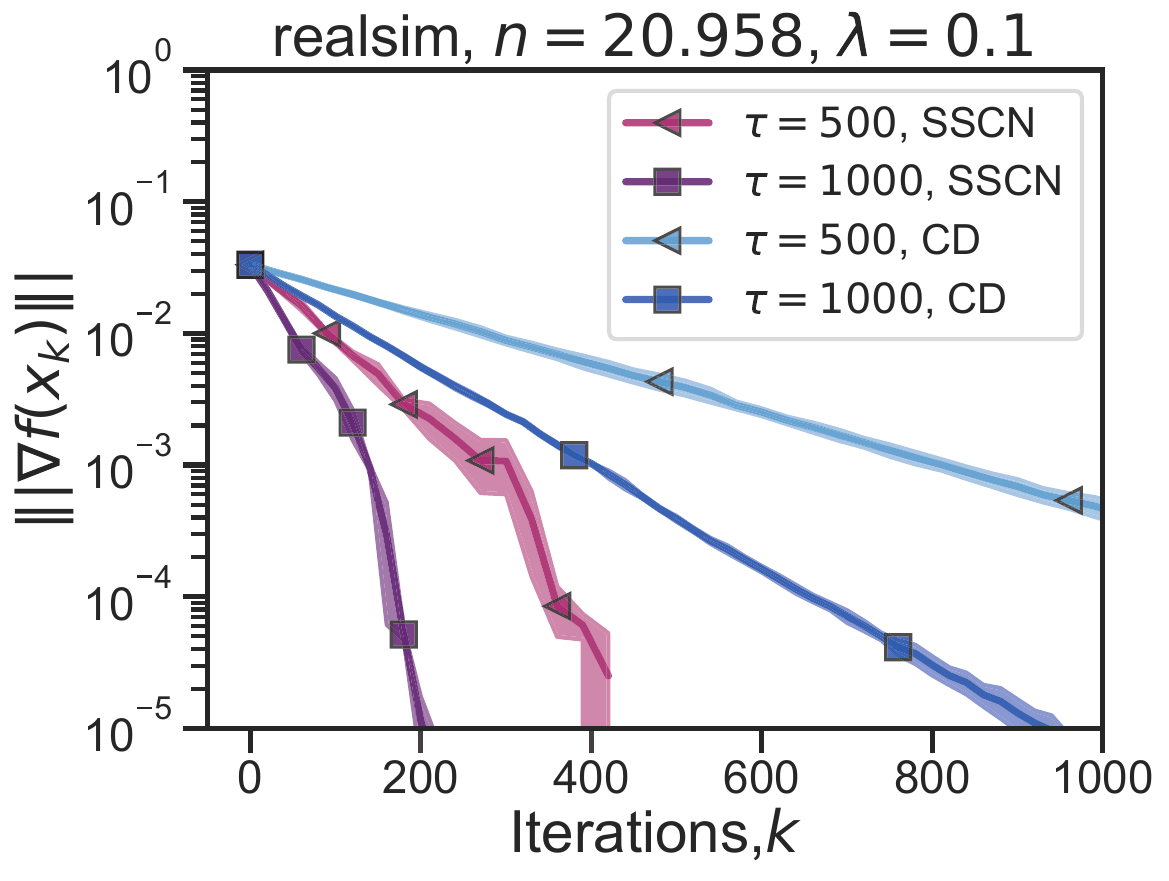}
        \includegraphics[width=0.29\linewidth]{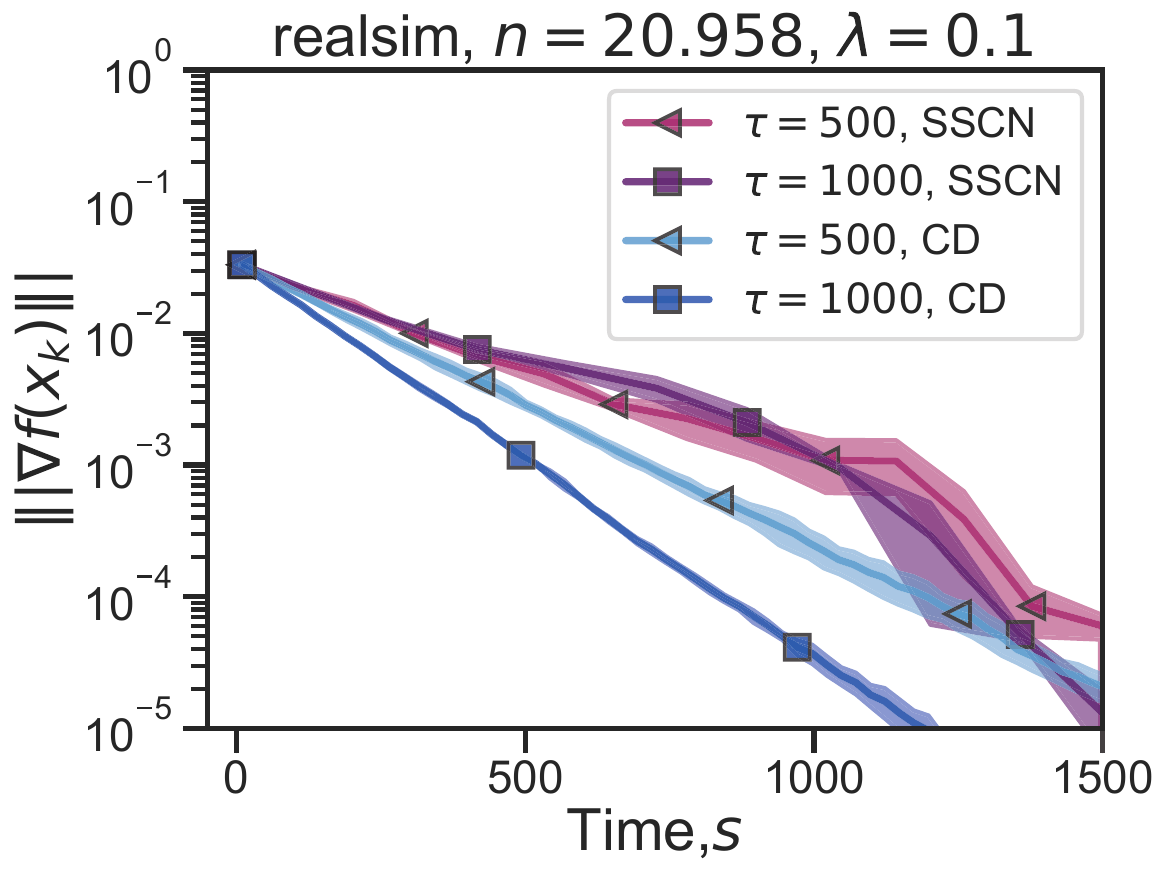}
        \includegraphics[width=0.31\linewidth]{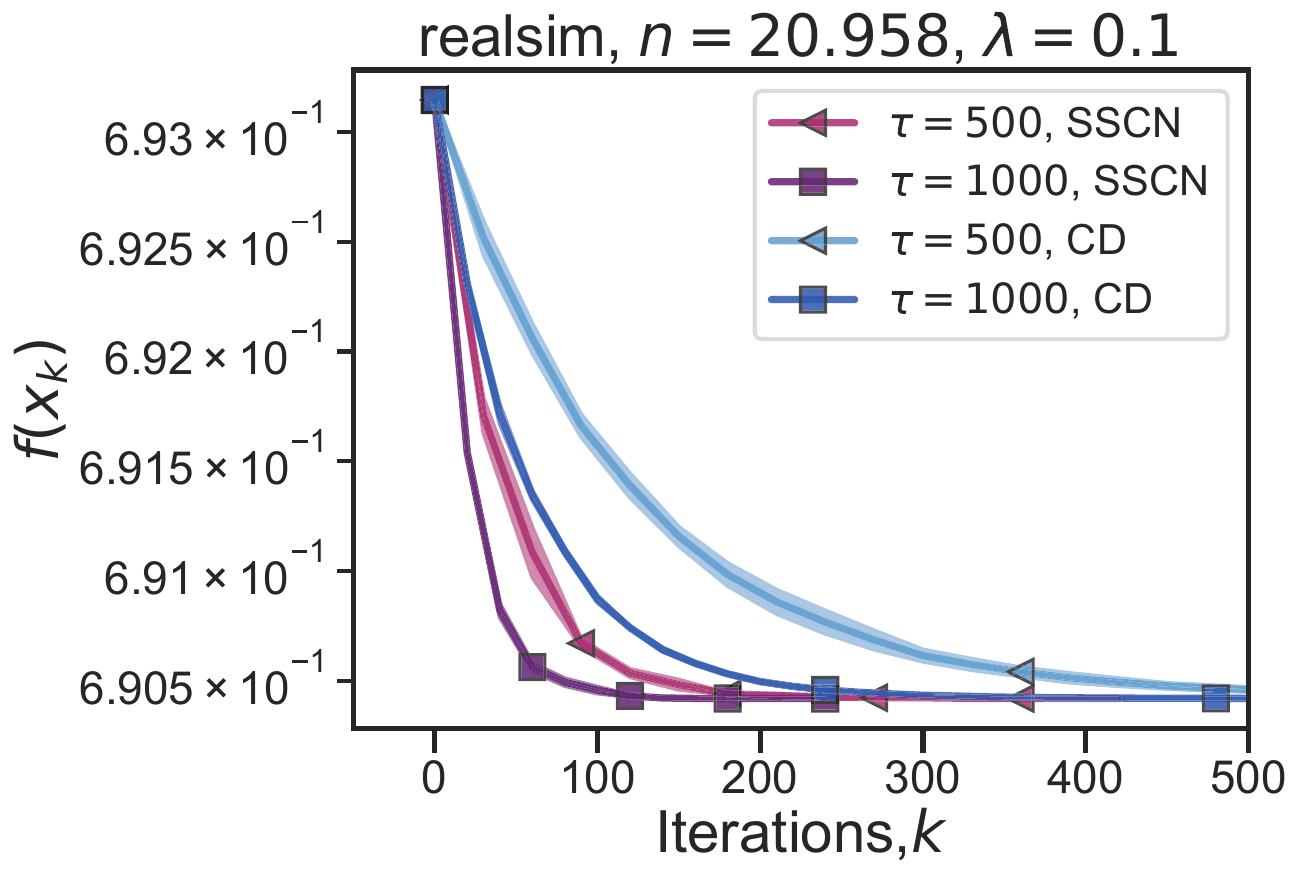}

    \end{tabular}
    \caption{Comparison of CD, SSCN and RS-RNM \citep{fuji2022randomized} for different constant coordinate schedules measured the gradient norm w.r.t. iterations (first column) and time (second column) and the loss w.r.t iterations (third column) averaged over three runs for logistic regression with non-convex regularization with $\lambda = 0.1$ for the \textit{duke} dataset (first row) and \textit{realsim} dataset (second row). Note that RS-RNM exceeded the compute budget of 3000 seconds before converging and is thus not shown.}
    \label{fig:logistic_regression_nonconv_CD_vs_SSCN_duke}
    
\end{figure*}

\begin{figure*}[h!]
    \centering
    \begin{tabular}{cc}

        \includegraphics[width=0.33\linewidth]{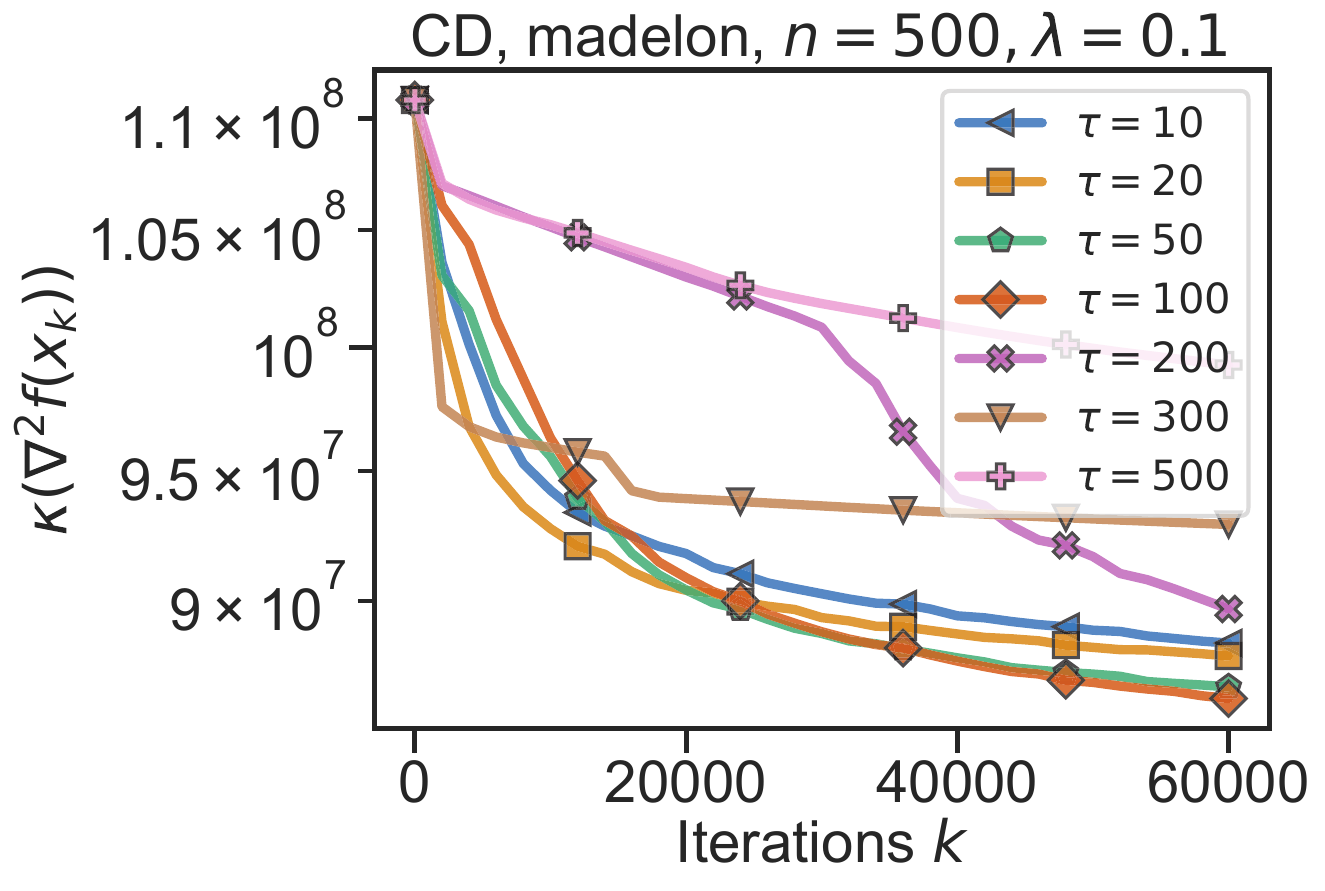}
        \includegraphics[width=0.33\linewidth]{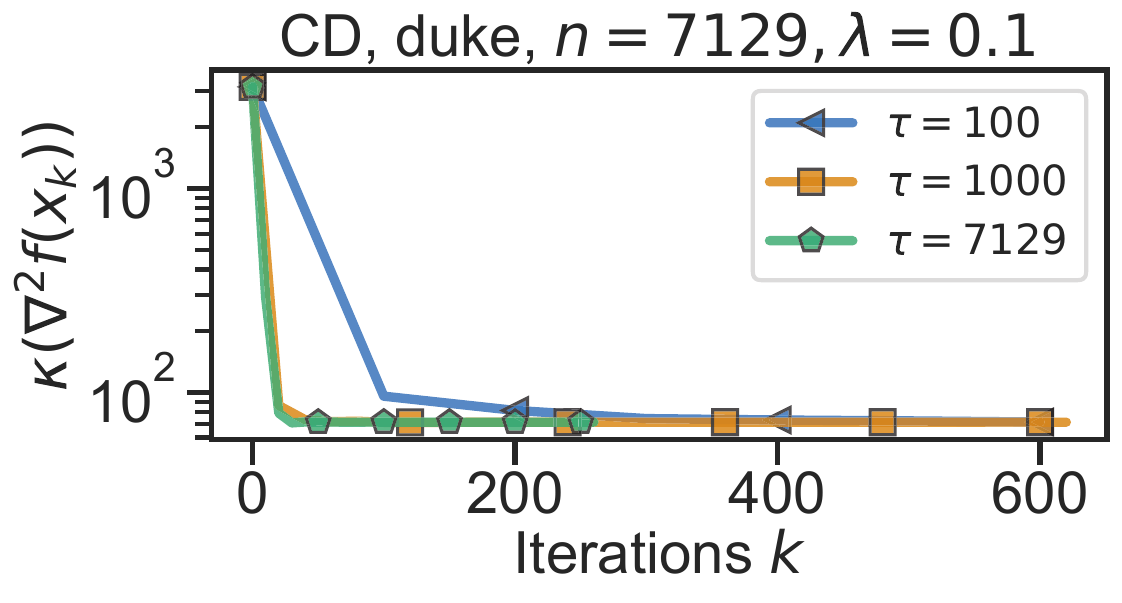}
        \includegraphics[width=0.33\linewidth]{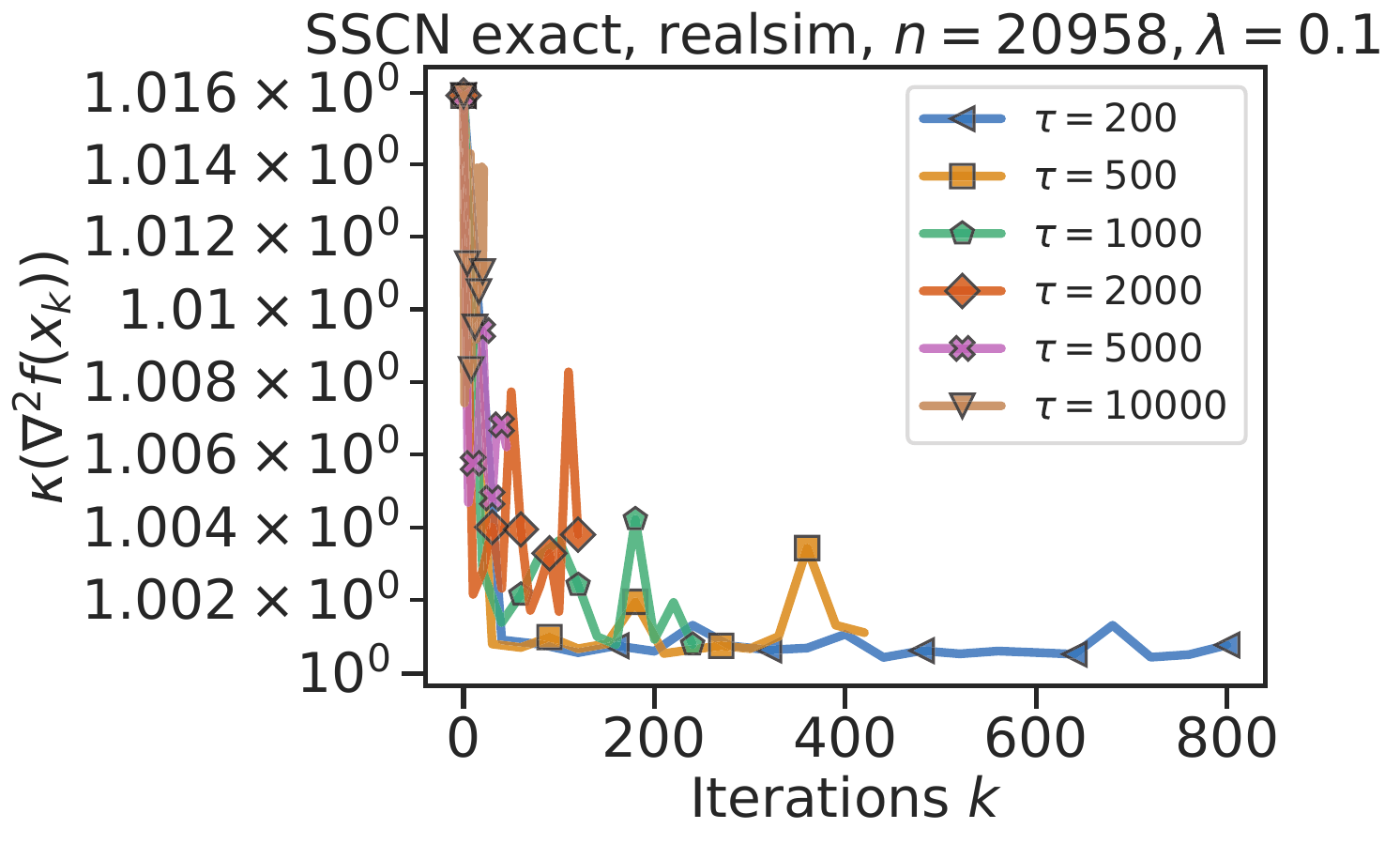}
    \end{tabular}
    \caption{Condition number of full space Hessian of the loss function for the \textit{madelon} (left), \textit{duke} (center) and \textit{realsim} (right) dataset throughout optimization. Note that the condition number on \textit{madelon} is 5 orders of magnitudes larger than the condition number for \textit{duke} and 8 orders larger than the the condition number for the \textit{realsim} dataset.}
    \label{fig:logistic_regression_nonconv_condnum}
    
\end{figure*}

\subsection{Convergence to an $\epsilon$-ball}

We validate our prediction from Theorem 6 which guarantees the convergence to a ball whose radius is determined by $\epsilon_1$ and $\epsilon_2$, which in turn depends on the number of sampled coordinates $\tau(S)$. 
The larger $\tau(S)$, the smaller the radius of the ball, as stated in Lemma \ref{lemma:concentration_bounds_gradient_norm} and Lemma \ref{lemma:concentration_bounds_hess_vec}. In Figure \ref{fig:convergence_to_epsilon_ball} we can see that in the setting of binary logistic regression with non-convex regularizer $\lambda \cdot \sum_{i=1}^n \x_i^2/(1+ \x_i^2)$ indeed the gradient norm to which each constant coordinate schedule converges to decreases with increasing number of sampled coordinates.

\begin{figure*}[ht]
    \centering
    \begin{tabular}{cc}
        \includegraphics[width=0.3\linewidth]{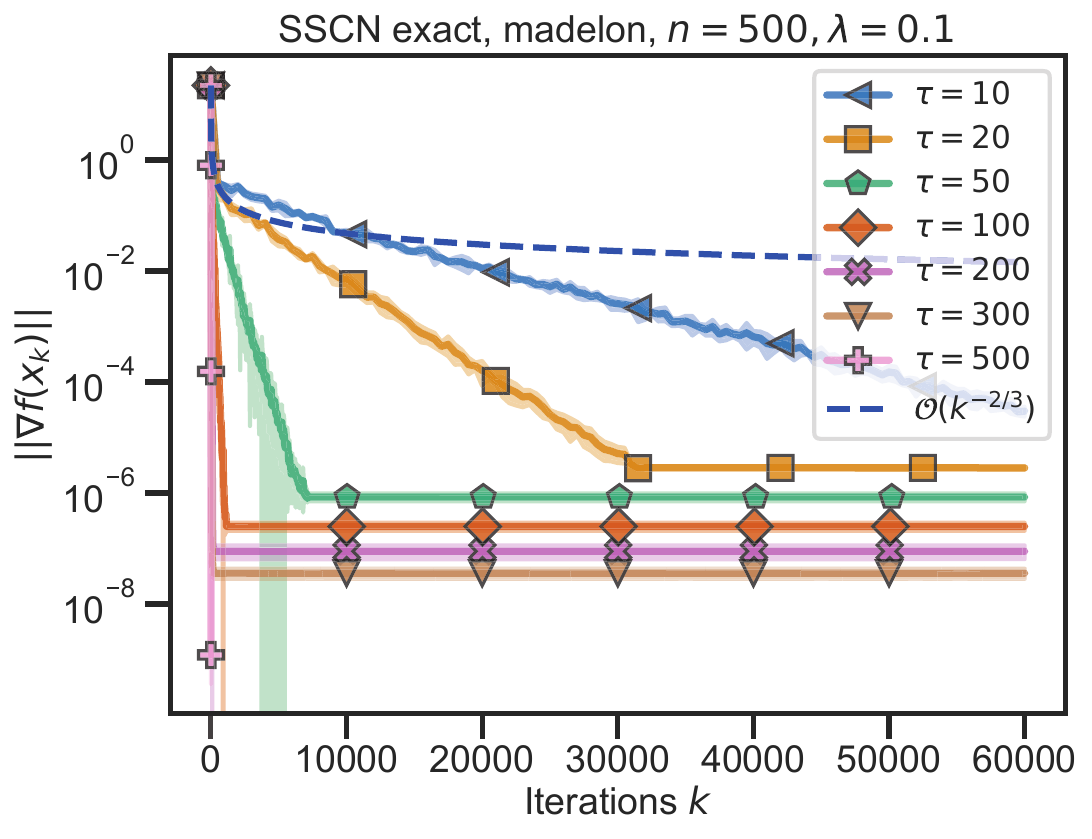}

        \includegraphics[width=0.3\linewidth]{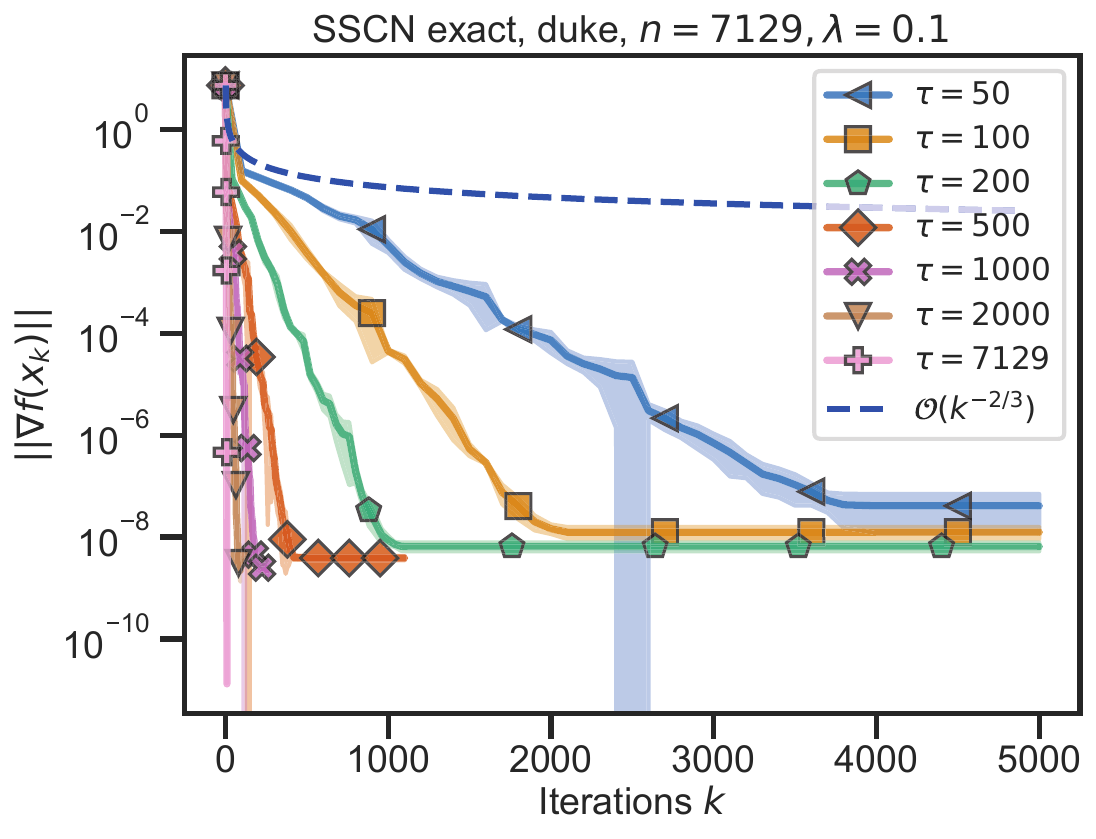}

        \includegraphics[width=0.3\linewidth]{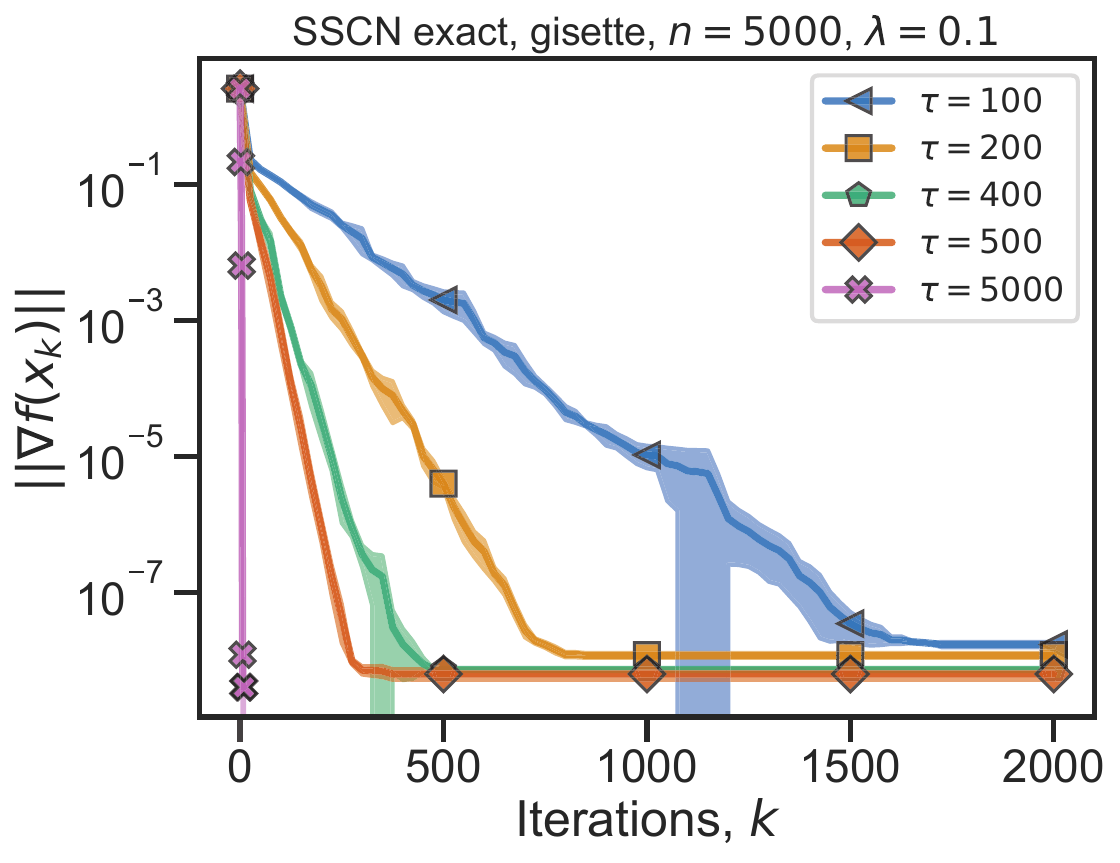}

    \end{tabular}
    \caption{Convergence of different constant coordinate schedules measured w.r.t. iterations averaged over three runs for
logistic regression with non-convex regularization with $\lambda = 0.1$ for three different dataset. Plots correspond to the same setting as~\cref{fig:logistic_regression_nonconv_convergence_gisette} where the limits of the y-axis are chosen to be larger.}
    \label{fig:convergence_to_epsilon_ball}
\end{figure*}

\subsection{Adaptive schedule on $\textit{gisette}$ dataset}

We also verify the proposed adaptive schedule from Eq. \eqref{eq:adaptive_tau} on the $\textit{gisette}$ dataset, where we replaced the full gradient norm and Hessian Frobenius norm by estimates $\|\nabla f(\x_k)_{\text{est}} \|$ and $\| \nabla^2 f(\x_k)_{\text{est}} \|_2 $, which are estimated as exponential moving averages:

\begin{align}
    \|\nabla f(\x_{k+1})_{\text{est}} \| &= \alpha \| \nabla f(\x_{k+1})_{[S]} \| + (1-\alpha) \| \nabla f(\x_k)_{\text{est}} \| \\
    \|\nabla^2 f(\x_{k+1})_{\text{est}} \| &= \alpha \| \nabla^2 f(\x_{k+1})_{[S]} \| + (1-\alpha) \| \nabla^2 f(\x_k)_{\text{est}} \|,
\end{align}

where the weighting factor was chosen as $\alpha=0.2$.
The proposed schedule $\tau(S_k)_{\text{prop}}$ is further smoothed through an exponential moving average $\tau(S_{k+1})=\beta \cdot \tau(S_{k+1})_{\text{prop}} + (1-\beta) \tau(S_k)$. As we can see the schedule is indeed close to an exponential schedule.

\begin{figure*}[htb]
    \centering
    \includegraphics[width=0.45\linewidth]{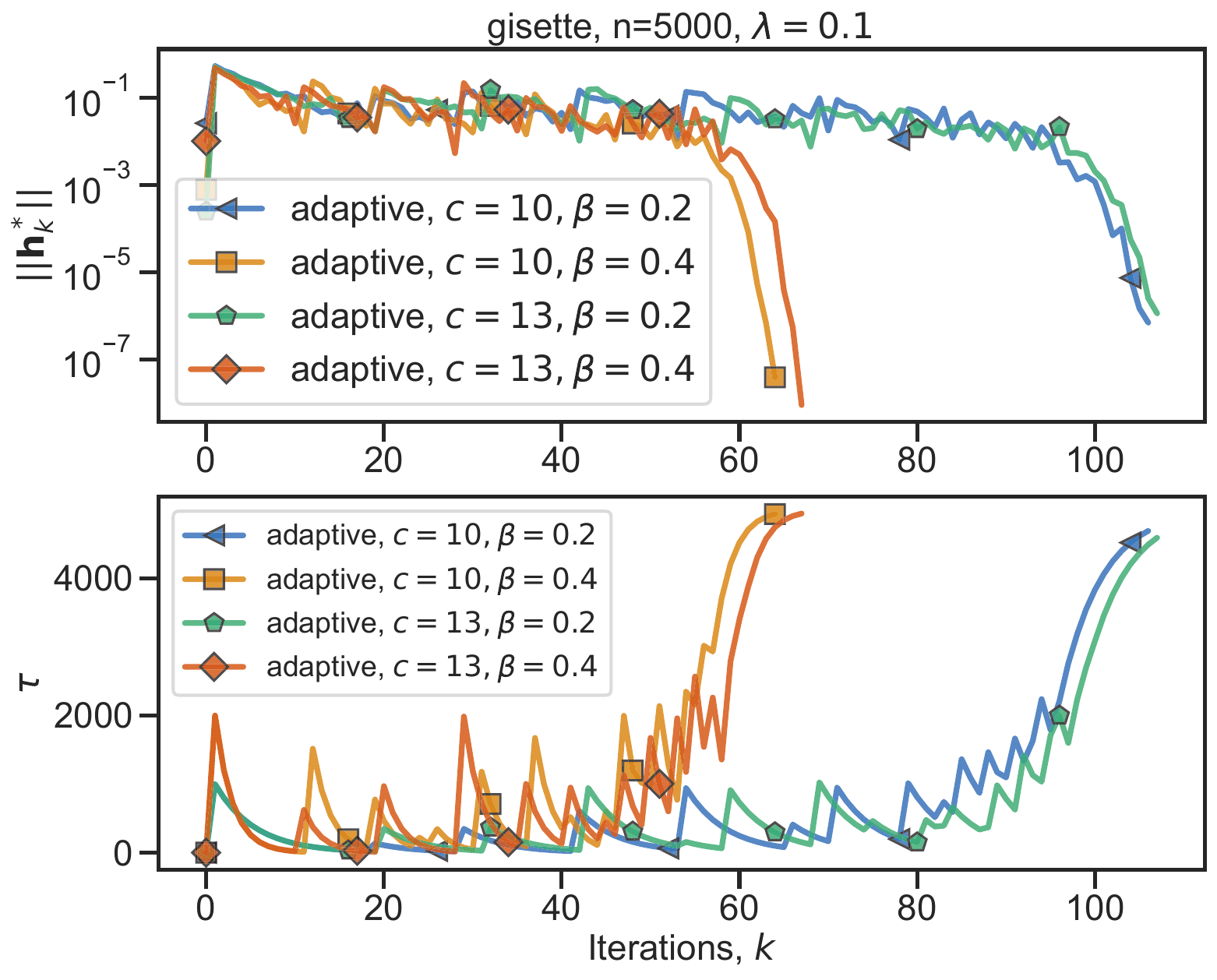}
    \caption{Evolution of $\|\h_k^*\|$ and $\tau(S_k)$ for the \textit{gisette} dataset for the adaptive coordinate schedule $\tau(S_k)$.}
\end{figure*}

\newpage

\subsection{Squared norm of the step  $\h_k$ for \textit{duke} dataset}

We also provide the plot of the squared norm of the step $\h_k$ for the \textit{duke} dataset in \cref{fig:logistic_regression_norm_h_k2_duke} below.
\begin{figure}[h]
    \centering
        \begin{tabular}{cc}
     \includegraphics[width=0.4\textwidth]{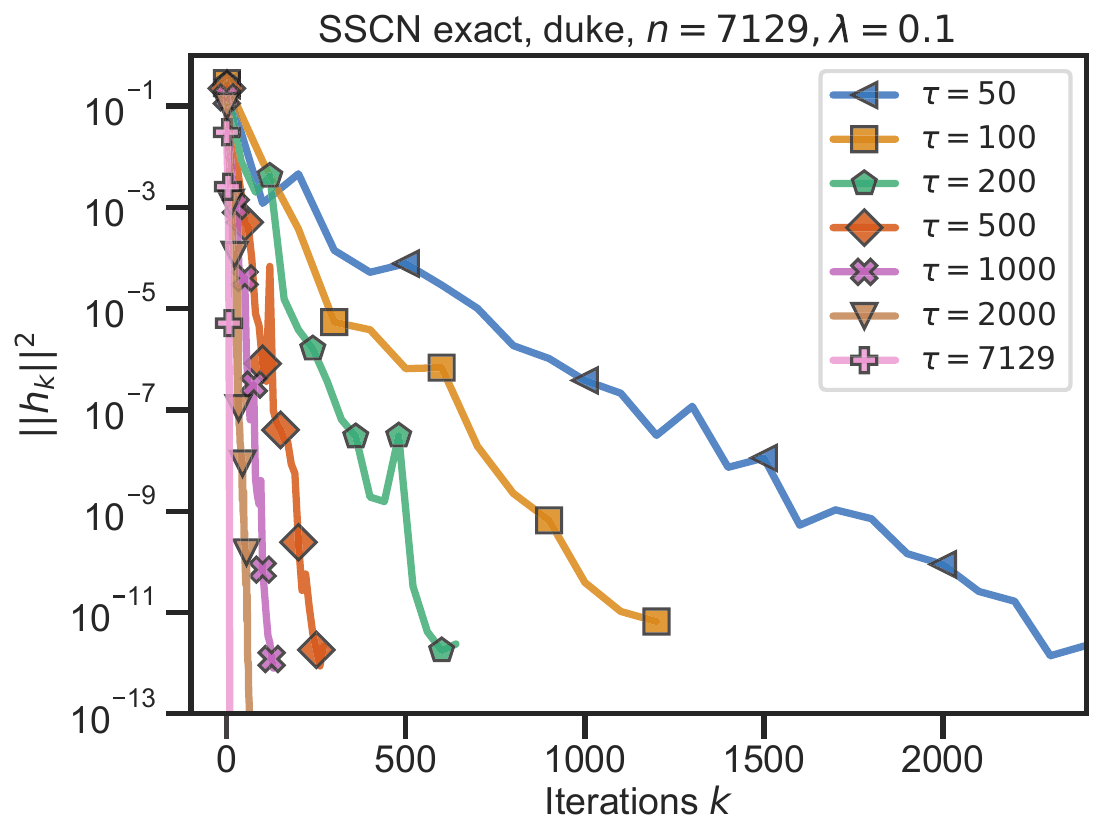}
        \end{tabular}
        
    \caption{Squared norm of the step  $\h_k$ for different constant coordinate schedules for logistic regression with non-convex regularization with $\lambda = 0.1$ for the \textit{duke} dataset.}
\label{fig:logistic_regression_norm_h_k2_duke}
\end{figure}

\newpage